\documentclass[twoside]{article}
\usepackage{ecj,palatino,epsfig,latexsym,natbib}

\usepackage{array}
\usepackage{booktabs}
\usepackage[figuresright]{rotating}
\usepackage{makecell}
\usepackage{utopia}
\usepackage{amssymb}
\usepackage{amsfonts}
\usepackage[fleqn]{amsmath}\allowdisplaybreaks
\usepackage{mathtools}
\usepackage{amsthm}
\usepackage{bm}
\usepackage{graphicx}
\usepackage{multirow}
\usepackage{algorithm}
\usepackage{algorithmic}
\usepackage{hyperref}
\usepackage[american]{babel}
\mathtoolsset{showonlyrefs,showmanualtags}

\newenvironment{myproof}[1]
{\par\noindent\textbf{Proof of #1.}\ \enspace\ignorespaces\begin{allowdisplaybreaks}}
{\end{allowdisplaybreaks}\hspace{\stretch{1}}$\square$}

\newtheorem{definition}{Definition}
\newtheorem{lemma}{Lemma}
\newtheorem{theorem}{Theorem}

\newtheorem{example}{Example}

\begin{document}

\ecjHeader{x}{x}{xxx-xxx}{201X}{Maximizing Monotone Approximately Submodular Minus Modular Functions}{C. Qian}
\title{\bf Multi-objective Evolutionary Algorithms are Still Good: Maximizing Monotone Approximately Submodular Minus Modular Functions}

\author{\name{\bf Chao Qian} \hfill \addr{qianc@nju.edu.cn}\\
        \addr{State Key Laboratory for Novel Software Technology, Nanjing University, Nanjing, 210023, China}
}

\maketitle

\begin{abstract}
As evolutionary algorithms (EAs) are general-purpose optimization algorithms, recent theoretical studies have tried to analyze their performance for solving general problem classes, with the goal of providing a general theoretical explanation of the behavior of EAs. Particularly, a simple multi-objective EA, i.e., GSEMO, has been shown to be able to achieve good polynomial-time approximation guarantees for submodular optimization, where the objective function is only required to satisfy some properties and its explicit formulation is not needed. Submodular optimization has wide applications in diverse areas, and previous studies have considered the cases where the objective functions are monotone submodular, monotone non-submodular, or non-monotone submodular. To complement this line of research, this paper studies the problem class of maximizing monotone approximately submodular minus modular functions (i.e., $g-c$) with a size constraint, where $g$ is a so-called non-negative monotone approximately submodular function and $c$ is a so-called non-negative modular function, resulting in the objective function $(g-c)$ being non-monotone non-submodular in general. Different from previous analyses, we prove that by optimizing the original objective function $(g-c)$ and the size simultaneously, the GSEMO fails to achieve a good polynomial-time approximation guarantee. However, we also prove that by optimizing a distorted objective function and the size simultaneously, the GSEMO can still achieve the best-known polynomial-time approximation guarantee. Empirical studies on the applications of Bayesian experimental design and directed vertex cover show the excellent performance of the GSEMO.
\end{abstract}


\section{Introduction}

As a kind of randomized metaheuristic optimization algorithm, evolutionary algorithms (EAs)~\citep{back:96} have been successfully applied to solve sophisticated optimization problems in diverse areas, e.g., data mining~\citep{mukhopadhyay2013survey}, machine learning~\citep{zhou2019evolutionary}, image processing~\citep{liang-fcs20} and controller optimization~\citep{qiao2020improved}, to name a few. One main advantage of EAs is the general-purpose property, i.e., EAs can be used to optimize any problem where solutions can be represented and evaluated. Meanwhile, the theoretical analysis, particularly running time analysis, of EAs has achieved progress during the past two decades. A lot of theoretical results (e.g.,~\citep{neumann2010bioinspired,auger2011theory}) have been derived, helping us understand the practical behaviors of EAs.

Previous analyses, however, mainly focused on isolated problems, which cannot reflect the general-purpose nature of EAs. To provide a general theoretical explanation of the behavior of EAs, there are some recent efforts~\citep{friedrich2015maximizing,qian2017generalsubset,friedrich2018heavy,qian2019maximizing,bian2020efficient} studying the running time complexity of EAs for solving general classes of submodular optimization problems.

Submodularity~\citep{nemhauser1978analysis} characterizes the diminishing returns property of a set function $f: 2^V \rightarrow \mathbb{R}$, i.e., $\forall X \subseteq Y \subseteq V, v \notin Y: f(X \cup \{v\}) -f(X) \geq f(Y \cup \{v\})-f(Y)$. It has several equivalent definitions, e.g., $\forall X \subseteq Y \subseteq V: f(Y)-f(X) \leq \sum\nolimits_{v \in Y \setminus X} (f(X \cup \{v\})-f(X))$. Submodular optimization has played an important role in many areas such as machine learning, data mining, natural language processing, computer vision, economics and operation research. On one hand, many of their applications involve submodular objective functions, e.g., active learning~\citep{golovin2011adaptive}, influence maximization~\citep{kempe2003maximizing}, document summarization~\citep{lin2011class}, image segmentation~\citep{jegelka2011submodularity} and maximum coverage~\citep{feige1998threshold}. On the other hand, the submodular property allows for polynomial-time approximation algorithms with theoretical guarantees. For example, a celebrated result by~\cite{nemhauser1978analysis} shows that for maximizing monotone submodular functions with a size constraint, the greedy algorithm, which iteratively adds one item with the largest marginal gain, can achieve an approximation ratio of $1-1/e$, which is optimal in general~\citep{nemhauser1978best}. Note that a set function $f: 2^V \rightarrow \mathbb{R}$ is monotone if $\forall X \subseteq Y \subseteq V: f(X) \leq f(Y)$, and a size constraint requires the size of a subset to be no larger than a budget $k$.

\cite{friedrich2015maximizing} first proved that for maximizing monotone submodular functions with a size constraint, the GSEMO, a simple multi-objective EA (MOEA) widely used in theoretical analyses, can achieve the optimal approximation ratio, i.e., $1-1/e$, in $O(n^2(\log n +k))$ expected running time, where $n$ is the size of the ground set $V$ and $k$ is the budget. They also considered a more general problem class, i.e., maximizing monotone submodular functions with $m$ matroid constraints. Note that a size constraint is actually a uniform matroid constraint. The (1+1)-EA, a simple single-objective EA with population size 1 and bit-wise mutation only, has been shown to be able to achieve a $(\frac{1}{m+1/p+\epsilon})$-approximation ratio in $O(\frac{1}{\epsilon}n^{2p(m+1)+1}m\log n)$ expected running time, where $p\geq 1$ and $\epsilon >0$.

Later,~\cite{qian2019maximizing} studied the problem class of maximizing monotone and approximately submodular functions with a size constraint. That is, the objective function to be maximized is not necessarily submodular, but approximately submodular, i.e., satisfies the submodular property to some extent. Several notions of approximate submodularity~\citep{krause2010submodular,das2011submodular,horel2016maximization} have been proposed to measure to what extent a set function $f$ has the submodular property in different ways. For example,~\cite{das2011submodular} and~\cite{bian2017guarantees} introduced the submodularity ratio $\gamma(f)=\min_{X \subseteq Y \subseteq V} \frac{\sum_{v \in Y\setminus X} (f(X \cup \{v\})-f(X))}{f(Y)-f(X)}$, and~\cite{zhang2016submodular} introduced another one $\alpha(f)=\min_{X \subseteq Y, v \notin Y} \frac{f(X \cup \{v\})-f(X)}{f(Y \cup \{v\})-f(Y)}$. For a monotone function $f$, $\gamma(f) \in [0,1]$, $\alpha(f) \in [0,1]$, and the larger $\gamma(f)$ or $\alpha(f)$, the more close to submodularity $f$ is. Using different notions of approximate submodularity,~\cite{qian2019maximizing} proved that the GSEMO can always achieve the best-known polynomial-time approximation guarantee, previously obtained by the greedy algorithm.~\cite{qian2017generalsubset} also considered the case with a general cost constraint, that is, $c(X) \leq k$, where $c(\cdot)$ is a monotone function. They proved that the GSEMO can achieve the best-known approximation ratio of $(\alpha/2)(1-e^{-\alpha})$, where $\alpha$ is the submodularity ratio of the objective function $f$, but the required running time is unbounded.~\cite{bian2020efficient} then proposed a single-objective EA to optimize a surrogate objective integrating $f$ and the cost function $c$, which can achieve this approximation ratio in $O(n^3)$ expected running time.

For non-monotone cases,~\cite{friedrich2015maximizing} studied a specific case, i.e., symmetric functions. They proved that for maximizing symmetric submodular functions with $m$ matroid constraints, the GSEMO can achieve a $(\frac{1}{(m+2)(1+\epsilon)})$-approximation ratio in $O(\frac{1}{\epsilon}n^{m+6}\log n)$ expected running time.~\cite{qian2019maximizing} considered the problem class of maximizing submodular and approximately monotone functions with a size constraint. A set function $f$ is approximately monotone if $\forall X \subseteq V, v \notin X: f(X \cup \{v\}) \geq f(X)-\epsilon$, where $\epsilon \geq 0$ captures the degree of approximate monotonicity. They proved that the GSEMO can find a subset $X$ with $f(X) \geq (1-1/e)\cdot (\mathrm{OPT}-k\epsilon)$ in $O(n^2(\log n +k))$ expected running time, where $\mathrm{OPT}$ denotes the optimal function value. The general non-monotone case has been studied recently.~\cite{friedrich2018heavy} and~\cite{qian2019maximizing} considered the problem class of maximizing non-monotone submodular functions without constraints.~\cite{friedrich2018heavy} proved that the (1+1)-EA using a heavy-tailed mutation operator can achieve an approximation ratio of $(\frac{1}{3}-\frac{\epsilon}{n})$ in $O(\frac{1}{\epsilon}n^3\log\frac{n}{\epsilon}+n^{\beta})$ expected running time, where $\epsilon >0$. The heavy-tailed mutation operator samples $l \in \{1,2,\ldots,n\}$ according to a power-law distribution with a parameter $\beta>1$, and then flips $l$ bits of a solution chosen uniformly at random.~\cite{qian2019maximizing} proved that a variant of GSEMO, called GSEMO-C, can also achieve the $(\frac{1}{3}-\frac{\epsilon}{n})$-approximation ratio in $O(\frac{1}{\epsilon}n^4\log n)$ expected running time. The difference between the GSEMO and GSEMO-C is that the GSEMO generates a new offspring solution by bit-wise mutation in each iteration, whereas the GSEMO-C generates this new solution (i.e., set) as well as its complement.

The above mentioned works, showing the good general approximation ability of EAs, are summarized in Table~\ref{table_runtime}, where each work is categorized according to whether the concerned objective functions satisfy the monotone and submodular property. A natural question is then whether EAs can still achieve good polynomial-time approximation guarantees when the objective functions are neither monotone nor submodular. In this paper, we thus consider the problem class of maximizing monotone approximately submodular minus modular functions with a size constraint, i.e.,
\begin{align}\label{eq-problem-1}
\arg\max\nolimits_{X \subseteq V} \quad g(X)-c(X) \qquad \text{s.t.} \qquad |X| \leq k,
\end{align}
where $g$ is a non-negative monotone approximately submodular function, and $c$ is a non-negative modular function, i.e., $\forall X \subseteq V: c(X)=\sum_{v\in X} c(\{v\})$. The objective function $(g-c)$ is non-submodular in general, and can be non-monotone and take negative values. It is known that monotone approximately submodular maximization with a size constraint has various applications, such as Bayesian experimental design~\citep{krause2008near}, dictionary selection~\citep{krause2010submodular} and sparse regression~\citep{das2011submodular}. The considered problem Eq.~(\refeq{eq-problem-1}) is a natural extension by encoding a cost for each item.~\cite{harshaw2019submodular} proposed the distorted greedy algorithm and proved that it can find a subset $X$ with $g(X)-c(X) \geq (1-e^{-\gamma})\cdot g(X^*) -c(X^*)$, where $X^*$ denotes an optimal solution of Eq.~(\refeq{eq-problem-1}), and $\gamma$ is the submodularity ratio of $g$ measuring how close $g$ is to submodularity. Note that this is the best-known polynomial-time approximation guarantee.

\begin{sidewaystable*}\centering
\caption{A summary of the works on analyzing the running time of EAs for solving problem classes of submodular optimization, where each work is categorized according to the property of considered objective functions.}\label{table_runtime}\vspace{1em}
\small
\begin{tabular}{c|l|c|c}
\hline\hline
& & &\\[-8pt]
Property of objectives & Problem & Approximation ratio & Expected running time  \\[2pt]
\hline
& &  &\\[-8pt]
\multirow{4}{*}{\makecell[c]{Monotone\\ submodular}}& \multirow{2}{*}{\makecell[l]{Monotone submodular maximization \\s.t. a size constraint~\citep{friedrich2015maximizing}}}  &  \multirow{2}{*}{$1-1/e$} & \multirow{2}{*}{$O(n^2(\log n +k))$} \\
&&&\\
\cline{2-4}
& & &\\[-8pt]
&\multirow{2}{*}{\makecell[l]{Monotone submodular maximization\\ s.t. $m$ matroid constraints~\citep{friedrich2015maximizing}}}  &\multirow{2}{*}{$1/(m+1/p+\epsilon)$} & \multirow{2}{*}{$O(\frac{1}{\epsilon}n^{2p(m+1)+1}m\log n)$}\\
&&&\\
\hline\hline
& & & \\[-8pt]
\multirow{4}{*}{\makecell[c]{Monotone\\ non-submodular}}& \multirow{2}{*}{\makecell[l]{Monotone approximately submodular maximization\\ s.t. a size constraint~\citep{qian2019maximizing}}}  &  \multirow{2}{*}{$1-e^{-\gamma}$} & \multirow{2}{*}{$O(n^2(\log n +k))$} \\
&&&\\
\cline{2-4}
& & &\\[-8pt]
&\multirow{2}{*}{\makecell[l]{Monotone approximately submodular maximization\\ s.t. a monotone cost constraint~\citep{bian2020efficient}}}  &\multirow{2}{*}{$(\alpha/2)(1-e^{-\alpha})$} & \multirow{2}{*}{$O(n^3)$}\\
&&&\\
\hline\hline
& & & \\[-8pt]
\multirow{6}{*}{\makecell[c]{Non-monotone\\ submodular}} & \multirow{2}{*}{\makecell[l]{Symmetric submodular maximization\\ s.t. $m$ matroid constraints~\citep{friedrich2015maximizing}}} & \multirow{2}{*}{$1/((m+2)(1+\epsilon))$} & \multirow{2}{*}{$O(\frac{1}{\epsilon}n^{m+6}\log n)$}\\
&&&\\
\cline{2-4}
& & &\\[-8pt]
& \multirow{2}{*}{\makecell[l]{Submodular approximately monotone maximization\\ s.t. a size constraint~\citep{qian2019maximizing}}} & \multirow{2}{*}{$(1-1/e)\cdot (\mathrm{OPT}-k\epsilon)$} & \multirow{2}{*}{$O(n^2(\log n +k))$}\\
&&&\\
\cline{2-4}
& & &\\[-8pt]
& \multirow{2}{*}{\makecell[l]{Non-monotone submodular maximization\\ without constraints~\citep{friedrich2018heavy,qian2019maximizing}}} & \multirow{2}{*}{$1/3-\epsilon/n$} & \multirow{2}{*}{\makecell[l]{$O(\frac{1}{\epsilon}n^3\log\frac{n}{\epsilon}+n^{\beta})$\\[1pt]$O(\frac{1}{\epsilon}n^4\log n)$}}\\
& & &\\[4pt]
\hline\hline
& & &\\[-8pt]
\multirow{2}{*}{\makecell[c]{Non-monotone\\ non-submodular}}&\multirow{2}{*}{\makecell[l]{Monotone approximately submodular minus modular\\ maximization s.t. a size constraint (This work)}}  &\multirow{2}{*}{$(1-e^{-\gamma})\cdot g(X^*) -c(X^*)$} & \multirow{2}{*}{$O(n^2(\log n+k))$}\\
&&&\\[2pt]
\hline\hline
\end{tabular}\vspace{0.8em}
\makecell[l]{Notes: $n$ is the problem size, $k$ is the budget of a size constraint $|X| \leq k$, $m$ is the number of matroid constraints, $p\geq 1$, $\epsilon >0$, $\gamma, \alpha \in [0,1]$ are the \\submodularity ratios, $\beta >1$, $\mathrm{OPT}$ and $X^*$ denote the optimal function value and an optimal solution, respectively.}
\end{sidewaystable*}

In this paper, we analyze the approximation performance of the GSEMO for solving Eq.~(\refeq{eq-problem-1}). We prove that by maximizing a distorted objective $(1-\gamma/k)^{k-|X|}g(X)-c(X)+(|X|/k)c(V)$ and minimizing the subset size $|X|$ simultaneously, the GSEMO can obtain a subset $X$ with $|X|\leq k$ and $g(X)-c(X) \geq (1-e^{-\gamma})\cdot g(X^*) -c(X^*)$ in $O(n^2(\log n+k))$ expected running time. Thus, the last row of Table~\ref{table_runtime} is filled. Our analysis together with~\citep{friedrich2015maximizing,friedrich2018heavy,qian2019maximizing} show that a simple MOEA, i.e., GSEMO, can achieve good polynomial-time approximation guarantees for diverse submodular optimization problems, disclosing the general-purpose property of EAs. Note that our analysis is different from previous ones. In previous analyses, the original objective function is often treated as one objective to be optimized in the bi-objective reformulation~\citep{friedrich2015maximizing,friedrich2018heavy,qian2019maximizing}, while we use a distorted one here. In fact, we prove that by maximizing the original objective function $(g(X)-c(X))$ and minimizing $|X|$ simultaneously, the GSEMO fails to achieve the approximation guarantee $g(X)-c(X) \geq (1-e^{-\gamma})\cdot g(X^*) -c(X^*)$ in polynomial running time.

As the distorted greedy algorithm is the existing algorithm with the best-known approximation guarantee~\citep{harshaw2019submodular}, we empirically compare the GSEMO with it as well as its stochastic version on the applications of Bayesian experimental design and directed vertex cover. The results show that the GSEMO can perform significantly better by using more running time. Compared with the running time bound (i.e., the worst-case running time) derived in the theoretical analysis, the GSEMO can be relatively efficient in practice. We also run the stochastic distorted greedy algorithm multiple times independently until the running time reaches that of the GSEMO. By comparing with the best solution found in the multiple runs, we find that the GSEMO is still significantly better.

To examine whether employing more advanced MOEAs can further bring performance improvement, we use the popular NSGA-II algorithm~\citep{deb2002fast} to maximize the distorted objective function and minimize the subset size simultaneously. Surprisingly, we observe that the GSEMO performs better. One reason may be that the population of the NSGA-II can contain dominated solutions, leading to the low efficiency, while the population of the GSEMO contains only non-dominated solutions. Another reason may be that in our experiments, the NSGA-II uses a much smaller probability of performing mutation than the GSEMO. Thus, an interesting future work is to investigate whether the NSGA-II can be better by varying its parameters, e.g., the population size and the probability of performing mutation.

Furthermore, we run the GSEMO to maximize the original objective function $(g(X)-c(X))$ and minimize the subset size $|X|$ simultaneously. We observe that it can achieve good performance in most cases, but can even be worse than the distorted greedy algorithm sometimes. This verifies the theoretical analysis that using $(g(X)-c(X))$ directly cannot guarantee a good approximation, i.e., can lead to bad performance in worst cases.

The rest of this paper is organized as follows. Section~\ref{sec-problem} introduces the considered problem class. Sections~\ref{sec-algorithm} to~\ref{sec-experiment} present the GSEMO, theoretical analysis and empirical study, respectively. Section~\ref{sec-conlusion} concludes the paper.

\section{Maximizing Monotone Approximately Submodular Minus Modular Functions with a Size Constraint}\label{sec-problem}

Let $\mathbb{R}$ and $\mathbb{R}^{+}$ denote the set of reals and non-negative reals, respectively. Given a ground set $V=\{v_1,v_2,\ldots,v_n\}$ of items, a set function $f:2^V \rightarrow \mathbb{R}$ is defined on subsets of $V$, and maps any subset to a real value. A set function $f:2^V \rightarrow \mathbb{R}$ is monotone if $\forall X \subseteq Y \subseteq V: f(X) \leq f(Y)$, implying that the function value will not decrease as a set extends.

\begin{definition}[Submodularity~\citep{nemhauser1978analysis}]
A set function $f:2^V \rightarrow \mathbb{R}$ is submodular if
\begin{align}\label{def-submodular}
\forall X,Y\subseteq V: f(X)+f(Y) \geq f(X\cup Y) +f(X \cap Y);
\end{align}
or equivalently
\begin{align}\label{def-submodular-1}
\forall X \subseteq Y \subseteq V, v \notin Y: f(X \cup \{v\})-f(X) \geq f(Y \cup \{v\}) - f(Y);
\end{align}
or equivalently
\begin{align}\label{def-submodular-2}
\forall X \subseteq Y \subseteq V: f(Y)-f(X) \leq \sum\nolimits_{v \in Y \setminus X} \big(f(X \cup \{v\})-f(X)\big).
\end{align}
\end{definition}

Eq.~(\refeq{def-submodular}) implies that the sum of the function values of any two sets is at least as large as that of their union and intersection. Eq.~(\refeq{def-submodular-1}) intuitively represents the diminishing returns property, i.e., the benefit of adding an item to a set will not increase as the set extends. Eq.~(\refeq{def-submodular-2}) implies that the benefit by adding a set of items to a set $X$ is no larger than the combined benefits of adding its individual items to $X$. A set function $f:2^V \rightarrow \mathbb{R}$ is modular if Eq.~(\refeq{def-submodular}), Eq.~(\refeq{def-submodular-1}) or Eq.~(\refeq{def-submodular-2}) holds with equality. For a modular function $f$, it holds that $\forall X \subseteq V: f(X)=\sum_{v \in X} f(\{v\})$; it is non-negative iff $\forall v \in V: f(\{v\}) \geq 0$.

For a general set function $f:2^V \rightarrow \mathbb{R}$, several notions of approximate submodularity~\citep{krause2010submodular,das2011submodular,zhang2016submodular,horel2016maximization,zhou2016causal} have been introduced to measure to what extent $f$ has the submodular property. Among them, the submodularity ratio as presented in Definition~\ref{def-approx-submodular-2} has been used most widely.
\begin{definition}[Submodularity Ratio~\citep{das2011submodular}]\label{def-approx-submodular-2}
Let $f: 2^V \rightarrow \mathbb{R}$ be a set function. The submodularity ratio of $f$ w.r.t. a set $X \subseteq V$ and a parameter $l \geq 1$ is
\begin{align}
\gamma_{X,l}(f)=\min_{L \subseteq X, S: |S|\leq l, S \cap L =\emptyset} \frac{\sum_{v \in S} (f(L \cup \{v\})-f(L))}{f(L \cup S)-f(L)}.
\end{align}
\end{definition}
The submodularity ratio is actually defined based on Eq.~(\refeq{def-submodular-2}), and captures how much more $f$ can increase by adding any set $S$ of size at most $l$ to any subset $L$ of $X$, compared with the combined benefits of adding the individual items of $S$ to $L$. For a monotone set function $f$, it holds that (1) $\forall X \subseteq V, l \geq 1: \gamma_{X,l}(f) \in [0,1]$; (2) $f$ is submodular iff $\forall X \subseteq V, l \geq 1: \gamma_{X,l}(f) = 1$. The submodularity ratio has been used to measure the closeness of the objective function to submodularity in diverse non-submodular applications, e.g., sparse regression~\citep{das2011submodular}, low rank optimization~\citep{khanna2017approximation}, sparse support selection~\citep{elenberg2018restricted}, and determinantal function maximization~\citep{qian2018approximation}, where the corresponding lower bounds of $\gamma_{X,l}(f)$ have been derived.

In this paper, we will use a slightly different definition of submodularity ratio as in~\citep{bian2017guarantees,bogunovic2018robust,harshaw2019submodular}, i.e.,
\begin{align}\label{eq-submodular}
\gamma(f)=\min_{X \subseteq Y \subseteq V} \frac{\sum_{v \in Y\setminus X} (f(X \cup \{v\})-f(X))}{f(Y)-f(X)}.
\end{align}
It is easy to see that $\forall X \subseteq V, l \geq 1: \gamma_{X,l}(f) \geq \gamma(f)$. For a monotone set function $f$, it holds that (1) $\gamma(f) \in [0,1]$, and (2) $f$ is submodular iff $\gamma(f) = 1$.

The studied problem class is presented in Definition~\ref{def-Prob}. The goal is to find a subset of size at most $k$ maximizing a given objective function, which is the difference between a non-negative monotone approximately submodular function $g$ and a non-negative modular function $c$. The approximately submodular degree of $g$ is characterized by its submodularity ratio $\gamma(g)$, which will be represented as $\gamma$ for short.

\begin{definition}[Maximizing Monotone Approximately Submodular Minus Modular Functions with a Size Constraint]\label{def-Prob}
Given a non-negative monotone approximately submodular function $g: 2^V \rightarrow \mathbb{R}^+$, a non-negative modular function $c: 2^V \rightarrow \mathbb{R}^+$, and a budget $k$, to find a subset $X\subseteq V$ of size at most $k$ such that
\begin{align}\label{eq-problem}
\arg \max\nolimits_{X \subseteq V}\quad g(X)-c(X) \qquad  \text{s.t.} \qquad |X| \leq k.
\end{align}
\end{definition}

This is a natural extension of the widely studied problem of maximizing monotone approximately submodular functions with a size constraint~\citep{das2018approximate} by considering the cost (modeled by the modular function $c$) for each item. Note that the objective function $(g-c)$ is non-submodular in general, because otherwise $g$ is submodular, making a contradiction. It is easy to see that $(g-c)$ can be non-monotone and take negative values.

For submodular optimization, it is well known that the greedy algorithm, which iteratively adds one item with the largest marginal gain on the objective function, is a good approximation solver in many cases.~\cite{harshaw2019submodular}, however, showed that the greedy algorithm fails to obtain an approximation guarantee for the considered problem Eq.~(\refeq{eq-problem}), and thus, proposed the distorted greedy algorithm as presented in Algorithm~\ref{alg:dg}. In the $i$-th iteration, rather than maximizing the marginal gain on $g-c$, i.e., $g(X_i \cup \{v\})-c(X_i \cup \{v\})-(g(X_i)-c(X_i))=(g(X_i \cup \{v\})-g(X_i))-c(\{v\})$, it maximizes a distorted one, $\left(1-\frac{\gamma}{k}\right)^{k-(i+1)}(g(X_i \cup \{v\})-g(X_i))-c(\{v\})$, which gradually increases the importance of $g$. It has been proved that the distorted greedy algorithm outputs a subset $X$ with $g(X)-c(X) \geq (1-e^{-\gamma})\cdot g(X^*)-c(X^*)$, where $X^*$ denotes an optimal solution of Eq.~(\refeq{eq-problem}). Note that this polynomial-time approximation guarantee is the best known one. Though its optimality is not yet known, the factor $(1-e^{-\gamma})$ w.r.t. $g(X^*)$ is optimal, because~\cite{harshaw2019submodular} have proved that when the cost for each item is 0, i.e., the objective function is just $g$, no polynomial-time algorithm can achieve $(1-e^{-\gamma}+\epsilon)$-approximation, where $\epsilon>0$.

\begin{algorithm}[h!]\caption{Distorted Greedy Algorithm}
\label{alg:dg}
\textbf{Input}: monotone approximately submodular $g: 2^V \rightarrow \mathbb{R}^+$ with the submodularity ratio $\gamma$, modular $c: 2^V \rightarrow \mathbb{R}^+$, and budget $k$\\
\textbf{Process}:
\begin{algorithmic}[1]
\STATE Let $X_0=\emptyset$;
\STATE \textbf{for} $i=0$ \textbf{to} $k-1$ \textbf{do}
\STATE \quad $v_i \leftarrow \arg\max_{v \in V} \left(1-\frac{\gamma}{k}\right)^{k-(i+1)}(g(X_i \cup \{v\})-g(X_i))-c(\{v\})$;
\STATE \quad \textbf{if} {$\left(1-\frac{\gamma}{k}\right)^{k-(i+1)}(g(X_i \cup \{v_i\})-g(X_i))-c(\{v_i\})>0$} \,\textbf{then}
\STATE \qquad $X_{i+1} \leftarrow X_{i} \cup \{v_i\}$
\STATE \quad \textbf{else}
\STATE \qquad $X_{i+1} \leftarrow X_{i}$
\STATE \quad \textbf{end if}
\STATE \textbf{end for}
\STATE \textbf{return} $X_k$
\end{algorithmic}
\end{algorithm}

For acceleration,~\cite{harshaw2019submodular} further proposed the stochastic distorted greedy algorithm by adopting the random sampling technique~\citep{Baharan.aaai15}. As presented in Algorithm~\ref{alg:sdg}, in each iteration, it selects an item from a random sample of size $\lceil\frac{n}{k}\ln(\frac{1}{\epsilon})\rceil$, instead of the whole set $V$. Thus, the running time (counted by the number of function evaluations) is reduced from $O(kn)$ to $O(n \log \frac{1}{\epsilon})$, while the output subset $X$ can keep an approximation guarantee as $\mathbb{E}(g(X)-c(X)) \geq (1-e^{-\gamma}-\epsilon)\cdot g(X^*)-c(X^*)$, where $\mathbb{E}(\cdot)$ denotes the expectation of a random variable.

\begin{algorithm}[h!]\caption{Stochastic Distorted Greedy Algorithm}
\label{alg:sdg}
\textbf{Input}: monotone approximately submodular $g: 2^V \rightarrow \mathbb{R}^+$ with the submodularity ratio $\gamma$, modular $c: 2^V \rightarrow \mathbb{R}^+$, budget $k$, and $\epsilon>0$\\
\textbf{Process}:
\begin{algorithmic}[1]
\STATE Let $X_0=\emptyset$;
\STATE \textbf{for} $i=0$ \textbf{to} $k-1$ \textbf{do}
\STATE \quad $V_i \leftarrow \lceil\frac{n}{k}\ln(\frac{1}{\epsilon})\rceil$ items uniformly sampled from $V$ with replacement;
\STATE \quad $v_i \leftarrow \arg\max_{v \in V_i} \left(1-\frac{\gamma}{k}\right)^{k-(i+1)}(g(X_i \cup \{v\})-g(X_i))-c(\{v\})$;
\STATE \quad \textbf{if} {$\left(1-\frac{\gamma}{k}\right)^{k-(i+1)}(g(X_i \cup \{v_i\})-g(X_i))-c(\{v_i\})>0$} \,\textbf{then}
\STATE \qquad $X_{i+1} \leftarrow X_{i} \cup \{v_i\}$
\STATE \quad \textbf{else}
\STATE \qquad $X_{i+1} \leftarrow X_{i}$
\STATE \quad \textbf{end if}
\STATE \textbf{end for}
\STATE \textbf{return} $X_k$
\end{algorithmic}
\end{algorithm}

These two algorithms require the submodularity ratio $\gamma$ of the function $g$. In cases where the exact value of $\gamma$ is unknown, lower bounds of $\gamma$ can be used, and the approximation guarantees change accordingly, i.e., $\gamma$ is replaced by its lower bound. Note that the value oracle model is assumed, i.e., for a subset $X$, an algorithm can query an oracle to obtain its function value.

\section{Multi-objective Evolutionary Algorithms}\label{sec-algorithm}

To examine the approximation performance of EAs optimizing the problem class in Definition~\ref{def-Prob}, we consider the GSEMO, a simple MOEA widely used in previous theoretical analyses~\citep{laumanns2004running,friedrich2010approximating,neumann2011computing,qian2013analysis}. As presented in Algorithm~\ref{algo:GSEMO}, the GSEMO is used for maximizing multiple pseudo-Boolean objective functions simultaneously. Note that a subset $X$ of $V$ can be represented by a Boolean vector $\bm{x} \in \{0,1\}^n$, where the $i$-th bit $x_i=1$ if $v_i \in X$, otherwise $x_i=0$. Thus, a pseudo-Boolean function $f: \{0,1\}^n \rightarrow \mathbb{R}$ naturally characterizes a set function $f: 2^{V} \rightarrow \mathbb{R}$. In the following, $\bm{x}\in \{0,1\}^n$ and its corresponding subset will not be distinguished for notational convenience.

Different from the scenario of single-objective optimization, solutions may be incomparable in multi-objective maximization $\max\, (f_1,f_2,\ldots,f_m)$, due to the conflicting of objectives. The \emph{domination}-based comparison is usually adopted.

\begin{definition}[Domination]\label{def_Domination}
For two solutions $\bm x$ and $\bm{x}'$,
\begin{enumerate}
  \item $\bm{x}$ \emph{weakly dominates} $\bm{x}'$ (i.e., $\bm{x}$ is \emph{better} than $\bm{x}'$, denoted by $\bm{x} \succeq \bm{x}'$) if \;$\forall 1 \leq i \leq m: f_i(\bm{x}) \geq f_i(\bm{x}')$;
  \item ${\bm{x}}$ \emph{dominates} $\bm{x}'$ (i.e., $\bm{x}$ is \emph{strictly better} than $\bm{x}'$, denoted by $\bm{x} \succ \bm{x}'$) if ${\bm{x}} \succeq \bm{x}' \wedge \exists i: f_i(\bm{x}) > f_i(\bm{x}')$;
  \item $\bm{x}$ and $\bm{x}'$ are \emph{incomparable} if neither $\bm{x} \succeq \bm{x}'$ nor $\bm{x}' \succeq \bm{x}$.
\end{enumerate}
\end{definition}

As presented in Algorithm~\ref{algo:GSEMO}, the GSEMO starts from a random initial solution (lines~1--2), and iteratively improves the quality of solutions in the population $P$ (lines~3--9). In each iteration, a parent solution $\bm{x}$ is selected from $P$ uniformly at random (line~4), and used to generate an offspring solution $\bm{x}'$ by bit-wise mutation (line~5), which flips each bit of $\bm{x}$ independently with probability $1/n$. The offspring solution $\bm{x}'$ is then used to update the population $P$ (lines~6--8). If $\bm{x}'$ is not dominated by any parent solution in $P$ (line~6), it will be included into $P$, and meanwhile those parent solutions weakly dominated by $\bm{x}'$ will be deleted (line~7). By this updating procedure, the solutions contained in the population $P$ are always incomparable.

\begin{algorithm}[t!]\caption{GSEMO Algorithm}\label{algo:GSEMO}
\textbf{Input}: $m$ pseudo-Boolean functions $f_1,f_2,\ldots,f_m$, where $f_i: \{0,1\}^n \rightarrow \mathbb{R}$\\
\textbf{Process}:
    \begin{algorithmic}[1]
    \STATE Choose $\bm{x} \in \{0,1\}^n$ uniformly at random;
    \STATE $P \gets \{\bm x\}$;
    \STATE \textbf{repeat}
    \STATE \quad Choose $\bm x$ from $P$ uniformly at random;
    \STATE \quad Create $\bm{x}'$ by flipping each bit of $\bm x$ with probability $1/n$;
    \STATE \quad \textbf{if} \, {$\nexists \bm z \in P$ such that $\bm z \succ \bm {x}'$} \,\textbf{then}
    \STATE \qquad $P \gets (P \setminus \{\bm z \in P \mid \bm {x}' \succeq \bm z\}) \cup \{\bm {x}'\}$
    \STATE \quad \textbf{end if}
    \STATE \textbf{until} some criterion is met
    \STATE \textbf{return} $\arg\max_{\bm{x} \in P, |\bm{x}|\leq k} \; g(\bm{x})-c(\bm{x})$
    \end{algorithmic}
\end{algorithm}

To employ the GSEMO, the problem Eq.~(\refeq{eq-problem}) is transformed into a bi-objective maximization problem
\begin{align}\label{eq-bi-problem}
&\arg\max\nolimits_{\bm{x} \in \{0,1\}^n} \quad  (f_1(\bm{x}),f_2(\bm{x})),\\
&\text{where}\quad\begin{cases}
f_1(\bm{x}) = (1-\frac{\gamma}{k})^{k-|\bm{x}|}g(\bm{x})-c(\bm{x})+\frac{|\bm{x}|}{k}c(\bm{1}),\\
f_2(\bm x) = -|\bm{x}|.
\end{cases}
\end{align}
Note that $\bm{1}$ denotes the all-1s vector (i.e., the whole set $V$), implying that $c(\bm{1})=c(V)=\sum_{v \in V} c(\{v\})$ is a constant. Thus, the GSEMO is to maximize the distorted objective function $f_1$ and minimize the subset size $|\bm{x}|=\sum^n_{i=1}x_i$ simultaneously. The setting of $f_1$ is inspired by the distorted greedy algorithm (i.e., Algorithm~\ref{alg:dg}). In line~10 of the GSEMO, the best solution w.r.t. the original single-objective constrained problem Eq.~(\refeq{eq-problem}) will be selected from the resulting population $P$ as the final solution; that is, the solution with the largest $(g-c)$ value satisfying the size constraint in $P$ (i.e., $\arg\max_{\bm{x} \in P, |\bm{x}|\leq k} \; g(\bm{x})-c(\bm{x})$) will be returned.

Note that bi-objective reformulation here is an intermediate process for solving single-objective constrained optimization problems, which has been shown helpful in several cases~\citep{neumann2006minimum,friedrich2010approximating,neumann2011computing,qian.ijcai15}. What we focus on is still the quality of the best solution w.r.t. the original single-objective problem, in the population found by the GSEMO, rather than the quality of the population w.r.t. the reformulated bi-objective problem. Thus, the running time of the GSEMO is measured by the number of function evaluations until the best solution w.r.t. the original single-objective problem in the population reaches some approximation guarantee for the first time.

\section{Theoretical Analysis}\label{sec-theory}

In this section, we first prove the approximation guarantee of the GSEMO in Theorem~\ref{theo-main}, showing that the returned solution can obtain at least $(1-e^{-\gamma})$ as much $g$ as an optimal solution by paying the same cost. This reaches the best known guarantee, obtained by the distorted greedy algorithm~\citep{harshaw2019submodular}. As in previous analyses for non-monotone submodular optimization, e.g.,~\citep{buchbinder2014submodular,friedrich2018greedy}, we may assume that there is a set $D$ of $k$ ``dummy" items whose marginal contribution to any set is 0, i.e., $\forall X \subseteq V: g(X)=g(X\setminus D) \wedge c(X)=c(X\setminus D)$. Otherwise, we can add $k$ such dummy items to the ground set $V$, and delete them from the returned solution of the GSEMO, effecting neither the objective value of an optimal solution, nor the objective value of the GSEMO's returned solution.

\begin{theorem}\label{theo-main}
For maximizing monotone approximately submodular minus modular functions with a size constraint, i.e., solving the problem Eq.~(\ref{eq-problem}), the expected running time of the GSEMO until finding a solution $\bm{x}$ with $|\bm{x}| \leq k$ and $g(\bm{x})-c(\bm{x}) \geq (1-e^{-\gamma}) \cdot g(\bm{x}^*)-c(\bm{x}^*)$ is $O(n^2(\log n+k))$, where $\gamma$ denotes the submodularity ratio of $g$ as in Eq.~(\refeq{eq-submodular}), and $\bm{x}^*$ denotes an optimal solution of Eq.~(\ref{eq-problem}).
\end{theorem}

In the proof, we first derive the expected running time upper bound $O(n^2\log n)$ of the GSEMO until finding the special solution $\bm{0}$, as shown in Lemma~\ref{lemma-first-phase}. The result actually can be applied to any situation where the GSEMO maximizes a bi-objective pseudo-Boolean problem with $(-|\bm{x}|)$ being one objective, and has been used in previous analyses, e.g., Theorem~2 of~\citep{friedrich2015maximizing} and Theorem~1 of~\citep{qian2019maximizing}. Here, we still give the proof for completeness.

\begin{lemma}\label{lemma-first-phase}
For maximizing monotone approximately submodular minus modular functions with a size constraint, i.e., solving the problem Eq.~(\ref{eq-problem}), the expected running time of the GSEMO until finding the all-0s solution $\bm{0}$ is $O(n^2\log n)$.
\end{lemma}
\begin{proof}
According to the procedure of updating the population $P$ in the GSEMO, the solutions maintained in $P$ must be incomparable. Because two solutions with the same value on one objective are comparable, $P$ contains at most one solution for each value of one objective. As $f_2(\bm{x})=-|\bm{x}|$ can take values $0,-1,\ldots,-n$, it holds that $|P| \leq n+1$.

Let $i=\min\{|\bm{x}| \mid \bm{x} \in P\}$ denote the minimum number of 1-bits of the solutions in the population $P$, and $\bm{x}$ denote the corresponding solution, i.e., $|\bm{x}|=i$. First, $i$ will not increase, because solutions with more 1-bits cannot dominate $\bm{x}$. Second, $i$ can decrease in one iteration by selecting $\bm{x}$ in line~4 of Algorithm~\ref{algo:GSEMO} and flipping only one 1-bit of $\bm{x}$ in line~5, occurring with probability $(1/|P|)\cdot (i/n)(1-1/n)^{n-1} \geq i/(en(n+1))$ due to uniform selection and bit-wise mutation. Note that the generated offspring solution $\bm{x}'$ has $(i-1)$ number of 1-bits, and will be included into $P$, implying that $i$ decreases by 1. Thus, the expected running time until $i=0$ (i.e., finding the all-0s vector) is at most $\sum^{n}_{i=1} en(n+1)/i=O(n^2\log n)$.
\end{proof}

After finding the all-0s solution $\bm{0}$, we analyze the expected running time of the GSEMO until finding a solution with the desired approximation guarantee. This proof part is inspired by the analysis of the distorted greedy algorithm in~\citep{harshaw2019submodular}, and relies on Lemma~\ref{lemma-one-step}, that for any $\bm{x} \in \{0,1\}^n$ with $|\bm{x}|<k$, there always exists one item, whose inclusion can improve the objective $f_1$ by at least some quantity relating to $g(\bm{x}^*)$ and $c(\bm{x}^*)$.

\begin{lemma}\label{lemma-one-step}
For any $\bm{x} \in \{0,1\}^n$ with $|\bm{x}| < k$, there exists one item $v \notin \bm{x}$ such that
\begin{align}\label{eq-mid-4}
& f_1(\bm{x} \cup \{v\})-f_1(\bm{x}) \geq \frac{\gamma}{k} \left(1-\frac{\gamma}{k}\right)^{k-|\bm{x}|-1}g(\bm{x}^*)+\frac{1}{k}(c(\bm{1})-c(\bm{x}^*)),
\end{align}
where $f_1$ is defined in Eq.~(\refeq{eq-bi-problem}), and $k$ is the size constraint.
\end{lemma}
\begin{proof}
Let $v^*=\arg \max_{v \in V \setminus \bm{x}} (1-\gamma/k)^{k-|\bm{x}|-1}(g(\bm{x} \cup \{v\})-g(\bm{x}))-c(\{v\})$. Due to the existence of $k$ dummy items and $|\bm{x}|<k$, it holds that $(1-\gamma/k)^{k-|\bm{x}|-1}(g(\bm{x} \cup \{v^*\})-g(\bm{x}))-c(\{v^*\}) \geq 0$. Note that $|\bm{x}^*| \leq k$. Thus, we have
\begin{align}
&k \cdot \left(\left(1-\frac{\gamma}{k}\right)^{k-|\bm{x}|-1}(g(\bm{x} \cup \{v^*\})-g(\bm{x}))-c(\{v^*\})\right)\\
&\geq |\bm{x}^*| \cdot \left(\left(1-\frac{\gamma}{k}\right)^{k-|\bm{x}|-1}(g(\bm{x} \cup \{v^*\})-g(\bm{x}))-c(\{v^*\})\right)\\
&\geq \sum_{v \in \bm{x}^*} \left(\left(1-\frac{\gamma}{k}\right)^{k-|\bm{x}|-1}(g(\bm{x} \cup \{v\})-g(\bm{x}))-c(\{v\})\right)\\
&=\left(1-\frac{\gamma}{k}\right)^{k-|\bm{x}|-1} \sum_{v \in \bm{x}^*} \left(g(\bm{x} \cup \{v\})-g(\bm{x})\right)-c(\bm{x}^*)\\
&\geq \gamma \left(1-\frac{\gamma}{k}\right)^{k-|\bm{x}|-1} \left(g(\bm{x} \cup \bm{x}^*)-g(\bm{x})\right)-c(\bm{x}^*)\\
&\geq \gamma \left(1-\frac{\gamma}{k}\right)^{k-|\bm{x}|-1} \left(g(\bm{x}^*)-g(\bm{x})\right)-c(\bm{x}^*),
\end{align}
where the second inequality holds by the definition of $v^*$ and $\forall v \in \bm{x}: (1-\gamma/k)^{k-|\bm{x}|-1}(g(\bm{x} \cup \{v\})-g(\bm{x}))-c(\{v\}) =-c(\{v\})\leq 0$, the equality holds by the modularity of $c$, the third inequality holds by the definition of $\gamma$ in Eq.~(\ref{eq-submodular}) and the monotonicity of $g$, and the last inequality holds by $\bm{x}^* \subseteq \bm{x} \cup \bm{x}^*$ and the monotonicity of $g$. This implies that
\begin{align}\label{eq-mid-1}
&\left(1-\frac{\gamma}{k}\right)^{k-|\bm{x}|-1}(g(\bm{x} \cup \{v^*\})-g(\bm{x}))-c(\{v^*\})\\
&\geq \frac{\gamma}{k} \left(1-\frac{\gamma}{k}\right)^{k-|\bm{x}|-1} \left(g(\bm{x}^*)-g(\bm{x})\right)-\frac{1}{k}c(\bm{x}^*).
\end{align}

According to the definition of $f_1$ in Eq.~(\refeq{eq-bi-problem}), we have
\begin{align}
& f_1(\bm{x} \cup \{v^*\})-f_1(\bm{x}) \\
&= \left(1-\frac{\gamma}{k}\right)^{k-|\bm{x}\cup \{v^*\}|}g(\bm{x} \cup \{v^*\})-c(\bm{x} \cup \{v^*\})+\frac{|\bm{x}\cup \{v^*\}|}{k}c(\bm{1})\\
&\quad - \left(\left(1-\frac{\gamma}{k}\right)^{k-|\bm{x}|}g(\bm{x})-c(\bm{x})+\frac{|\bm{x}|}{k}c(\bm{1})\right)\\
&=\left(1\!-\!\frac{\gamma}{k}\right)^{k-|\bm{x}|-1}\!\!(g(\bm{x} \cup \{v^*\})-g(\bm{x}))-c(\{v^*\})+\frac{1}{k}c(\bm{1})+\frac{\gamma}{k}\left(1\!-\!\frac{\gamma}{k}\right)^{k-|\bm{x}|-1}\!\!g(\bm{x})\\
&\geq \frac{\gamma}{k}\left(1-\frac{\gamma}{k}\right)^{k-|\bm{x}|-1}g(\bm{x}^*) +\frac{1}{k}(c(\bm{1})-c(\bm{x}^*)),
\end{align}
where the second equality holds by $|\bm{x}\cup \{v^*\}|=|\bm{x}|+1$ and the modularity of $c$, and the inequality holds by Eq.~(\refeq{eq-mid-1}). Thus, the lemma holds.
\end{proof}

\begin{myproof}{Theorem~\ref{theo-main}}
After finding the solution $\bm{0}$, it will always be kept in the population $P$. This is because $\bm{0}$ has the largest $f_2$ value (i.e., $f_2(\bm{0})=0$), and no solution can weakly dominate it. To analyze the expected running time until reaching the desired approximation guarantee, we consider a quantity $J_{\max}$, which is defined as
\begin{align}
&J_{\max}=\max\left\{j \in \{0,1,\ldots,k\} \mid \exists \bm{x} \in P: |\bm{x}| \leq j  \right.\\
& \qquad \quad \left. \wedge \; f_1(\bm{x}) \geq \left(1-\frac{\gamma}{k}\right)^{k-j} \left(1-\left(1-\frac{\gamma}{k}\right)^j\right) \cdot g(\bm{x}^*)+\frac{j}{k}(c(\bm{1})-c(\bm{x}^*))\right\}.
\end{align}
It can be seen that $J_{\max}=k$ implies that there exists one solution $\bm{x}$ in $P$ satisfying that $|\bm{x}| \leq k$ and
\begin{align}\label{eq-mid-2}
f_1(\bm{x}) &\geq \left(1-\frac{\gamma}{k}\right)^{k-k} \left(1-\left(1-\frac{\gamma}{k}\right)^k\right) \cdot g(\bm{x}^*)+\frac{k}{k}(c(\bm{1})-c(\bm{x}^*))\\
&\geq (1-e^{-\gamma})\cdot g(\bm{x}^*)+c(\bm{1})-c(\bm{x}^*).
\end{align}
According to the definition of $f_1$ in Eq.~(\refeq{eq-bi-problem}) and $|\bm{x}| \leq k$, we have
\begin{align}\label{eq-mid-3}
f_1(\bm{x})&=\left(1-\frac{\gamma}{k}\right)^{k-|\bm{x}|}g(\bm{x})-c(\bm{x})+\frac{|\bm{x}|}{k}c(\bm{1})\\
&\leq g(\bm{x})-c(\bm{x})+c(\bm{1}).
\end{align}
Combining Eqs.~(\refeq{eq-mid-2}) and~(\refeq{eq-mid-3}) leads to
\begin{align}
g(\bm{x})-c(\bm{x}) \geq (1-e^{-\gamma})\cdot g(\bm{x}^*)-c(\bm{x}^*).
\end{align}
Thus, $J_{\max}=k$ implies that there exists one solution $\bm{x}$ in $P$ satisfying that $|\bm{x}| \leq k$ and $g(\bm{x})-c(\bm{x}) \geq (1-e^{-\gamma})\cdot g(\bm{x}^*)-c(\bm{x}^*)$; that is, the desired approximation guarantee is reached. Next, we only need to analyze the expected running time until $J_{\max}=k$.

As the population $P$ contains the solution $\bm{0}$, which satisfies that $|\bm{0}|=0$ and $f_1(\bm{0})=(1-\gamma/k)^kg(\bm{0})\geq 0$, $J_{\max}$ is at least 0. Assume that currently $J_{\max}=i <k$, implying that $P$ contains solutions satisfying that $|\bm{x}|\leq i$ and
\begin{align}\label{eq-mid-5}
f_1(\bm{x}) \geq \left(1-\frac{\gamma}{k}\right)^{k-i} \left(1-\left(1-\frac{\gamma}{k}\right)^i\right) \cdot g(\bm{x}^*)+\frac{i}{k}(c(\bm{1})-c(\bm{x}^*)).
\end{align}
Let $\hat{\bm{x}}$ be the one with the largest $f_1$ value among these solutions, which is actually the solution with size at most $i$ and the largest $f_1$ value in $P$. First, $J_{\max}$ will not decrease. If $\hat{\bm{x}}$ is deleted from $P$ in line~7 of Algorithm~\ref{algo:GSEMO}, the newly included solution $\bm{x}'$ must weakly dominate $\hat{\bm{x}}$, implying that $|\bm{x}'| \leq |\hat{\bm{x}}|$ and $f_1(\bm{x}')\geq f_1(\hat{\bm{x}})$.

Second, we analyze the expected time required to increase $J_{\max}$. We consider such an event in one iteration of Algorithm~\ref{algo:GSEMO}: $\hat{\bm{x}}$ is selected for mutation in line~4, and only one specific 0-bit corresponding to the item $v$ in Lemma~\ref{lemma-one-step} is flipped in line~5. This event is called ``a successful event", occurring with probability $(1/|P|)\cdot (1/n)(1-1/n)^{n-1} \geq 1/(en(n+1))$ due to uniform selection and bit-wise mutation. Note that the size $|P|$ of population is always no larger than $n+1$, as shown in the proof of Lemma~\ref{lemma-first-phase}. According to Lemma~\ref{lemma-one-step}, the offspring solution $\bm{x}'$ generated by a successful event satisfies
\begin{align}\label{eq-mid-6}
& f_1(\bm{x}')-f_1(\hat{\bm{x}}) \geq \frac{\gamma}{k} \left(1-\frac{\gamma}{k}\right)^{k-|\hat{\bm{x}}|-1}g(\bm{x}^*)+\frac{1}{k}(c(\bm{1})-c(\bm{x}^*)).
\end{align}
We next consider two cases according to the value of $|\hat{\bm{x}}|$, satisfying $|\hat{\bm{x}}| \leq i$.

(1) $|\hat{\bm{x}}| = i$. Combining Eqs.~(\refeq{eq-mid-5}) and~(\refeq{eq-mid-6}) leads to
\begin{align}\label{eq-mid-7}
f_1(\bm{x}') &\geq \left(1-\frac{\gamma}{k}\right)^{k-i} \left(1-\left(1-\frac{\gamma}{k}\right)^i\right) \cdot g(\bm{x}^*)+\frac{i}{k}(c(\bm{1})-c(\bm{x}^*))\\
&\quad + \frac{\gamma}{k} \left(1-\frac{\gamma}{k}\right)^{k-i-1}g(\bm{x}^*)+\frac{1}{k}(c(\bm{1})-c(\bm{x}^*))\\
&=\left(1-\frac{\gamma}{k}\right)^{k-i-1} \left(1-\left(1-\frac{\gamma}{k}\right)^{i+1}\right) \cdot g(\bm{x}^*)+\frac{i+1}{k}(c(\bm{1})-c(\bm{x}^*)).
\end{align}
Note that $|\bm{x}'|=|\hat{\bm{x}}|+1 = i+1$. Then, $\bm{x}'$ will be added into $P$; otherwise, $\bm{x}'$ must be dominated by one solution in $P$ (line~6 of Algorithm~\ref{algo:GSEMO}), and this implies that $J_{\max}$ has already been larger than $i$, contradicting the assumption $J_{\max}=i$. After including $\bm{x}'$, $J_{\max} \geq i+1$, i.e., $J_{\max}$ increases.

(2) $|\hat{\bm{x}}| < i$. It holds that $|\bm{x}'|=|\hat{\bm{x}}|+1 \leq i$, and by Eq.~(\refeq{eq-mid-6}),
\begin{align}\label{eq-mid-8}
f_1(\bm{x}')-f_1(\hat{\bm{x}}) \geq \frac{\gamma}{k} \left(1-\frac{\gamma}{k}\right)^{k-1}g(\bm{x}^*)+\frac{1}{k}(c(\bm{1})-c(\bm{x}^*)).
\end{align}
Note that $\bm{x}'$ will be added into $P$; otherwise, $\bm{x}'$ must be dominated by one solution in $P$, contradicting the definition of $\hat{\bm{x}}$, which is the solution with size at most $i$ and the largest $f_1$ value in $P$. If $f_1(\bm{x}') \geq (1-\gamma/k)^{k-i-1} (1-(1-\gamma/k)^{i+1}) \cdot g(\bm{x}^*)+((i+1)/k)(c(\bm{1})-c(\bm{x}^*))$, $J_{\max}$ increases. Otherwise, the solution $\hat{\bm{x}}$ now becomes $\bm{x}'$, and $f_1(\hat{\bm{x}})$ increases by at least $(\gamma/k)(1-\gamma/k)^{k-1}g(\bm{x}^*)+(1/k)(c(\bm{1})-c(\bm{x}^*))$ according to Eq.~(\refeq{eq-mid-8}).

Based on the above analysis, a successful event will either increase $J_{\max}$ directly or increase $f_1(\hat{\bm{x}})$ by at least $(\gamma/k)(1-\gamma/k)^{k-1}g(\bm{x}^*)+(1/k)(c(\bm{1})-c(\bm{x}^*))$. It is easy to see that $f_1(\hat{\bm{x}})$ will not decrease due to the domination-based comparison. It is also known from Eq.~(\refeq{eq-mid-7}) that $f_1(\hat{\bm{x}})$ needs to increase at most $(\gamma/k) (1-\gamma/k)^{k-i-1}g(\bm{x}^*)+(1/k)(c(\bm{1})-c(\bm{x}^*))$ for increasing $J_{\max}$. Thus, the number of successful events required to increase $J_{\max}$ is at most
\begin{align}\label{eq-mid-9}
&\left\lceil\frac{(\gamma/k) (1-\gamma/k)^{k-i-1}g(\bm{x}^*)+(1/k)(c(\bm{1})-c(\bm{x}^*))}{(\gamma/k)(1-\gamma/k)^{k-1}g(\bm{x}^*)+(1/k)(c(\bm{1})-c(\bm{x}^*))}\right\rceil\\ &\leq \left\lceil \frac{(\gamma/k) (1-\gamma/k)^{k-i-1}g(\bm{x}^*)}{(\gamma/k)(1-\gamma/k)^{k-1}g(\bm{x}^*)}\right\rceil=\left\lceil \left(1-\frac{\gamma}{k}\right)^{-i} \right\rceil.
\end{align}
A successful event occurs with probability at least $1/(en(n+1))$ in one iteration, implying that the expected time of one successful event is at most $en(n+1)$. Thus, the expected time to make $J_{\max} \geq i+1$ (i.e., increase $J_{\max}$) is at most $\lceil (1-\gamma/k)^{-i} \rceil \cdot en(n+1)$.

To make $J_{\max}=k$, it is sufficient to increase $J_{\max}$ from $0$ to $k$ step-by-step, implying that the expected running time until $J_{\max}=k$ is at most
\begin{align}
\sum^{k-1}_{i=0}\left\lceil \left(1-\frac{\gamma}{k}\right)^{-i} \right\rceil \cdot en(n+1) &\leq en(n+1) \cdot \sum^{k-1}_{i=0} \left(\left(1-\frac{\gamma}{k}\right)^{-i} +1\right)\\
&\leq en(n+1) \cdot (ek+1)=O(n^2k).
\end{align}

Lemma~\ref{lemma-first-phase} gives the expected running time $O(n^2 \log n)$ of the GSEMO for finding the solution $\bm{0}$. Thus, the total expected running time of the GSEMO for finding a solution $\bm{x}$ with $|\bm{x}| \leq k$ and $g(\bm{x})-c(\bm{x}) \geq (1-e^{-\gamma})\cdot g(\bm{x}^*)-c(\bm{x}^*)$ is $O(n^2(\log n+k))$. The theorem holds.\vspace{0.5em}
\end{myproof}

From the proof, we can find the reason of adding $(|\bm{x}|/k)\cdot c(\bm{1})$ into $f_1$ in Eq.~(\ref{eq-bi-problem}). The term $(|\bm{x}|/k)\cdot c(\bm{1})$ can increase the benefit of adding a single item, and make the derived lower bound of the benefit in Lemma~\ref{lemma-one-step} positive. Without this term, the lower bound of the benefit in Lemma~\ref{lemma-one-step} will become $(\gamma/k)(1-\gamma/k)^{k-|\bm{x}|-1}g(\bm{x}^*)-c(\bm{x}^*)/k \geq (\gamma/k)(1-\gamma/k)^{k-1}g(\bm{x}^*)-c(\bm{x}^*)/k$, which is not necessarily positive, and the required increment on $f_1$ for $J_{\max}$ increasing from $i$ to $(i+1)$ will become $(\gamma/k) (1-\gamma/k)^{k-i-1}g(\bm{x}^*)-c(\bm{x}^*)/k$; thus, the analysis of Eq.~(\refeq{eq-mid-9}) will fail.

\subsection{GSEMO Fails When $f_1=g-c$}

We have proved that by using the distorted objective function, i.e., $f_1(\bm{x}) = (1-\frac{\gamma}{k})^{k-|\bm{x}|}g(\bm{x})-c(\bm{x})+\frac{|\bm{x}|}{k}c(\bm{1})$, the GSEMO can find a solution $\bm{x}$ with $|\bm{x}| \leq k$ and $g(\bm{x})-c(\bm{x}) \geq (1-e^{-\gamma}) \cdot g(\bm{x}^*)-c(\bm{x}^*)$, in $O(n^2(\log n+k))$ expected running time. A natural question is whether the GSEMO using the original objective function (i.e., $f_1(\bm{x})=g(\bm{x})-c(\bm{x})$) can obtain the same polynomial-time approximation guarantee. To answer this question, we first introduce the application, i.e., directed vertex cover with costs, of the considered general problem in Definition~\ref{def-Prob}.

\begin{definition}[Directed Vertex Cover with Costs~\citep{harshaw2019submodular}]\label{def-vertex}
Given a directed graph $G=(V,E)$ with non-negative vertex weights $w: V \rightarrow \mathbb{R}^+$ and costs $c: V \rightarrow \mathbb{R}^+$, and a budget $k$, to find a subset $X \subseteq V$ of at most $k$ vertices such that
\begin{align}
\arg\max\nolimits_{X \subseteq V}\quad g(X)-c(X) \qquad \text{s.t.}\qquad |X|\leq k,
\end{align}
where $g(X)=\sum_{v \in N(X) \cup X} w(v)$, $N(X)=\{(u,v) \in E \mid u \in X\}$ is the set of vertices pointed to by $X$, and $c(X)=\sum_{v \in X} c(v)$.
\end{definition}

It is easy to verify that $g$ is non-negative, monotone and submodular, i.e., $\gamma=1$. Note that this problem is actually the dual norm of vertex cover, but we still use this name for consistency with~\citep{harshaw2019submodular}.

Next, we give a negative answer by proving that the GSEMO with $f_1=g-c$ requires exponential expected running time to achieve the desired approximation guarantee on a specific example of directed vertex cover with costs.

\begin{example}\label{example-cover}
The parameters of directed vertex cover with costs in Definition~\ref{def-vertex} are set as: the directed graph $G=(V,E)$ is shown in Figure~\ref{fig_graph}, the weights satisfy $\forall v \in V: w(v)=1$, the costs satisfy $c(v_1)=n-3n/\log n$, $\forall i \geq 2: c(v_i)=1/\log n$, and the budget $k=n$, where the base of the logarithm is 2.
\end{example}

\begin{figure}[t!]\centering
\begin{minipage}[c]{0.3\linewidth}\centering
  \includegraphics[width=\linewidth]{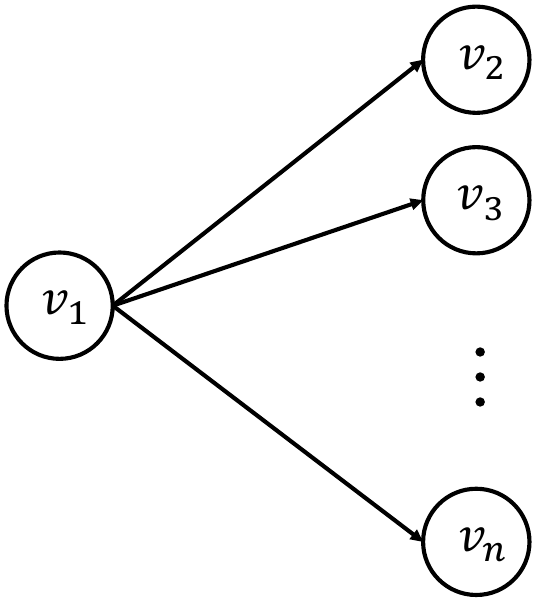}
\end{minipage}
\caption{An example graph of directed vertex cover with costs.}\label{fig_graph}
\end{figure}

The objective function $(g(X)-c(X))$ of this example can be calculated as follows:
\begin{align}\label{example-f}
\forall X \subseteq V,\; g(X)-c(X)=
\begin{cases}
(3n-|X|+1)/\log n  & \text{if}\;\; v_1 \in X,\\
|X|\cdot (1-1/\log n) & \text{otherwise}.
\end{cases}
\end{align}
Thus, when $v_1 \in X$, the best solution is $\{v_1\}$, which has the objective value $3n/\log n$; when $v_1 \notin X$, the best solution is $V \setminus \{v_1\}$, which has the objective value $(n-1)(1-1/\log n)$. Assume that $n \geq 64$. We then have $(n-1)(1-1/\log n) > 3n/\log n$, implying that the optimal solution $\bm{x}^*$ is $V \setminus \{v_1\}$.

Theorem~\ref{theo-example} shows that for this example, the GSEMO using $f_1=g-c$ fails to find a solution $\bm{x}$ with $|\bm{x}| \leq k$ and $g(\bm{x})-c(\bm{x}) \geq (1-e^{-\gamma}) \cdot g(\bm{x}^*)-c(\bm{x}^*)$ in polynomial expected time. The proof idea is that the GSEMO has some probability to start from the specific solution $\{v_1\}$, which requires to be flipped by many bits simultaneously in mutation for making improvement, and thus leads to exponential running time.

\begin{theorem}\label{theo-example}
For Example~\ref{example-cover} with $n \geq 64$, when using $f_1=g-c$, the expected running time until the GSEMO finding a solution $\bm{x}$ with $|\bm{x}| \leq k$ and $g(\bm{x})-c(\bm{x}) \geq (1-e^{-\gamma}) \cdot g(\bm{x}^*)-c(\bm{x}^*)$ is at least $2^n$.
\end{theorem}
\begin{proof}
For convenience of analysis, assume that $3n/ \log n$ is an integer. Consider that the initial solution of the GSEMO is $\{v_1\}$, occurring with probability $1/2^n$ due to uniform selection. By Eq.~(\refeq{example-f}), $f_1(\{v_1\})=3n/\log n$, and for any $X$ with $X \neq \{v_1\} \wedge |X| \leq 3n/\log n$, $f_1(X) < f_1(\{v_1\})$. This implies that in the bi-objective formulation $(f_1(X)=g(X)-c(X), f_2(X)=-|X|)$, the solution $\{v_1\}$ dominates any solution with size no larger than $3n/\log n$, except the empty solution $\bm{0}$ and $\{v_1\}$ itself. Thus, to generate a solution $X$ with $f_1(X)=g(X)-c(X) > 3n/ \log n $, it is necessary to flip at least $3n/\log n$ bits simultaneously when mutating a solution in line~5 of the GSEMO. As the probability of flipping at least $3n/\log n$ bits simultaneously in mutation is at most $\binom{n}{3n/\log n}/ n^{3n/\log n} \leq 2^n/2^{3n}=1/2^{2n}$, the expected running time of the GSEMO until finding a solution $X$ with $g(X)-c(X) > 3n/ \log n $, when starting from $\{v_1\}$, is at least $2^{2n}$. Combining the probability $1/2^n$ of the initial solution being $\{v_1\}$, the expected running time of the GSEMO until finding a solution $X$ with $g(X)-c(X) > 3n/ \log n$ is at least $2^n$.

As the optimal solution $\bm{x}^*$ is $V \setminus \{v_1\}$, we have $g(\bm{x}^*)=n-1$ and $c(\bm{x}^*)=(n-1)/\log n$. Thus, $3n / \log n = \Theta(1/\log n) \cdot g(\bm{x}^*) -c(\bm{x}^*)$. Because $\gamma =1$ for the application of directed vertex cover with costs, it must require more running time for the GSEMO to find a solution $X$ with $g(X)-c(X) \geq (1-e^{-\gamma}) \cdot g(\bm{x}^*)-c(\bm{x}^*)$ than $g(X)-c(X) > 3n/ \log n$. Thus, the theorem holds.
\end{proof}

The above analysis relies on a specific initialization with the solution $10^{n-1}$, i.e., $\{v_1\}$. Though the occurring probability is only $1/2^n$, it is sufficient to prove the exponential expected running time of the GSEMO. We also note that with a constant probability, the initial solution $\bm{x}$ satisfies that $x_1=0$ (i.e., $v_1$ is not selected) and $|\bm{x}|>3n/((\log n) -1)$. It is known from Eq.~(\refeq{example-f}) that for any $\bm{x}$ with $x_1=0$ and $|\bm{x}|>3n/((\log n) -1)$, $g(\bm{x})-c(\bm{x}) = |\bm{x}|\cdot (1-1/\log n)> 3n / \log n$, implying that such an $\bm{x}$ is incomparable with $10^{n-1}$ (which has $g(10^{n-1})-c(10^{n-1}) =3n / \log n$) in the bi-objective reformulation. In fact, any solution $\bm{x}$ with $x_1=0$ and $|\bm{x}|>3n/((\log n) -1)$ is Pareto optimal, which cannot be dominated by other solutions. Let $\hat{\bm{x}}$ denote the solution with the largest number of 1-bits in the population. Thus, in this case, the GSEMO can find Pareto optimal solutions with more 1-bits by selecting $\hat{\bm{x}}$ and flipping only some of its 0-bits except the first one. By repeating this process, the GSEMO can efficiently find a solution $\bm{x}$ with $x_1=0$ and $|\bm{x}| \geq (1-1/e)(n-1)$, satisfying $g(\bm{x})-c(\bm{x})=|\bm{x}|\cdot (1-1/\log n) \geq (1-1/e)(n-1)-(n-1)/\log n =(1-e^{-\gamma}) \cdot g(\bm{x}^*)-c(\bm{x}^*)$, i.e., reaching the desired approximation guarantee.

\section{Empirical Study}\label{sec-experiment}

In this section, we empirically examine the performance of the GSEMO by comparing it with the following algorithms:\\
$\bullet$\ \ {\bf DG}~\citep{harshaw2019submodular} is the distorted greedy algorithm, as presented in Algorithm~\ref{alg:dg}. It is the previous algorithm, achieving the best known polynomial-time approximation guarantee.\\
$\bullet$\ \ {\bf SDG}~\citep{harshaw2019submodular} is the stochastic version of DG, as presented in Algorithm~\ref{alg:sdg}. The SDG with $\epsilon\in \{0.1,0.2\}$ will be compared, which are denoted by SDG(0.1) and SDG(0.2), respectively.\\
$\bullet$\ \ {\bf Multi-SDG} runs the SDG multiple times independently, and returns the best found solution. For each independent run of the SDG, the value of parameter $\epsilon$ is uniformly sampled from $[0.1,0.5]$ at random.\\
$\bullet$\ \ {\bf GSEMO$_{\bm{g-c}}$} is similar to the GSEMO. The only difference is the setting of $f_1$. The GSEMO$_{g-c}$ sets $f_1(\bm{x})$ as the original objective function $g(\bm{x})-c(\bm{x})$, while the GSEMO adopts the distorted objective function, i.e., $f_1(\bm{x}) = (1-\frac{\gamma}{k})^{k-|\bm{x}|}g(\bm{x})-c(\bm{x})+\frac{|\bm{x}|}{k}c(\bm{1})$.\\
$\bullet$\ \ {\bf NSGA-II} is similar to the GSEMO except the employed MOEA. Here, the popular NSGA-II~\citep{deb2002fast} instead of the GSEMO is employed to optimize the reformulated bi-objective problem Eq.~(\ref{eq-bi-problem}).

Note that when implementing the MOEA-based algorithms (i.e., GSEMO, GSEMO$_{g-c}$ and NSGA-II), the solutions with size at least $(k+3)$ (i.e., the overly infeasible solutions) are excluded by setting their first objective $f_1$ to $-\infty$, in order to improve the efficiency. The number of iterations of the GSEMO is set to $\lceil ek^2n \rceil$. The GSEMO performs one function evaluation in each iteration. For the fairness of comparison, the Multi-SDG, GSEMO$_{g-c}$ and NSGA-II run until the number of function evaluations reaches $\lceil ek^2n \rceil$, so that the same computational budget is used. For the NSGA-II, we employ the bit-wise mutation and uniform crossover operators, and use the parameter setting: the population size is 100, and the probabilities of performing mutation and crossover in each iteration are 0.1 and 1.0, respectively. The 100 initial solutions of the NSGA-II are randomly generated. Note that the initial solution of the GSEMO and GSEMO$_{g-c}$ is set to the all-0s solution $\bm{0}$ in the experiments.

We compare these algorithms on two applications of the considered problem Eq.~(\ref{eq-problem}): Bayesian experimental design with costs, and directed vertex cover with costs, where the function $g$ is approximately and exactly submodular, respectively. Because all the algorithms except DG are randomized algorithms, we repeat the run 20 times independently, and report the average $(g-c)$ values and the standard deviation.

\subsection{Bayesian Experimental Design with Costs}

Let $\mathbf{V}=[\bm{v}_1,\bm{v}_2,\ldots,\bm{v}_n] \in \mathbb{R}^{d\times n}$ denote a measurement matrix, and $\mathbf{V}_X \in \mathbb{R}^{d \times |X|}$ denote the submatrix of $\mathbf{V}$ with its columns indexed by $X \subseteq \{1,2,\ldots,n\}$. In Bayesian experimental design, the goal is to select measurements $\mathbf{V}_X$ to maximize the quality of parameter estimation.~\cite{krause2008near} considered the Bayesian A-optimality objective function, in order to maximally reduce the variance of the posterior distribution over parameters in linear models, i.e., $\bm{y}_X=\mathbf{V}^{\mathrm{T}}_{X} \bm{\theta}+\bm{\zeta}_{X}$, where $\bm{\theta} \in \mathbb{R}^d$ and $\bm{\zeta}_{X}\in \mathbb{R}^{|X|}$ are the parameter and noise vectors, respectively. Each measurement $\bm{v}_i$ corresponds to one experiment, which can be performed to obtain a noisy linear observation, but also introduces a cost $c_i$. To have a low cost,~\cite{harshaw2019submodular} considered the problem of Bayesian experimental design with costs, presented as follows.

\begin{definition}[Bayesian Experimental Design with Costs~\citep{harshaw2019submodular}]\label{def-bayesian}
Given a measurement matrix $\mathbf{V}=[\bm{v}_1,\bm{v}_2,\ldots,\bm{v}_n] \in \mathbb{R}^{d\times n}$, a linear model $\bm{y}_X=\mathbf{V}^{\mathrm{T}}_{X} \bm{\theta}+\bm{\zeta}_{X}$, costs $c_1,c_2,\ldots,c_n$, and a budget $k$, where $\bm{\theta}$ has a Gaussian prior distribution $\bm{\theta} \sim \mathcal{N}(0,\mathbf{\Sigma})$, the Gaussian i.i.d. noise $\bm{\zeta}_{X} \sim \mathcal{N}(0,\sigma^2\mathbf{I}_{|X|})$, and $\mathbf{I}_{|X|}$ denotes the identity matrix of size $|X|$, to find a submatrix $\mathbf{V}_{X}$ of at most $k$ columns such that
\begin{align}
\arg\max\nolimits_{X \subseteq \{1,2,\ldots,n\}}\quad g(X)-c(X) \qquad \text{s.t.}\qquad |X|\leq k,
\end{align}
where $g(X)={\rm tr}(\mathbf{\Sigma})-{\rm tr}((\mathbf{\Sigma}^{-1}+\sigma^{-2}\mathbf{V}_X\mathbf{V}^{\mathrm{T}}_X)^{-1})$ is the Bayesian A-optimality function, $\rm{tr}(\cdot)$ denotes the trace of a matrix, and $c(X)=\sum_{i \in X}c_i$ is the cost function.
\end{definition}

It has been shown~\citep{harshaw2019submodular} that $g$ is non-negative, monotone, and approximately submodular with $\gamma \geq (1+(s^2/\sigma^2)\lambda_{\max}(\mathbf{\Sigma}))^{-1}$, where $s=\max_{i \in \{1,2,\ldots,n\}} \|\bm{v}_i\|_2$, and $\lambda_{\max}(\cdot)$ denotes the largest eigenvalue of a square matrix. Note that as the exact computation of $\gamma$ is difficult, this lower bound of $\gamma$ will be used in the implementation of the compared algorithms.

We use two data sets\footnote{\url{http://www.csie.ntu.edu.tw/~cjlin/libsvmtools/datasets/}}: \textit{housing} and \textit{segment}. The former has 506 instances and 14 features, i.e., $n=506$ and $d=14$, and the latter has 2,310 instances and 19 features, i.e., $n=2,310$ and $d=19$. Each feature vector is normalized to have mean 0 and variance 1. The covariance matrix $\mathbf{\Sigma}$ of the Gaussian prior distribution of $\bm{\theta}$ is set to $\mathbf{A}\mathbf{D}\mathbf{A}^{\mathrm{T}}$, where each entry of $\mathbf{A}$ is randomly drawn from the standard Gaussian distribution $\mathcal{N}(0,1)$ and $\mathbf{D}$ is a diagonal matrix with the $i$-th entry on the diagonal equal to $(i/d)^2$. The cost $c_i$ is set to $\alpha \cdot g(\{i\})$, where $\alpha=0.8$.

For the \textit{housing} data set, we use $\sigma\in \{3d,4d,7d\}$ to generate three instances with the lower bound of $\gamma$ equal to $0.456$, $0.567$ and $0.800$, respectively, representing different approximately submodular degrees. For the \textit{segment} data set, we use $\sigma\in \{6d,11d,20d\}$ to generate three instances with the lower bound of $\gamma$ equal to $0.281$, $0.621$ and $0.840$, respectively. The budget $k$ is set from 5 to 20. Note that~\cite{harshaw2019submodular} also used the \textit{housing} data set in their experiments, and our experimental setting is similar to theirs, except that we use different values of $\sigma$ in order to generate instances with different approximately submodular degrees.

\begin{figure*}[t!]\centering
\begin{minipage}[c]{0.8\linewidth}\centering
        \includegraphics[width=\linewidth]{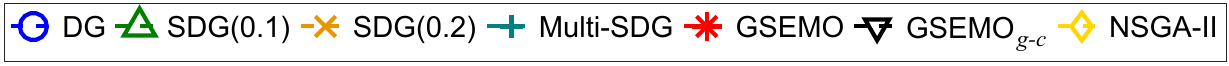}
\end{minipage}\\\vspace{0.2em}
\begin{minipage}[c]{0.32\linewidth}\centering
        \includegraphics[width=1\linewidth]{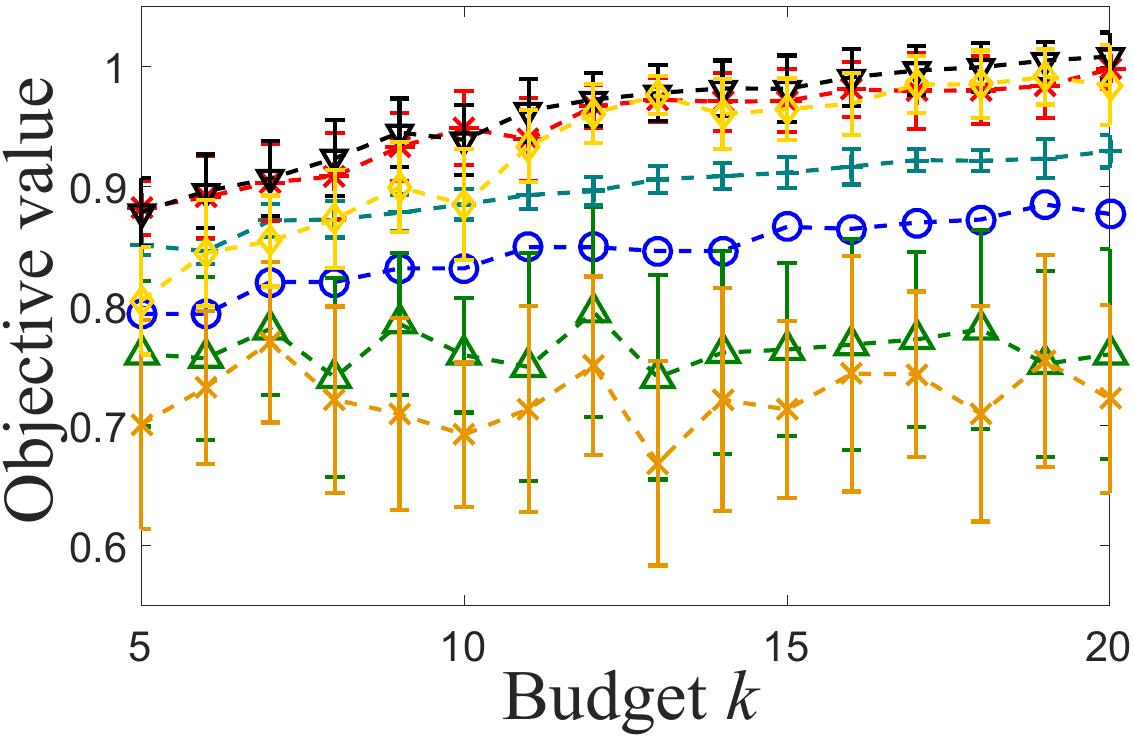}
\end{minipage}
\begin{minipage}[c]{0.32\linewidth}\centering
        \includegraphics[width=1\linewidth]{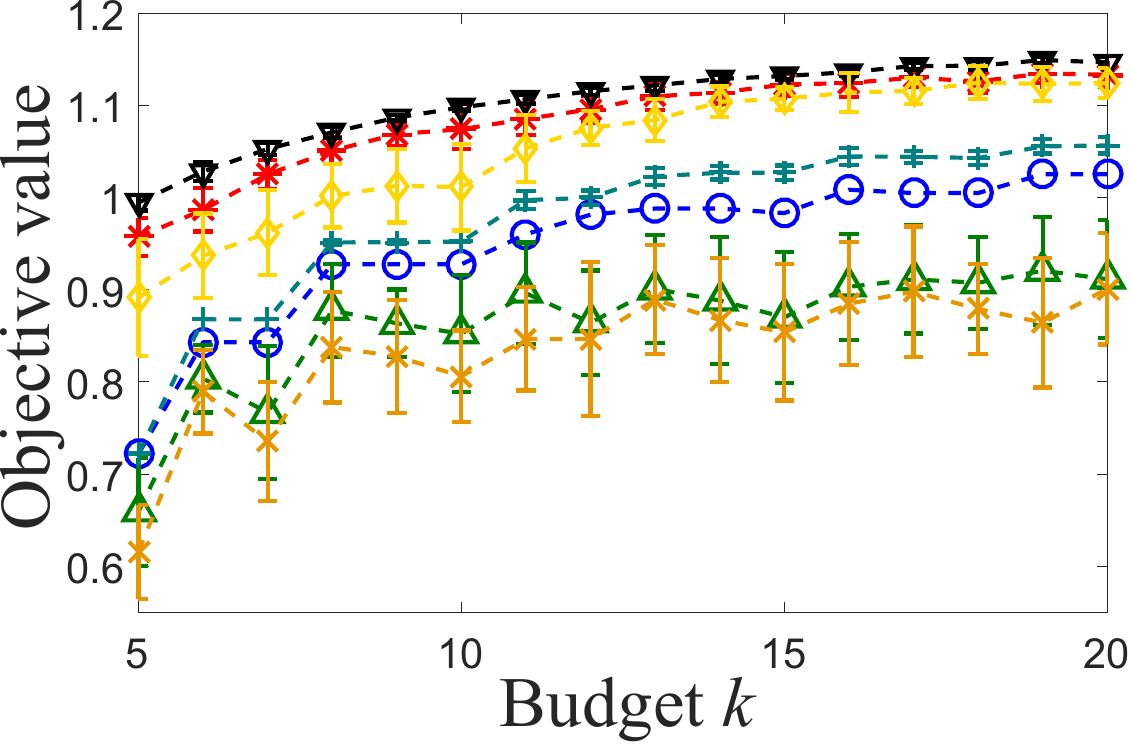}
\end{minipage}
\begin{minipage}[c]{0.32\linewidth}\centering
        \includegraphics[width=1\linewidth]{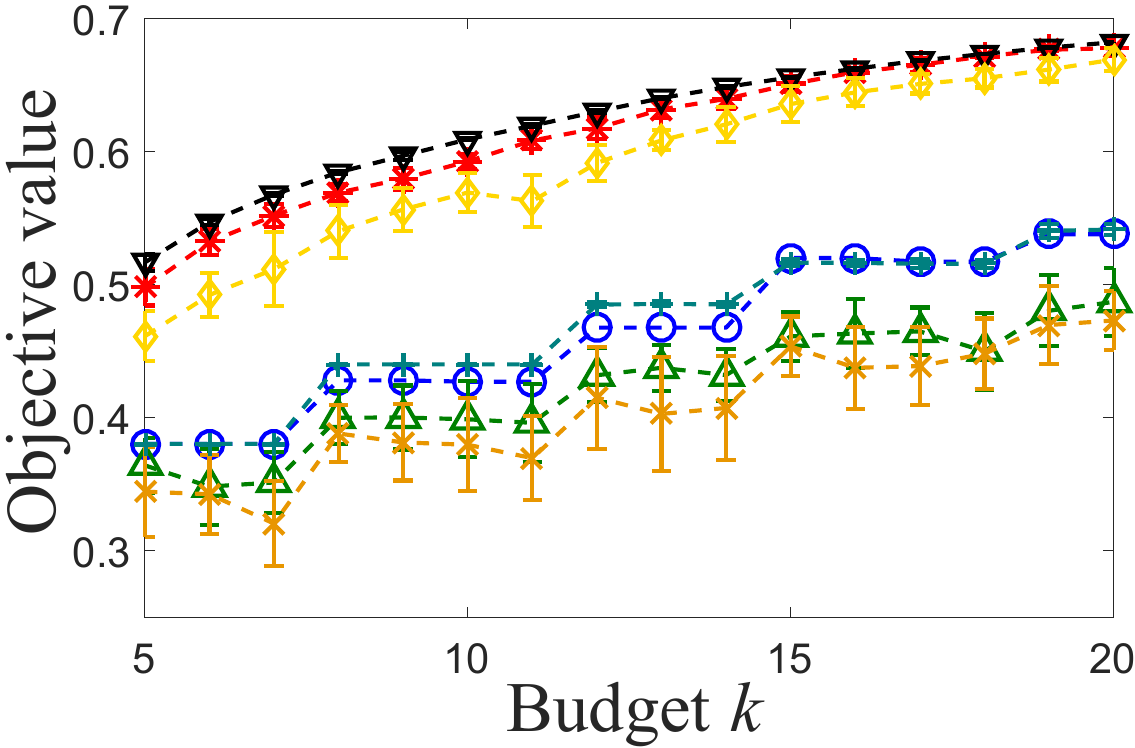}
\end{minipage}\\\vspace{0.3em}
\begin{minipage}[c]{0.32\linewidth}\centering
    \small(a) $\gamma=0.456$
\end{minipage}
\begin{minipage}[c]{0.32\linewidth}\centering
    \small(b) $\gamma=0.567$
\end{minipage}
\begin{minipage}[c]{0.32\linewidth}\centering
    \small(c) $\gamma=0.800$
\end{minipage}\vspace{-0.3em}
\caption{Comparison on the application of Bayesian experimental design with costs. The data set \textit{housing} (506\#inst, 14\#feat) is used to generate three problem instances with different values of $\gamma$.}\label{fig_BO_housing}
\end{figure*}

\begin{figure*}[t!]\centering
\begin{minipage}[c]{0.8\linewidth}\centering
        \includegraphics[width=\linewidth]{figures/legend_horizontal}
\end{minipage}\\\vspace{0.2em}
\begin{minipage}[c]{0.32\linewidth}\centering
        \includegraphics[width=1\linewidth]{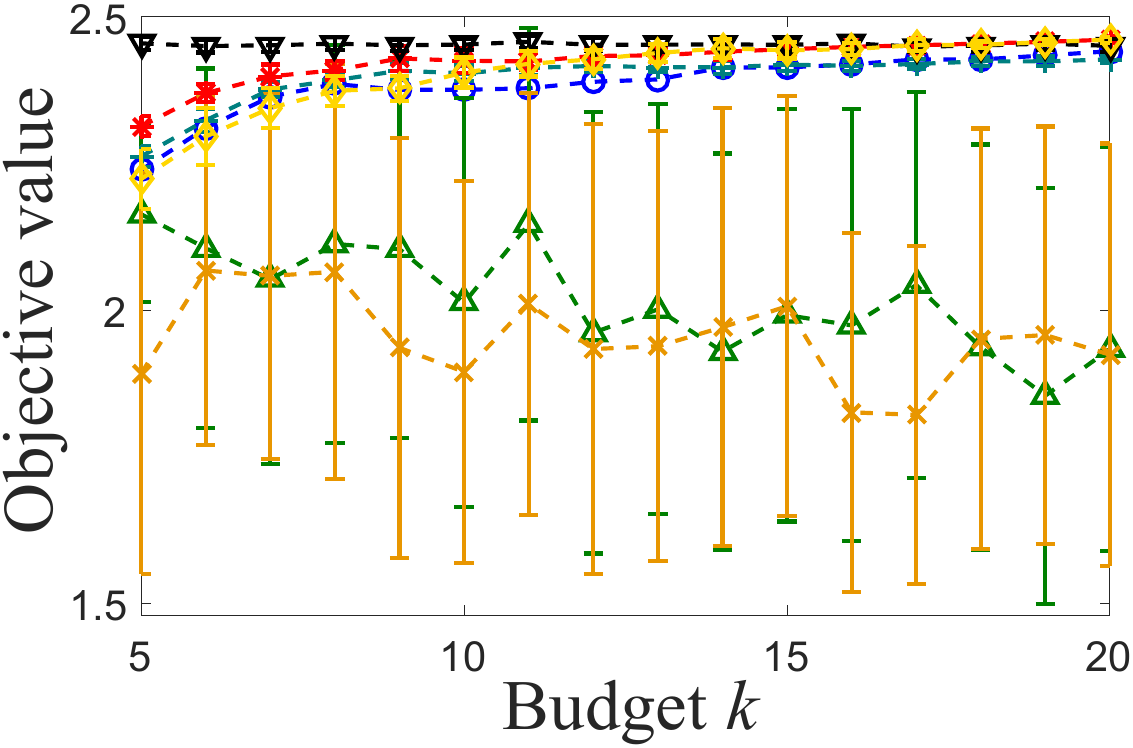}
\end{minipage}
\begin{minipage}[c]{0.32\linewidth}\centering
        \includegraphics[width=1\linewidth]{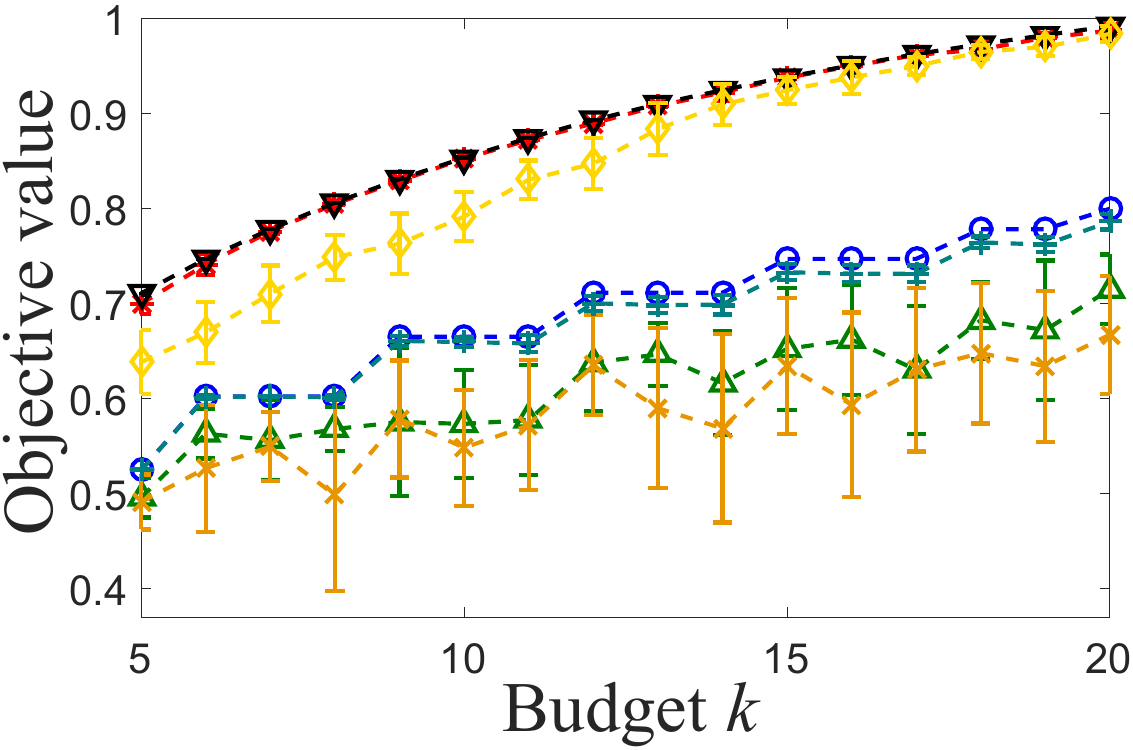}
\end{minipage}
\begin{minipage}[c]{0.32\linewidth}\centering
        \includegraphics[width=1\linewidth]{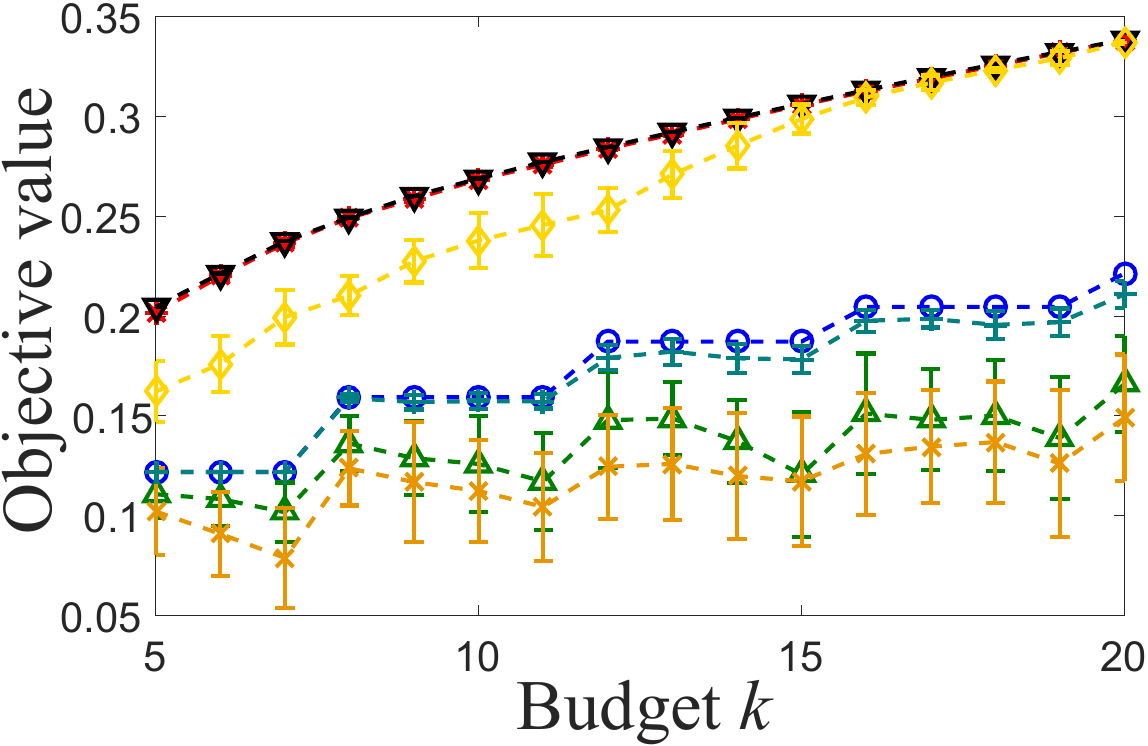}
\end{minipage}\\\vspace{0.3em}
\begin{minipage}[c]{0.32\linewidth}\centering
    \small(a) $\gamma=0.281$
\end{minipage}
\begin{minipage}[c]{0.32\linewidth}\centering
    \small(b) $\gamma=0.621$
\end{minipage}
\begin{minipage}[c]{0.32\linewidth}\centering
    \small(c) $\gamma=0.840$
\end{minipage}\vspace{-0.3em}
\caption{Comparison on the application of Bayesian experimental design with costs. The data set \textit{segment} (2,310\#inst, 19\#feat) is used to generate three problem instances with different values of $\gamma$.}\label{fig_BO_segment}
\end{figure*}

The results are plotted in Figures~\ref{fig_BO_housing} and~\ref{fig_BO_segment}. Note that the standard deviation of some algorithms (e.g., the GSEMO in Figure~\ref{fig_BO_housing}(b)) can be very small, which is because almost the same good solutions are found in 20 runs of one algorithm. As expected, the DG is better than the SDG, and the SDG becomes worse as $\epsilon$ increases. This is because to decrease the running time, a smaller set of candidate items is examined for selection in each iteration of the SDG, which may degrade the performance. The Multi-SDG is also better than the SDG, because the Multi-SDG outputs the best solution found in multiple independent runs of the SDG.

It can be clearly observed that the MOEA-based algorithms (i.e., GSEMO, GSEMO$_{g-c}$ and NSGA-II) are better than the greedy-style algorithms (i.e., DG, SDG(0.1), SDG(0.2) and Multi-SDG), showing the advantage of bi-objective reformulation. Compared with the GSEMO, the NSGA-II performs worse in most cases. As the solutions with size at least $(k+3)$ are excluded in the optimization process, the population can contain at most $(k+3)$ non-dominated solutions, corresponding to the subset size $\{0,1,\ldots,k+2\}$, respectively. Note that any two solutions with the same size are comparable. We test the budget $k$ from 5 to 20, while the population size of the NSGA-II is 100. This implies that the population of the NSGA-II may contain many redundant dominated solutions, thus leading to the bad performance. The GSEMO will not encounter this issue, because it always contains only non-dominated solutions generated-so-far. Another possible reason for the inferior performance of the NSGA-II is the relatively low probability (which is 0.1) of employing mutation in each iteration. Compared with the GSEMO$_{g-c}$ using the original objective function $(g-c)$ as $f_1$, the GSEMO using the distorted objective function as $f_1$ does not show the advantage, and can even be a little worse sometimes. Note that this does not contradict the theoretical analysis, because the GSEMO$_{g-c}$ may not encounter the worst-case examples in practice.

\begin{table*}[t!]\centering
\caption{The objective values (mean+std.) of the compared algorithms on the application of Bayesian experimental design with costs, using the data set \textit{housing} with $k=15$, $\gamma \in \{0.456, 0.567, 0.800\}$ and $\alpha \in \{0,0.1,\ldots,1\}$. For each $\alpha$, the largest values are bolded. In the rows of ``\#best", the largest values are bolded. The average rank of each algorithm on all $\alpha$ is computed, where the smallest values are bolded. The ``\#direct win" denotes the number of $\alpha$ on which the GSEMO has a larger objective value than the corresponding algorithm ($1$ tie is counted as $0.5$ win), where significant cells by the sign-test~\citep{demsar:06} with confidence level $0.05$ are bolded.\vspace{6pt}}\label{table_bayesian_housing}
\small
\newsavebox{\tablebox}
\begin{lrbox}{\tablebox}
\begin{tabular}{c|r|rrrr|rr}
\hline
\multicolumn{8}{c}{}\\[-8pt]
\multicolumn{8}{c}{$\gamma=0.456$}\\
\hline
&&&&&&&\\[-8pt]
$\alpha$ & \multicolumn{1}{c|}{GSEMO} & \multicolumn{1}{c}{DG} & \multicolumn{1}{c}{SDG(0.1)} & \multicolumn{1}{c}{SDG(0.2)} & \multicolumn{1}{c|}{Multi-SDG} & \multicolumn{1}{c}{GSEMO$_{g-c}$} & \multicolumn{1}{c}{NSGA-II} \\
\hline
&&&&&&&\\[-8pt]
0    &\bf{12.894$\pm$0.000} &\bf{12.894}  &12.378$\pm$0.285  &12.068$\pm$0.318        &  12.578$\pm$0.188&\bf{12.894$\pm$0.000}   &  12.872$\pm$0.025\\
0.1  &\bf{10.786$\pm$0.000} &\bf{10.786}  & 10.429$\pm$0.160 & 10.224$\pm$0.230       &  10.621$\pm$0.094&\bf{10.786$\pm$0.000}   &  10.773$\pm$0.018\\
0.2  &\bf{8.761$\pm$0.000}  &\bf{8.761}   &8.532$\pm$0.094   & 8.304$\pm$0.178        & 8.628$\pm$0.075 &\bf{8.761$\pm$0.000}    &  8.756$\pm$0.009 \\
0.3  &\bf{6.877$\pm$0.000}  &\bf{6.877}   &6.648$\pm$0.109   & 6.480$\pm$0.124        & 6.785$\pm$0.057 &  6.876$\pm$0.002&  6.874$\pm$0.005\\
0.4  &\bf{5.250$\pm$0.003}  &\bf{5.250}   &4.998$\pm$0.093   & 4.894$\pm$0.111        &  5.125$\pm$0.041 &\bf{5.250$\pm$0.002} &\bf{5.250$\pm$0.005}\\
0.5  & 3.784$\pm$0.008   &\bf{3.794}   & 3.606$\pm$0.077  & 3.532$\pm$0.072        &  3.716$\pm$0.037 &  3.789$\pm$0.007 &3.790$\pm$0.008\\
0.6  & 2.645$\pm$0.008   &  2.518   & 2.427$\pm$0.071  & 2.400$\pm$0.076        & 2.573$\pm$0.022 &  2.646$\pm$0.006 &\bf{2.648$\pm$0.004}\\
0.7  & 1.719$\pm$0.028   &  1.575   & 1.505$\pm$0.094  & 1.471$\pm$0.102        & 1.629$\pm$0.021 &\bf{1.728$\pm$0.022}&\bf{1.728$\pm$0.020}\\
0.8  & 0.972$\pm$0.031   &  0.866   & 0.789$\pm$0.061  & 0.723$\pm$0.082        &  0.915$\pm$0.011 &\bf{0.990$\pm$0.024}&  0.980$\pm$0.022\\
0.9  & 0.402$\pm$0.011   &  0.371   & 0.268$\pm$0.054  & 0.267$\pm$0.057        & 0.376$\pm$0.006 &\bf{0.409$\pm$0.006}&  0.397$\pm$0.011\\
1    & 0.004$\pm$0.002   &  0.000    & 0.000$\pm$0.000 &0.000$\pm$0.000         & 0.000$\pm$0.000       &\bf{0.006$\pm$0.001}&  0.000$\pm$0.000\\
\hline
&&&&&&&\\[-8pt]
\#best &   \multicolumn{1}{c|}{5} &  \multicolumn{1}{c}{6} &  \multicolumn{1}{c}{0}  &  \multicolumn{1}{c}{0} &  \multicolumn{1}{c|}{0} &  \multicolumn{1}{c}{\bf{8}} &   \multicolumn{1}{c}{3}  \\
\hline
&&&&&&&\\[-8pt]
average rank &   \multicolumn{1}{c|}{2.455} &  \multicolumn{1}{c}{3.273} &  \multicolumn{1}{c}{5.909}  &  \multicolumn{1}{c}{6.818} &  \multicolumn{1}{c|}{4.636} &  \multicolumn{1}{c}{\bf1.909} &   \multicolumn{1}{c}{3.000}  \\
\hline
\multicolumn{2}{c|}{}&&&&&&\\[-8pt]
\multicolumn{2}{c|}{GSEMO: \#direct win} & \multicolumn{1}{c}{7.5} & \multicolumn{1}{c}{\bf11}  & \multicolumn{1}{c}{\bf11} & \multicolumn{1}{c|}{\bf11} & \multicolumn{1}{c}{3} &  \multicolumn{1}{c}{6.5} \\
\hline
\hline
\multicolumn{8}{c}{}\\[-8pt]
\multicolumn{8}{c}{$\gamma=0.567$}\\
\hline
&&&&&&&\\[-8pt]
$\alpha$ & \multicolumn{1}{c|}{GSEMO} & \multicolumn{1}{c}{DG} & \multicolumn{1}{c}{SDG(0.1)} & \multicolumn{1}{c}{SDG(0.2)} & \multicolumn{1}{c|}{Multi-SDG} & \multicolumn{1}{c}{GSEMO$_{g-c}$} & \multicolumn{1}{c}{NSGA-II} \\
\hline
&&&&&&&\\[-8pt]
0    & \bf{15.578$\pm$0.000} &\bf{15.578}  &14.657$\pm$0.382  &14.298$\pm$0.371  &15.201$\pm$0.227&\bf{15.578$\pm$0.000}   &  15.571$\pm$0.027\\
0.1  &\bf{13.363$\pm$0.000}  &\bf{13.363}  &12.763$\pm$0.205  & 12.358$\pm$0.351 & 13.076$\pm$0.192&\bf{13.363$\pm$0.000}   &  13.362$\pm$0.003\\
0.2  &\bf{11.179$\pm$0.000}  &\bf{11.179}  & 10.742$\pm$0.183 & 10.411$\pm$0.271      & 10.911$\pm$0.173&\bf{11.179$\pm$0.000}   &  11.178$\pm$0.006 \\
0.3  &\bf{9.041$\pm$0.000}  &\bf{9.041}   &8.645$\pm$0.188   & 8.406$\pm$0.258  & 8.783$\pm$0.146 &\bf{9.041$\pm$0.000}    &  9.040$\pm$0.004\\
0.4  &\bf{6.943$\pm$0.000}   &  6.942   & 6.625$\pm$0.143  & 6.509$\pm$0.175  & 6.783$\pm$0.085 &\bf{6.943$\pm$0.000}    &  6.941$\pm$0.004\\
0.5  &\bf{5.053$\pm$0.000}   &  5.033   & 4.801$\pm$0.097  & 4.760$\pm$0.095   &4.942$\pm$0.048 &\bf{5.053$\pm$0.000}    &  5.050$\pm$0.007\\
0.6  &3.446$\pm$0.003   &  3.385   & 3.138$\pm$0.075  & 3.104$\pm$0.097   &3.311$\pm$0.027 &\bf{3.448$\pm$0.002}     &  3.435$\pm$0.012\\
0.7  &\bf{2.138$\pm$0.009}   &  2.000   & 1.857$\pm$0.083  & 1.811$\pm$0.112   &1.975$\pm$0.019 &  2.133$\pm$0.004&  2.129$\pm$0.013\\
0.8  & 1.120$\pm$0.013   &  0.983   & 0.894$\pm$0.066  & 0.851$\pm$0.070   & 1.00$\pm$0.010  &\bf{1.132$\pm$0.003}&  1.102$\pm$0.024\\
0.9  & 0.432$\pm$0.014   &  0.386   & 0.316$\pm$0.038  & 0.298$\pm$0.046   & 0.409$\pm$0.004 &\bf{0.443$\pm$0.005}&  0.416$\pm$0.026\\
1    & 0.002$\pm$0.001   &  0.000       & 0.000$\pm$0.000          & 0.000$\pm$0.000       &0.000$\pm$0.000        &\bf{0.004$\pm$0.000}    &  0.000$\pm$0.000\\
\hline
&&&&&&&\\[-8pt]
\#best &   \multicolumn{1}{c|}{7} &  \multicolumn{1}{c}{4} &  \multicolumn{1}{c}{0}  &  \multicolumn{1}{c}{0} &  \multicolumn{1}{c|}{0} &  \multicolumn{1}{c}{\bf10} &   \multicolumn{1}{c}{0}  \\
\hline
&&&&&&&\\[-8pt]
average rank &   \multicolumn{1}{c|}{1.818} &  \multicolumn{1}{c}{3.455} &  \multicolumn{1}{c}{5.909}  &  \multicolumn{1}{c}{6.818} &  \multicolumn{1}{c|}{4.818} &  \multicolumn{1}{c}{\bf 1.545} &   \multicolumn{1}{c}{3.636}  \\
\hline
\multicolumn{2}{c|}{}&&&&&&\\[-8pt]
\multicolumn{2}{c|}{GSEMO: \#direct win} & \multicolumn{1}{c}{\bf9} & \multicolumn{1}{c}{\bf11}  & \multicolumn{1}{c}{\bf11} & \multicolumn{1}{c|}{\bf11} & \multicolumn{1}{c}{4} &  \multicolumn{1}{c}{\bf11} \\
\hline
\hline
\multicolumn{8}{c}{}\\[-8pt]
\multicolumn{8}{c}{$\gamma=0.800$}\\
\hline
&&&&&&&\\[-8pt]
$\alpha$ & \multicolumn{1}{c|}{GSEMO} & \multicolumn{1}{c}{DG} & \multicolumn{1}{c}{SDG(0.1)} & \multicolumn{1}{c}{SDG(0.2)} & \multicolumn{1}{c|}{Multi-SDG} & \multicolumn{1}{c}{GSEMO$_{g-c}$} & \multicolumn{1}{c}{NSGA-II} \\
\hline
&&&&&&&\\[-8pt]
0    & \bf{8.688$\pm$0.000}  &\bf{8.688}   &8.392$\pm$0.165  &8.070$\pm$0.270        & 8.545$\pm$0.105&\bf{8.688$\pm$0.000}    &  8.680$\pm$0.018\\
0.1  &\bf{7.510$\pm$0.000}   &\bf{7.510}   & 7.189$\pm$0.192 &7.047$\pm$0.246         & 7.434$\pm$0.073&\bf{7.510$\pm$0.000}    &  7.506$\pm$0.011\\
0.2  &\bf{6.339$\pm$0.000}   &\bf{6.339}  & 6.122$\pm$0.094 & 6.024$\pm$0.137       &  6.243$\pm$0.062&\bf{6.339$\pm$0.000}    &  6.334$\pm$0.010 \\
0.3  &\bf{5.181$\pm$0.000}   &\bf{5.181}   &5.030$\pm$0.084   & 4.887$\pm$0.161        & 5.126$\pm$0.042 &\bf{5.181$\pm$0.000}    &  5.180$\pm$0.001\\
0.4  &\bf{4.075$\pm$0.000}   &\bf{4.075}   & 3.962$\pm$0.050  & 3.821$\pm$0.079        & 4.017$\pm$0.036 &\bf{4.075$\pm$0.000}    &\bf{4.075$\pm$0.001}\\
0.5  &\bf{3.046$\pm$0.000}   &  2.895   & 2.738$\pm$0.053  & 2.710$\pm$0.079        & 2.839$\pm$0.033 & \bf{3.046$\pm$0.000}    &\bf{3.046$\pm$0.000} \\
0.6  &\bf{2.095$\pm$0.000}  &  1.856   & 1.714$\pm$0.051  & 1.697$\pm$0.049        & 1.806$\pm$0.026 &\bf{2.095$\pm$0.000}    &  2.091$\pm$0.006\\
0.7  & 1.274$\pm$0.003   &  1.050   & 0.963$\pm$0.051  & 0.944$\pm$0.043        &  1.031$\pm$0.013 &\bf{1.277$\pm$0.000}     &  1.269$\pm$0.004\\
0.8  & 0.650$\pm$0.008   &  0.520   & 0.463$\pm$0.026  & 0.461$\pm$0.023        &  0.509$\pm$0.006 &\bf{0.656$\pm$0.000}     &  0.640$\pm$0.008\\
0.9  & 0.216$\pm$0.004   &  0.165   & 0.128$\pm$0.017  & 0.122$\pm$0.022        &  0.165$\pm$0.003  &\bf{0.226$\pm$0.000}    &  0.213$\pm$0.007\\
1    &\bf{0.000$\pm$0.000}   &\bf{0.000}   &\bf{0.000$\pm$0.000}         &\bf{0.000$\pm$0.000}      &\bf{0.000$\pm$0.000}      &\bf{0.000$\pm$0.000}    &\bf{0.000$\pm$0.000}\\
\hline
&&&&&&&\\[-8pt]
\#best &   \multicolumn{1}{c|}{8} &  \multicolumn{1}{c}{6} &  \multicolumn{1}{c}{1}  &  \multicolumn{1}{c}{1} &  \multicolumn{1}{c|}{1} &  \multicolumn{1}{c}{\bf11} &   \multicolumn{1}{c}{3}  \\
\hline
&&&&&&&\\[-8pt]
average rank &   \multicolumn{1}{c|}{2.182} &  \multicolumn{1}{c}{3.182} &  \multicolumn{1}{c}{5.818}  &  \multicolumn{1}{c}{6.727} &  \multicolumn{1}{c|}{4.864} &  \multicolumn{1}{c}{\bf1.909} &   \multicolumn{1}{c}{3.318}  \\
\hline
\multicolumn{2}{c|}{}&&&&&&\\[-8pt]
\multicolumn{2}{c|}{GSEMO: \#direct win} & \multicolumn{1}{c}{8} & \multicolumn{1}{c}{\bf10.5}  & \multicolumn{1}{c}{\bf10.5} & \multicolumn{1}{c|}{\bf10.5} & \multicolumn{1}{c}{4} &  \multicolumn{1}{c}{\bf9.5} \\
\hline
\end{tabular}
\end{lrbox}
\scalebox{0.74}{\usebox{\tablebox}}
\end{table*}

\begin{table*}[t!]\centering
\caption{The objective values (mean+std.) of the compared algorithms on the application of Bayesian experimental design with costs, using the data set \textit{segment} with $k=15$, $\gamma \in \{0.281, 0.621, 0.840\}$ and $\alpha \in \{0,0.1,\ldots,1\}$. For each $\alpha$, the largest values are bolded. In the rows of ``\#best", the largest values are bolded. The average rank of each algorithm on all $\alpha$ is computed, where the smallest values are bolded. The ``\#direct win" denotes the number of $\alpha$ on which the GSEMO has a larger objective value than the corresponding algorithm ($1$ tie is counted as $0.5$ win), where significant cells by the sign-test~\citep{demsar:06} with confidence level $0.05$ are bolded.\vspace{6pt}}\label{table_bayesian_segment}
\small
\begin{lrbox}{\tablebox}
\begin{tabular}{c|r|rrrr|rr}
\hline
\multicolumn{8}{c}{}\\[-8pt]
\multicolumn{8}{c}{$\gamma=0.281$}\\
\hline
&&&&&&&\\[-8pt]
$\alpha$ & \multicolumn{1}{c|}{GSEMO} & \multicolumn{1}{c}{DG} & \multicolumn{1}{c}{SDG(0.1)} & \multicolumn{1}{c}{SDG(0.2)} & \multicolumn{1}{c|}{Multi-SDG} & \multicolumn{1}{c}{GSEMO$_{g-c}$} & \multicolumn{1}{c}{NSGA-II} \\
\hline
&&&&&&&\\[-8pt]
0    &\bf{29.738$\pm$0.000} &\bf{29.738}   &28.688$\pm$0.528 &27.873$\pm$0.839      &29.122$\pm$0.248  &\bf{29.738$\pm$0.000}  &  29.564$\pm$0.110\\
0.1  &\bf{24.886$\pm$0.000} &\bf{24.886}   &23.852$\pm$0.440 &23.330$\pm$0.873&24.467$\pm$0.256  &\bf{24.886$\pm$0.000} &  24.779$\pm$0.099\\
0.2  &\bf{20.489$\pm$0.000} &\bf{20.489}   &19.706$\pm$0.368 &19.146$\pm$0.419       &20.113$\pm$0.195  &\bf{20.489$\pm$0.000} &  20.417$\pm$0.081 \\
0.3  &\bf{16.302$\pm$0.000} & 16.272   &15.473$\pm$0.482  &15.329$\pm$0.474       &16.066$\pm$0.147  &\bf{16.302$\pm$0.000}   &  16.255$\pm$0.064\\
0.4  &\bf{12.685$\pm$0.090}  & 12.563   &11.973$\pm$0.296  &11.870$\pm$0.359       &12.544$\pm$0.113  &  12.648$\pm$0.090   &  12.657$\pm$0.093\\
0.5  &9.426$\pm$0.139   & 9.259    &8.881$\pm$0.388   &8.808$\pm$0.323        &9.390$\pm$0.080    &  9.443$\pm$0.136  &\bf{9.520$\pm$0.086} \\
0.6  &6.636$\pm$0.062  &  6.536   &5.871$\pm$0.371   &5.947$\pm$0.431        &6.608$\pm$0.036   &  6.599$\pm$0.071  &\bf{6.638$\pm$0.060}\\
0.7  &4.279$\pm$0.025   &  4.242   &3.813$\pm$0.430   &3.548$\pm$0.427        &4.243$\pm$0.018   & 4.281$\pm$0.020  &\bf{4.289$\pm$0.028}\\
0.8  &2.448$\pm$0.012   &  2.414   &1.872$\pm$0.347   &1.667$\pm$0.246        &2.419$\pm$0.011   &\bf{2.451$\pm$0.009}  &  2.444$\pm$0.013\\
0.9  &\bf{1.074$\pm$0.006}   &  1.057   &0.748$\pm$0.196   &0.716$\pm$0.196        &1.044$\pm$0.006   &\bf{1.074$\pm$0.006}      &  1.063$\pm$0.013\\
1    &0.001$\pm$0.001   &  0.000       & 0.000$\pm$0.000          & 0.000$\pm$0.000                &  0.000$\pm$0.000   &\bf{0.003$\pm$0.000}           &  0.001$\pm$0.001\\
\hline
&&&&&&&\\[-8pt]
\#best &   \multicolumn{1}{c|}{6} &  \multicolumn{1}{c}{3} &  \multicolumn{1}{c}{0}  &  \multicolumn{1}{c}{0} &  \multicolumn{1}{c|}{0} &  \multicolumn{1}{c}{\bf7} &   \multicolumn{1}{c}{3}  \\
\hline
&&&&&&&\\[-8pt]
average rank &   \multicolumn{1}{c|}{2.045} &  \multicolumn{1}{c}{3.864} &  \multicolumn{1}{c}{6.045}  &  \multicolumn{1}{c}{6.773} &  \multicolumn{1}{c|}{4.591} &  \multicolumn{1}{c}{\bf2.000} &   \multicolumn{1}{c}{2.682}  \\
\hline
\multicolumn{2}{c|}{}&&&&&&\\[-8pt]
\multicolumn{2}{c|}{GSEMO: \#direct win} & \multicolumn{1}{c}{\bf9.5} & \multicolumn{1}{c}{\bf11}  & \multicolumn{1}{c}{\bf11} & \multicolumn{1}{c|}{\bf11} & \multicolumn{1}{c}{4.5} &  \multicolumn{1}{c}{7.5} \\
\hline
\hline
\multicolumn{8}{c}{}\\[-8pt]
\multicolumn{8}{c}{$\gamma=0.621$}\\
\hline
&&&&&&&\\[-8pt]
$\alpha$ & \multicolumn{1}{c|}{GSEMO} & \multicolumn{1}{c}{DG} & \multicolumn{1}{c}{SDG(0.1)} & \multicolumn{1}{c}{SDG(0.2)} & \multicolumn{1}{c|}{Multi-SDG} & \multicolumn{1}{c}{GSEMO$_{g-c}$} & \multicolumn{1}{c}{NSGA-II} \\
\hline
&&&&&&&\\[-8pt]
0    &\bf{10.152$\pm$0.000} &\bf{10.152}  &9.409$\pm$0.543 &8.896$\pm$0.633       &9.880$\pm$0.147   &\bf{10.152$\pm$0.000}  &  10.127$\pm$0.037\\
0.1  &\bf{8.804$\pm$0.000}    & \bf{8.804}   &8.275$\pm$0.381 &7.900$\pm$0.426 &8.616$\pm$0.112   &  \bf{8.804$\pm$0.000} &  8.785$\pm$0.029\\
0.2  &\bf{7.479$\pm$0.000}    &\bf{7.479}   &7.095$\pm$0.238 &6.570$\pm$0.534        &7.321$\pm$0.105   &\bf{7.479 $\pm$0.000} &  7.445$\pm$0.034 \\
0.3  &\bf{6.173$\pm$0.000}    &\bf{6.173}   &5.876$\pm$0.160  &5.620$\pm$0.298        &6.051$\pm$0.060   &\bf{6.173$\pm$0.000}   &  6.164$\pm$0.011\\
0.4  &\bf{4.941$\pm$0.000}    &\bf{4.941}   &4.641$\pm$0.143  &4.461$\pm$0.246        &4.816$\pm$0.060   &\bf{4.941$\pm$0.090}   &  4.924$\pm$0.018\\
0.5  &\bf{3.740$\pm$0.000}    &\bf{3.740}   &3.459$\pm$0.145  &3.360$\pm$0.167        &3.668$\pm$0.050   &\bf{3.740$\pm$0.136}  &  3.738$\pm$0.005\\
0.6  &\bf{2.645$\pm$0.000}    & 2.521   &2.331$\pm$0.073  &2.257$\pm$0.112        &2.459$\pm$0.035   &\bf{2.645$\pm$0.071}  &  2.643$\pm$0.004\\
0.7  &\bf{1.713$\pm$0.001}   & 1.490   &1.361$\pm$0.059  &1.281$\pm$0.097        &1.454$\pm$0.020   &\bf{1.713$\pm$0.001}  &  1.712$\pm$0.002\\
0.8  &0.938$\pm$0.002   & 0.747   &0.643$\pm$0.062  &0.600$\pm$0.105        &0.731$\pm$0.008   &\bf{0.939$\pm$0.000}  &  0.927$\pm$0.015\\
0.9  &0.349$\pm$0.004   & 0.285   &0.223$\pm$0.055  &0.211$\pm$0.050        &0.283$\pm$0.002   &\bf{0.360$\pm$0.001}      &  0.347$\pm$0.003\\
1    &\bf{0.000$\pm$0.000}   &\bf{0.000}   &\bf{0.000$\pm$0.000}   &\bf{0.000$\pm$0.000}  &\bf{0.000$\pm$0.000}    &\bf{0.000$\pm$0.000} &\bf{0.000$\pm$0.000}\\
\hline
&&&&&&&\\[-8pt]
\#best &   \multicolumn{1}{c|}{9} &  \multicolumn{1}{c}{7} &  \multicolumn{1}{c}{1}  &  \multicolumn{1}{c}{1} &  \multicolumn{1}{c|}{1} &  \multicolumn{1}{c}{\bf11} &   \multicolumn{1}{c}{1}  \\
\hline
&&&&&&&\\[-8pt]
average rank &   \multicolumn{1}{c|}{2.091} &  \multicolumn{1}{c}{2.909} &  \multicolumn{1}{c}{5.818}  &  \multicolumn{1}{c}{6.727} &  \multicolumn{1}{c|}{4.909} &  \multicolumn{1}{c}{\bf1.909} &   \multicolumn{1}{c}{3.636}  \\
\hline
\multicolumn{2}{c|}{}&&&&&&\\[-8pt]
\multicolumn{2}{c|}{GSEMO: \#direct win} & \multicolumn{1}{c}{7.5} & \multicolumn{1}{c}{\bf10.5}  & \multicolumn{1}{c}{\bf10.5} & \multicolumn{1}{c|}{\bf10.5} & \multicolumn{1}{c}{4.5} &  \multicolumn{1}{c}{\bf10.5} \\
\hline
\hline
\multicolumn{8}{c}{}\\[-8pt]
\multicolumn{8}{c}{$\gamma=0.840$}\\
\hline
&&&&&&&\\[-8pt]
$\alpha$ & \multicolumn{1}{c|}{GSEMO} & \multicolumn{1}{c}{DG} & \multicolumn{1}{c}{SDG(0.1)} & \multicolumn{1}{c}{SDG(0.2)} & \multicolumn{1}{c|}{Multi-SDG} & \multicolumn{1}{c}{GSEMO$_{g-c}$} & \multicolumn{1}{c}{NSGA-II} \\
\hline
&&&&&&&\\[-8pt]
0    & \bf{2.208$\pm$0.000}  &\bf{2.208}   &2.017$\pm$0.102  &1.804$\pm$0.155       &2.113$\pm$0.059   &  {\bf2.208$\pm$0.000}  &2.197$\pm$0.016\\
0.1  &\bf{1.966$\pm$0.000}   &\bf{1.966}   &1.764$\pm$0.101 &1.655$\pm$0.196       &1.877$\pm$0.057  &  {\bf1.966$\pm$0.000} &1.960$\pm$0.009\\
0.2  &\bf{1.725$\pm$0.000}   &\bf{1.725}   &1.546$\pm$0.113 &1.478$\pm$0.107       &1.664$\pm$0.050  &  {\bf1.725$\pm$0.000} &1.722$\pm$0.004\\
0.3  &\bf{1.484$\pm$0.000}   &\bf{1.484}   &1.354$\pm$0.087  &1.186$\pm$0.135       &1.446$\pm$0.024  &  {\bf1.484$\pm$0.000}  &1.480$\pm$0.006\\
0.4  &\bf{1.245$\pm$0.000}   &\bf{1.245}   &1.143$\pm$0.058  &1.060$\pm$0.096       &1.192$\pm$0.034  &  {\bf1.245$\pm$0.000}  &1.242$\pm$0.004\\
0.5  &\bf{1.006$\pm$0.000}   &0.952        &0.872$\pm$0.046  &0.753$\pm$0.073       &0.917$\pm$0.018  &  {\bf1.006$\pm$0.000}  &1.003$\pm$0.006\\
0.6  &\bf{0.768$\pm$0.000}   &0.623        &0.516$\pm$0.049  &0.468$\pm$0.060       &0.594$\pm$0.017  &  {\bf0.768$\pm$0.000}  &0.767$\pm$0.002\\
0.7  &\bf{0.531$\pm$0.000}   &0.392        &0.317$\pm$0.032  &0.299$\pm$0.051       &0.377$\pm$0.013  &  {\bf0.531$\pm$0.000}  &0.530$\pm$0.003\\
0.8  &\bf{0.306$\pm$0.000}   &0.187        &0.141$\pm$0.024   &0.117$\pm$0.031      &0.179$\pm$0.006  &  {\bf0.306$\pm$0.000}  &0.300$\pm$0.006\\
0.9  &0.116$\pm$0.000        &0.054        &0.038$\pm$0.010   &0.031$\pm$0.014      &0.054$\pm$0.000  &  {\bf0.118$\pm$0.000}  &0.104$\pm$0.006\\
1    &\bf{0.000$\pm$0.000}   &\bf{0.000}   &\bf{0.000$\pm$0.000}       &\bf{0.000$\pm$0.000}          &\bf{0.000$\pm$0.000}    &\bf{0.000$\pm$0.000}    &\bf{0.000$\pm$0.000}\\
\hline
&&&&&&&\\[-8pt]
\#best &   \multicolumn{1}{c|}{10} &  \multicolumn{1}{c}{6} &  \multicolumn{1}{c}{1}  &  \multicolumn{1}{c}{1} &  \multicolumn{1}{c|}{1} &  \multicolumn{1}{c}{\bf11} &   \multicolumn{1}{c}{1}  \\
\hline
&&&&&&&\\[-8pt]
average rank &   \multicolumn{1}{c|}{2.000} &  \multicolumn{1}{c}{3.136} &  \multicolumn{1}{c}{5.818}  &  \multicolumn{1}{c}{6.727} &  \multicolumn{1}{c|}{4.864} &  \multicolumn{1}{c}{\bf1.909} &   \multicolumn{1}{c}{3.545} \\
\hline
\multicolumn{2}{c|}{}&&&&&&\\[-8pt]
\multicolumn{2}{c|}{GSEMO: \#direct win} & \multicolumn{1}{c}{8} & \multicolumn{1}{c}{\bf{10.5}}  & \multicolumn{1}{c}{\bf{10.5}} & \multicolumn{1}{c|}{\bf{10.5}} & \multicolumn{1}{c}{5} &  \multicolumn{1}{c}{\bf{10.5}} \\
\hline
\end{tabular}
\end{lrbox}
\scalebox{0.73}{\usebox{\tablebox}}
\end{table*}

The above experiments compare the algorithms under different values of the budget $k$. By fixing $k=15$, we then consider different values of $\alpha \in \{0,0.1,\ldots,1\}$, which is used to construct the costs $c_i=\alpha \cdot g(\{i\})$. The results on the data sets \textit{housing} and \textit{segment} are shown in Tables~\ref{table_bayesian_housing} and~\ref{table_bayesian_segment}, respectively. To compare the algorithms, we employ several criteria: the number of best (i.e., \#best), the number of direct win (i.e., \#direct win) that is a pairwise comparison followed by the sign-test~\citep{demsar:06}, and the rank~\citep{demsar:06}. From the rows of ``\#best", we can observe that the GSEMO and GSEMO$_{g-c}$ achieve the best performance in most cases. From the rows of ``average rank", it can be observed that the compared algorithms have a similar performance rank as in Figures~\ref{fig_BO_housing} and~\ref{fig_BO_segment}. The SDG algorithms are the worst, between which the SDG(0.1) is better than the SDG(0.2). The DG and Multi-SDG perform better than the SDG, but are clearly worse than the GSEMO and GSEMO$_{g-c}$. The NSGA-II is worse than the GSEMO and GSEMO$_{g-c}$, and can even be worse than the DG. By the sign-test with confidence level 0.05, the GSEMO is significantly better than the SDG and Multi-SDG in all cases, significantly better than the DG in 2/6 cases, and significantly better than the NSGA-II in 4/6 cases. The GSEMO and GSEMO$_{g-c}$ have no significant difference in all cases.

The GSEMO requires $O(n^2(\log n+k))$ expected number of iterations in theory as shown in Theorem~\ref{theo-main}, and runs for $\lceil ek^2n \rceil$ iterations in the experiments. Note that one iteration of the GSEMO costs one objective function evaluation. We want to examine how efficient the GSEMO can be in practice. Thus, we select the DG and SDG for the baseline, and plot the curve of objective value over the running time for the GSEMO, as shown in Figures~\ref{fig_BO-time-housing} and~\ref{fig_BO-time-segment}. The budget $k$ is set to 20. Note that the running time is considered in the number of objective function evaluations, and one unit on the $x$-axis corresponds to $kn$ evaluations, the running time of the DG. We can observe that the GSEMO quickly obtains a better $(g-c)$ value, implying that the GSEMO can be efficient in practice. We also examine the running time in CPU seconds. For the GSEMO, we record the running time until finding a solution at least as good as that obtained by the DG. The results are shown in Figures~\ref{fig_BO-cputime-housing} and~\ref{fig_BO-cputime-segment}. As expected, the SDG is the fastest, and becomes more and more faster as $\epsilon$ increases. The GSEMO is not so efficient as observed in Figures~\ref{fig_BO-time-housing} and~\ref{fig_BO-time-segment}, because the objective function evaluation is very fast in our experiments, and thus the cost of performing mutation and updating the population cannot be neglected but increases the running time substantially.

\begin{figure*}[t!]\centering
\begin{minipage}[c]{0.4\linewidth}\centering
        \includegraphics[width=\linewidth]{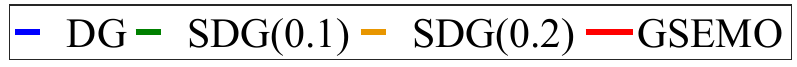}
\end{minipage}\\\vspace{0.2em}
\begin{minipage}[c]{0.32\linewidth}\centering
        \includegraphics[width=1\linewidth]{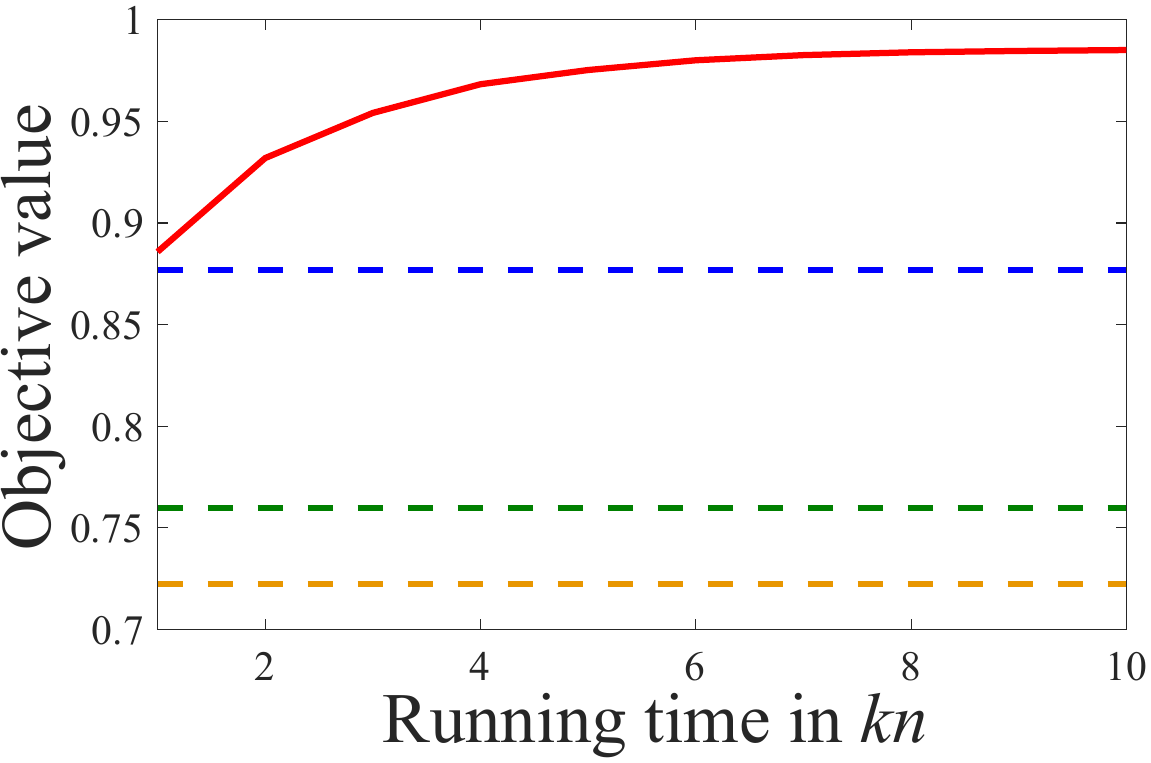}
\end{minipage}
\begin{minipage}[c]{0.32\linewidth}\centering
        \includegraphics[width=1\linewidth]{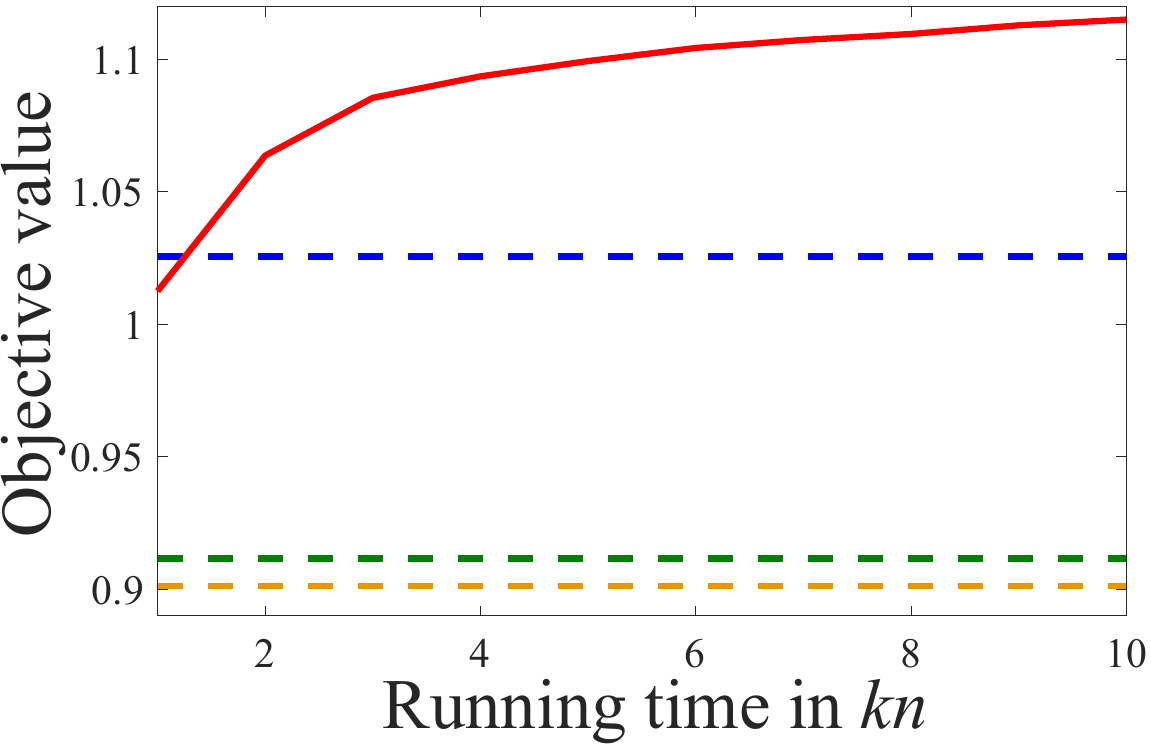}
\end{minipage}
\begin{minipage}[c]{0.32\linewidth}\centering
        \includegraphics[width=1\linewidth]{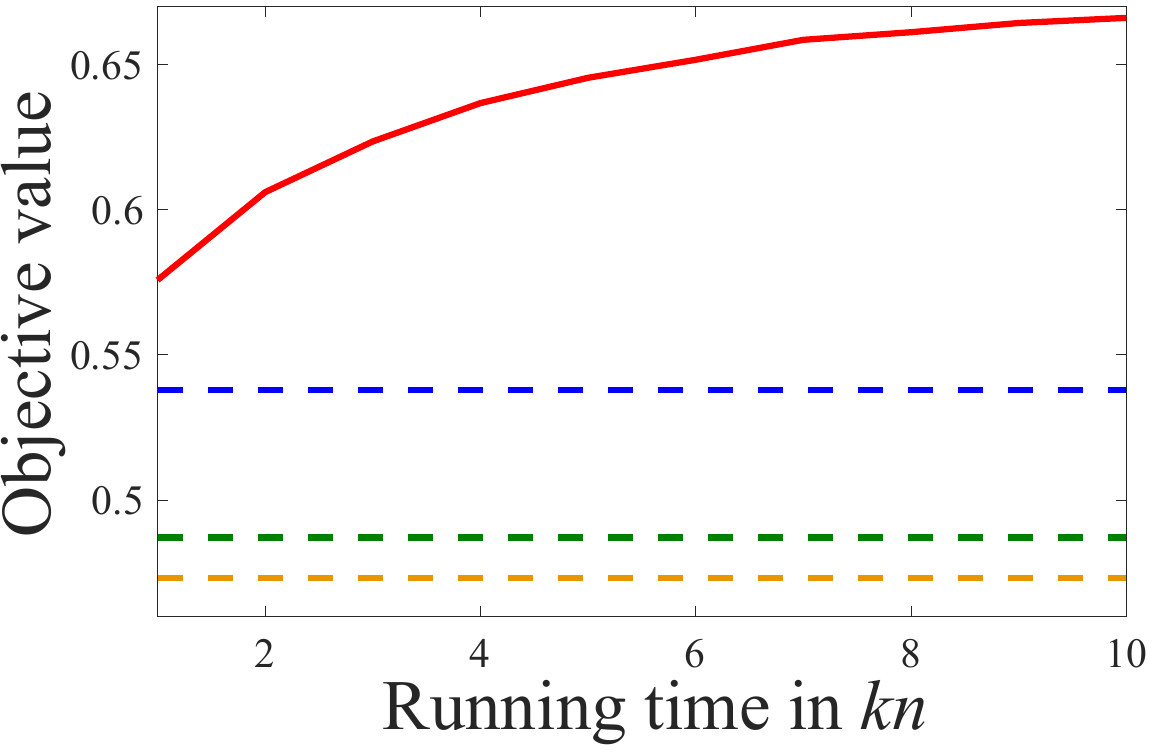}
\end{minipage}\\\vspace{0.3em}
\begin{minipage}[c]{0.32\linewidth}\centering
    \small(a) $\gamma=0.456$
\end{minipage}
\begin{minipage}[c]{0.32\linewidth}\centering
    \small(b) $\gamma=0.567$
\end{minipage}
\begin{minipage}[c]{0.32\linewidth}\centering
    \small(c) $\gamma=0.800$
\end{minipage}\vspace{-0.3em}
\caption{Performance vs. running time (i.e., \#objective evaluations) of the GSEMO on Bayesian experimental design with costs using the data set \textit{housing}, where $k=20$.}\label{fig_BO-time-housing}
\end{figure*}

\begin{figure*}[t!]\centering
\begin{minipage}[c]{0.4\linewidth}\centering
        \includegraphics[width=\linewidth]{figures/time_legend}
\end{minipage}\\\vspace{0.2em}
\begin{minipage}[c]{0.32\linewidth}\centering
        \includegraphics[width=1\linewidth]{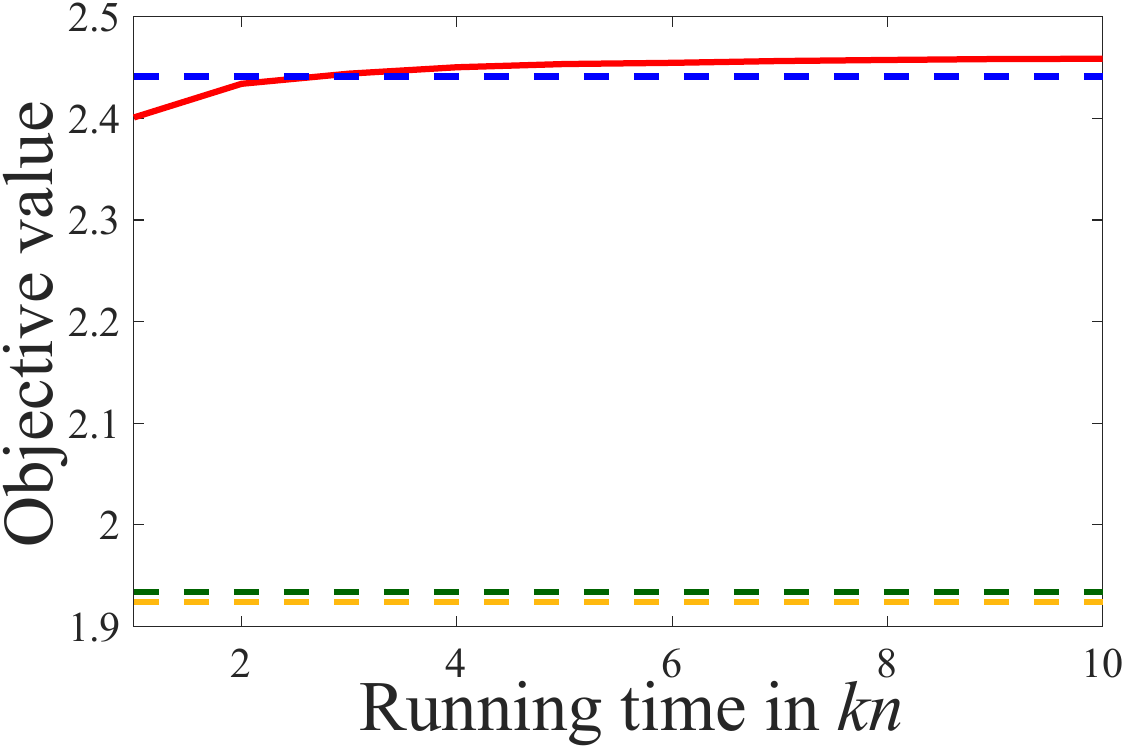}
\end{minipage}
\begin{minipage}[c]{0.32\linewidth}\centering
        \includegraphics[width=1\linewidth]{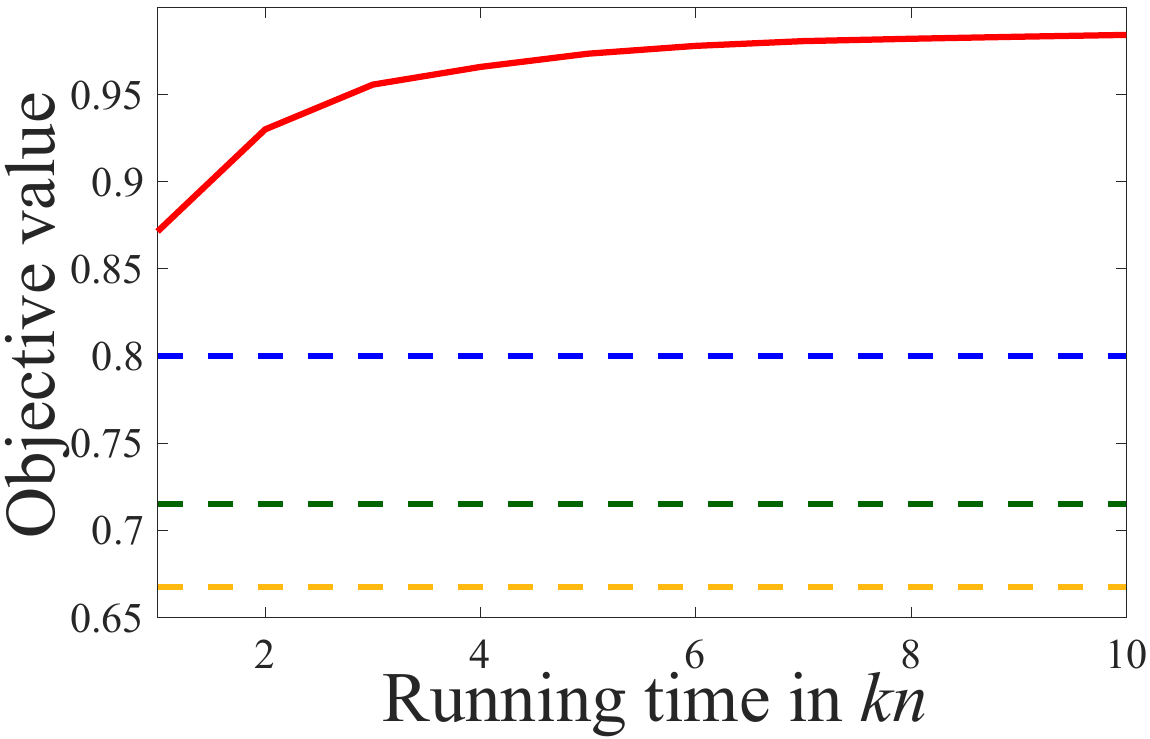}
\end{minipage}
\begin{minipage}[c]{0.32\linewidth}\centering
        \includegraphics[width=1\linewidth]{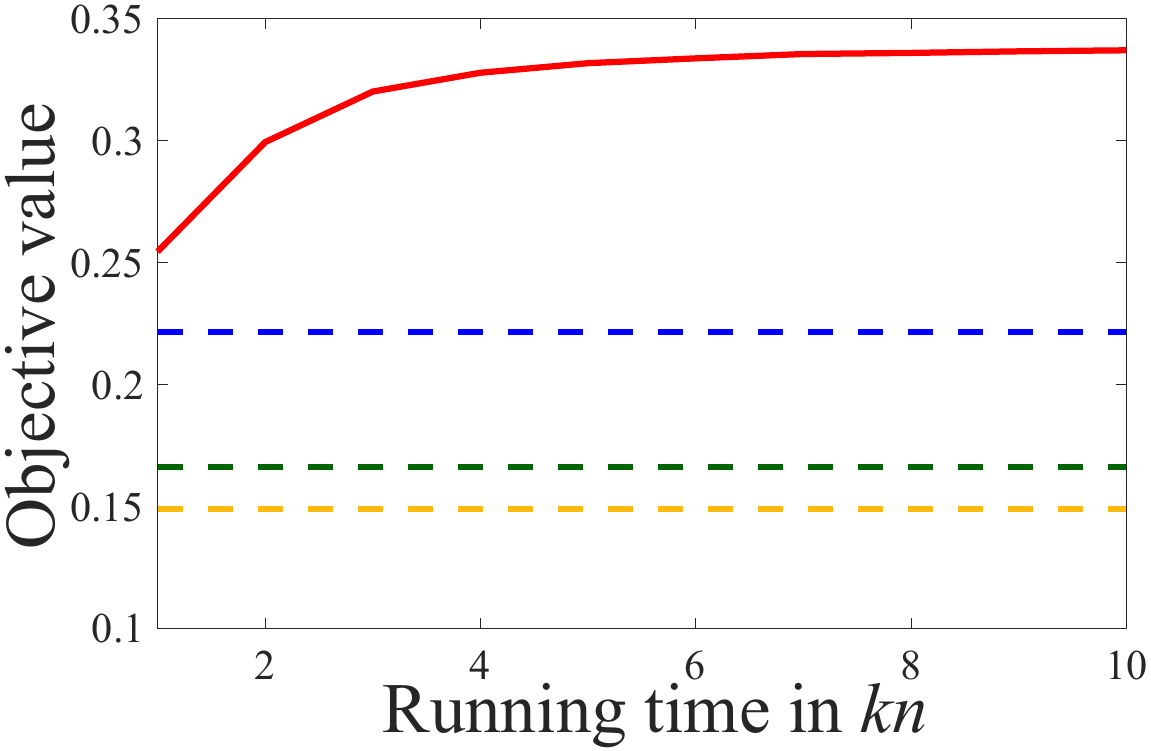}
\end{minipage}\\\vspace{0.3em}
\begin{minipage}[c]{0.32\linewidth}\centering
    \small(a) $\gamma=0.281$
\end{minipage}
\begin{minipage}[c]{0.32\linewidth}\centering
    \small(b) $\gamma=0.621$
\end{minipage}
\begin{minipage}[c]{0.32\linewidth}\centering
    \small(c) $\gamma=0.840$
\end{minipage}\vspace{-0.3em}
\caption{Performance vs. running time (i.e., \#objective evaluations) of the GSEMO on Bayesian experimental design with costs using the data set \textit{segment}, where $k=20$.}\label{fig_BO-time-segment}
\end{figure*}

\begin{figure*}[t!]\centering
\begin{minipage}[c]{0.32\linewidth}\centering
        \includegraphics[width=1\linewidth]{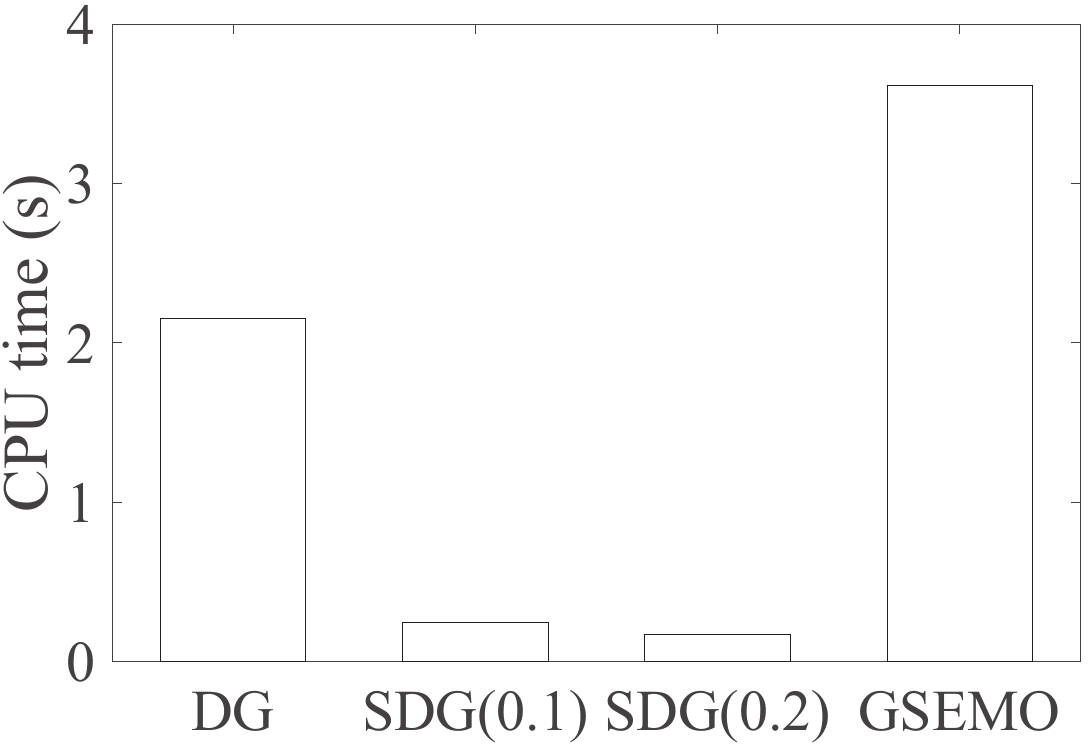}
\end{minipage}
\begin{minipage}[c]{0.32\linewidth}\centering
        \includegraphics[width=1\linewidth]{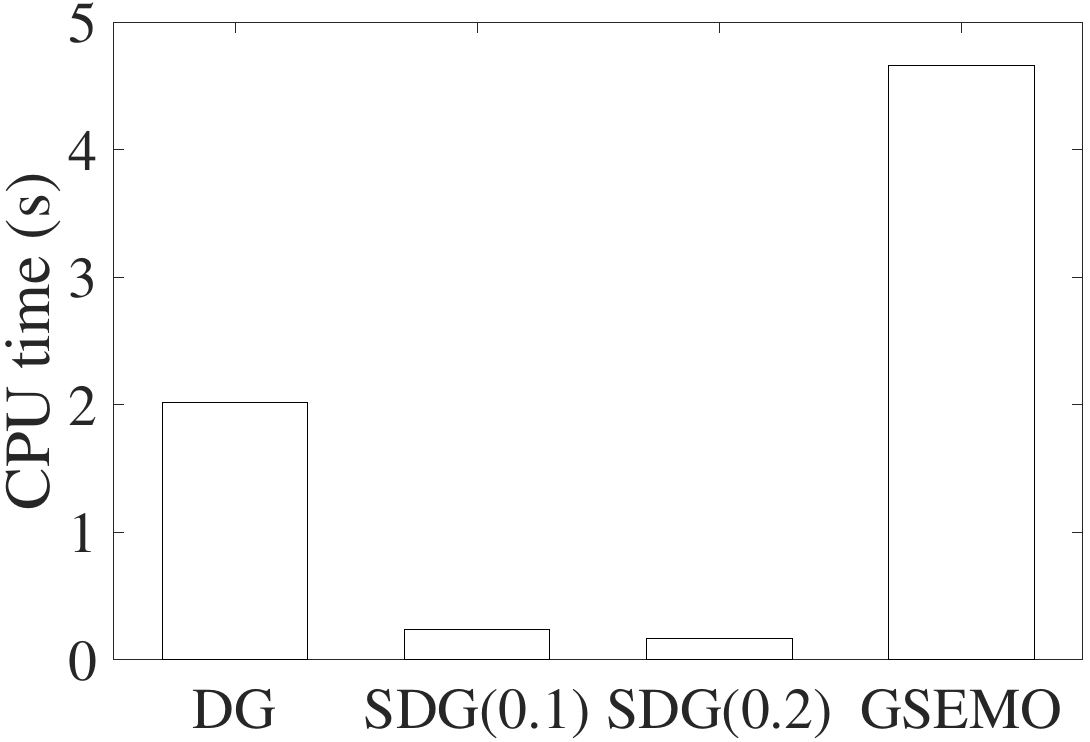}
\end{minipage}
\begin{minipage}[c]{0.32\linewidth}\centering
        \includegraphics[width=1\linewidth]{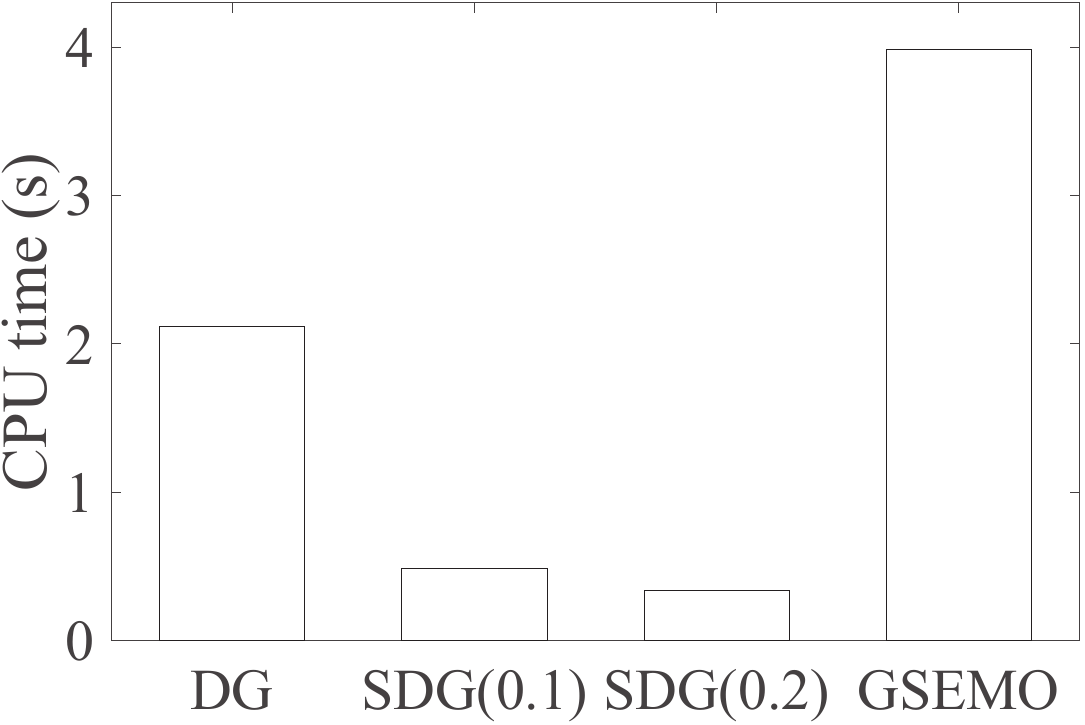}
\end{minipage}\\\vspace{0.3em}
\begin{minipage}[c]{0.32\linewidth}\centering
    \small(a) $\gamma=0.456$
\end{minipage}
\begin{minipage}[c]{0.32\linewidth}\centering
    \small(b) $\gamma=0.567$
\end{minipage}
\begin{minipage}[c]{0.32\linewidth}\centering
    \small(c) $\gamma=0.800$
\end{minipage}\vspace{-0.3em}
\caption{CPU running time of the DG, SDG(0.1), SDG(0.2) and GSEMO on Bayesian experimental design with costs using the data set \textit{housing}, where $k=20$.}\label{fig_BO-cputime-housing}
\end{figure*}

\begin{figure*}[t!]\centering
\begin{minipage}[c]{0.32\linewidth}\centering
        \includegraphics[width=1\linewidth]{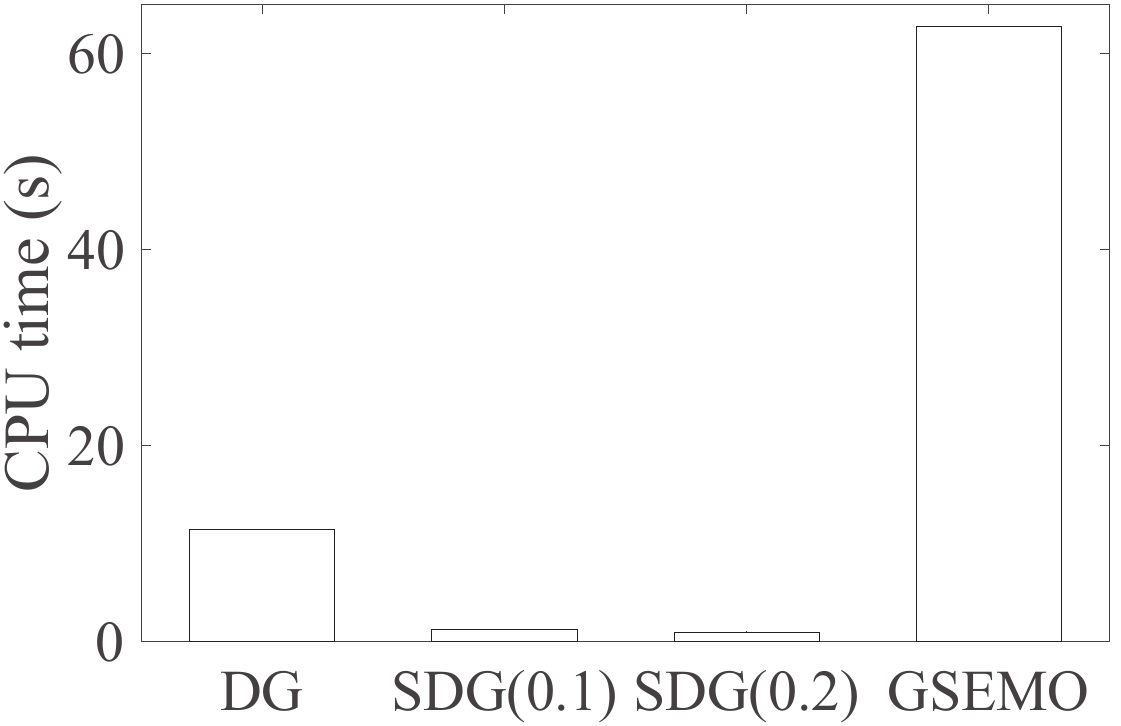}
\end{minipage}
\begin{minipage}[c]{0.32\linewidth}\centering
        \includegraphics[width=1\linewidth]{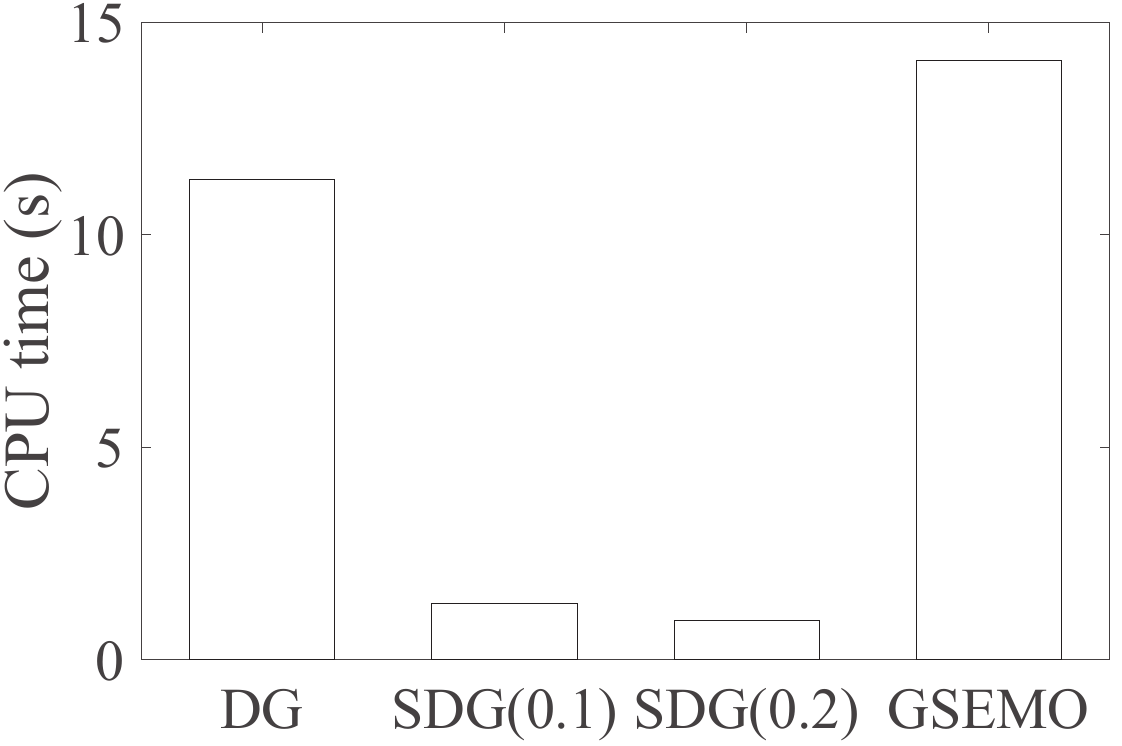}
\end{minipage}
\begin{minipage}[c]{0.32\linewidth}\centering
        \includegraphics[width=1\linewidth]{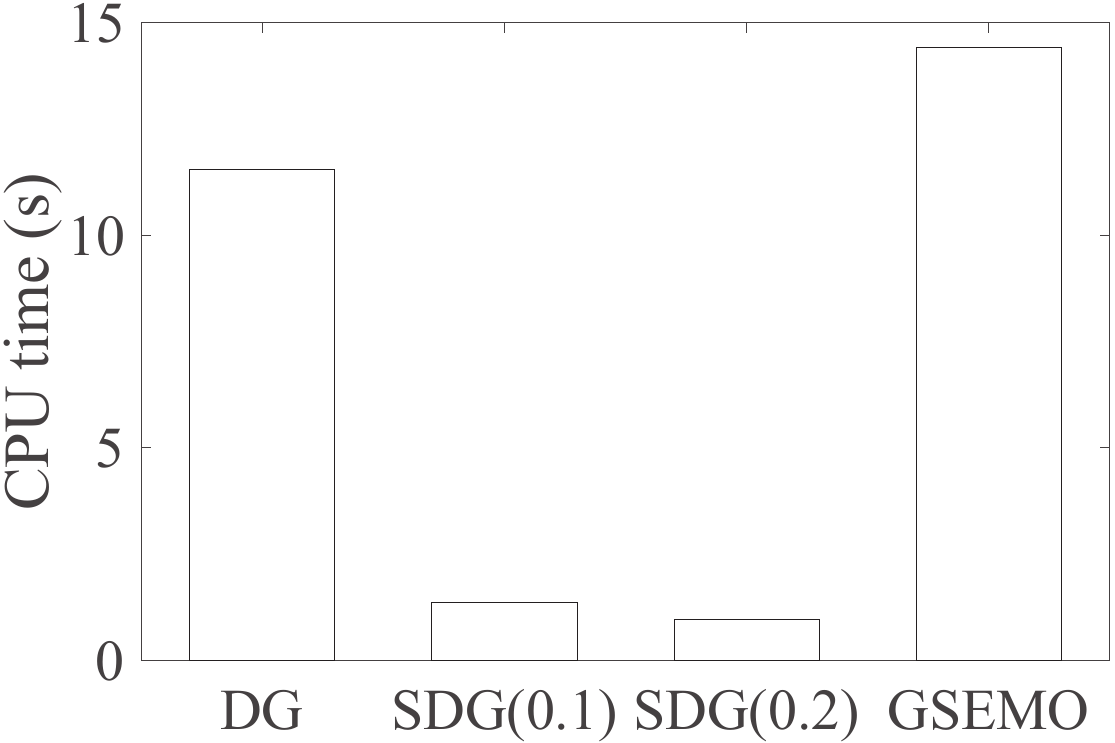}
\end{minipage}\\\vspace{0.3em}
\begin{minipage}[c]{0.32\linewidth}\centering
    \small(a) $\gamma=0.281$
\end{minipage}
\begin{minipage}[c]{0.32\linewidth}\centering
    \small(b) $\gamma=0.621$
\end{minipage}
\begin{minipage}[c]{0.32\linewidth}\centering
    \small(c) $\gamma=0.840$
\end{minipage}\vspace{-0.3em}
\caption{CPU running time of the DG, SDG(0.1), SDG(0.2) and GSEMO on Bayesian experimental design with costs using the data set \textit{segment}, where $k=20$.}\label{fig_BO-cputime-segment}
\end{figure*}

\subsection{Directed Vertex Cover with Costs}

Next, we want to compare the algorithms in cases where the function $g$ is submodular. For this purpose, we use the application of directed vertex cover with costs (as presented in Definition~\ref{def-vertex}), where the function $g$ is exactly submodular. As $g$ is submodular, $\gamma=1$ is used when implementing the algorithms.

We use three graphs\footnote{\url{https://snap.stanford.edu/data/} and \url{https://turing.cs.hbg.psu.edu/txn131/vertex_cover.html}}. \textit{email-Eu-core} is a directed graph with 1,005 vertices and 25,571 edges, generated using email data from a large European research institution. Each vertex represents one person, and edge $(u,v)$ means that person $u$ sent at least one email to person $v$. \textit{frb45-21-mis} and \textit{frb53-24-mis} are two benchmark graphs for vertex cover. Each of them contains five instances, and we use the first one. \textit{frb45-21-mis} contains 945 vertices and 59,186 edges, and \textit{frb53-24-mis} contains 1,272 vertices and 94,227 edges. For these two benchmark graphs, each edge is treated as two directed edges. The weight of each vertex $v$ is set to 1, i.e., $\forall v \in V: w(v)=1$. The cost of each vertex $v$ is set to $1+\max\{d(v)-q,0\}$, i.e., $\forall v\in V: c(v)=1+\max\{d(v)-q,0\}$, where $d(v)$ denotes the out-degree of $v$. Note that~\cite{harshaw2019submodular} also used the \textit{email-Eu-core} data set in their experiments, and our setting is as same as theirs.

By fixing $q=6$, the comparison results under different values of the budget $k \in \{10,20,\ldots,100\}$ are plotted in Figure~\ref{fig_vc}. By fixing $k=60$, the results under different values of $q \in \{1,2,\ldots,12\}$ are shown in Table~\ref{table_cover}. We can observe the similar performance rank of the algorithms as observed for the application of Bayesian experimental design with costs, except that the performance of the GSEMO$_{g-c}$ degrades largely sometimes. On the \textit{frb53-24-mis} data set, the GSEMO$_{g-c}$ is significantly worse than the GSEMO by the sign-test with confidence level $0.05$ in Table~\ref{table_cover}. In fact, it can be observed from Figure~\ref{fig_vc}(c) that the GSEMO$_{g-c}$ is even worse than the DG. This empirical observation verifies our theoretical analysis that the GSEMO$_{g-c}$ fails to achieve an approximation guarantee as good as the GSEMO, i.e., the GSEMO$_{g-c}$ can perform badly in worst cases. Furthermore, the advantage of the GSEMO over other algorithms is more clear for this application. For example, the GSEMO is always significantly better than the DG, SDG, Multi-SDG and NSGA-II in Table~\ref{table_cover}. We also examine the actual running time of the GSEMO to be better than the DG in Figures~\ref{fig_vc-time} and~\ref{fig_vc-cputime}. Compared with that observed in the experiments of Bayesian experimental design with costs (i.e., in Figures~\ref{fig_BO-time-housing}-\ref{fig_BO-cputime-segment}), the GSEMO requires more running time. However, it still takes much less time than the running time upper bound $O(n^2(\log n +k))$ derived in theory.

\begin{figure*}[t!]\centering
\begin{minipage}[c]{0.8\linewidth}\centering
        \includegraphics[width=\linewidth]{figures/legend_horizontal}
\end{minipage}\\\vspace{0.2em}
\begin{minipage}[c]{0.32\linewidth}\centering
        \includegraphics[width=1\linewidth]{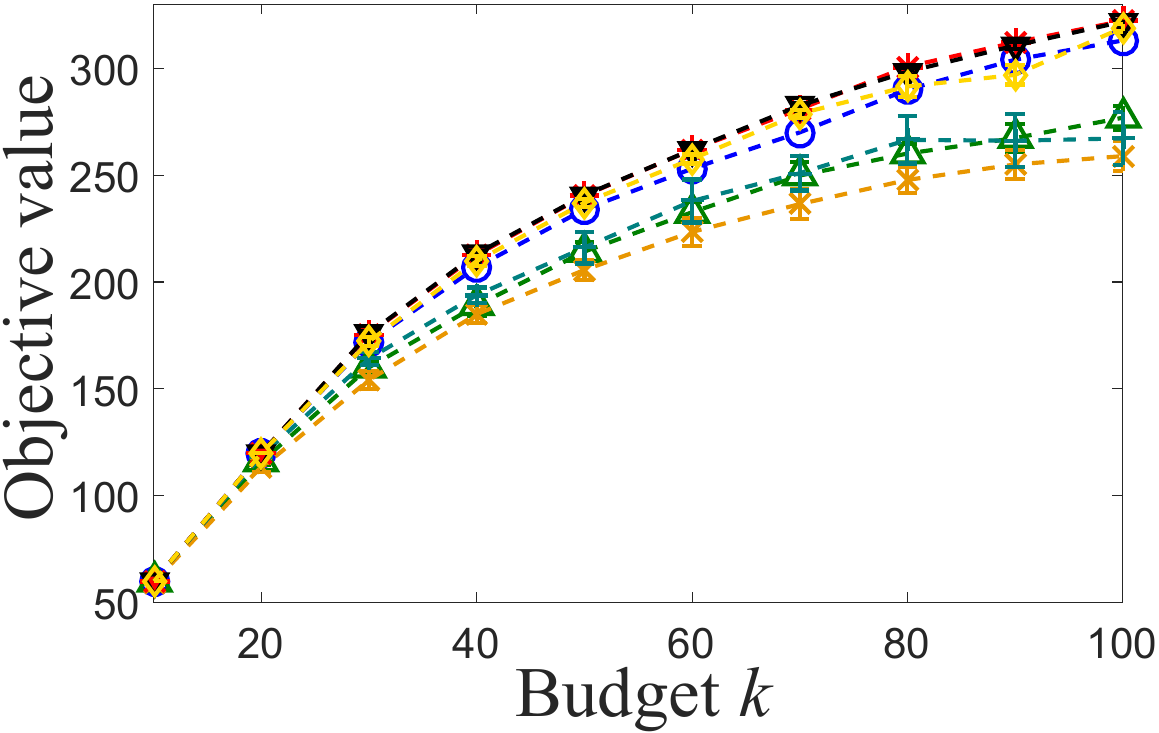}
\end{minipage}
\begin{minipage}[c]{0.32\linewidth}\centering
        \includegraphics[width=1\linewidth]{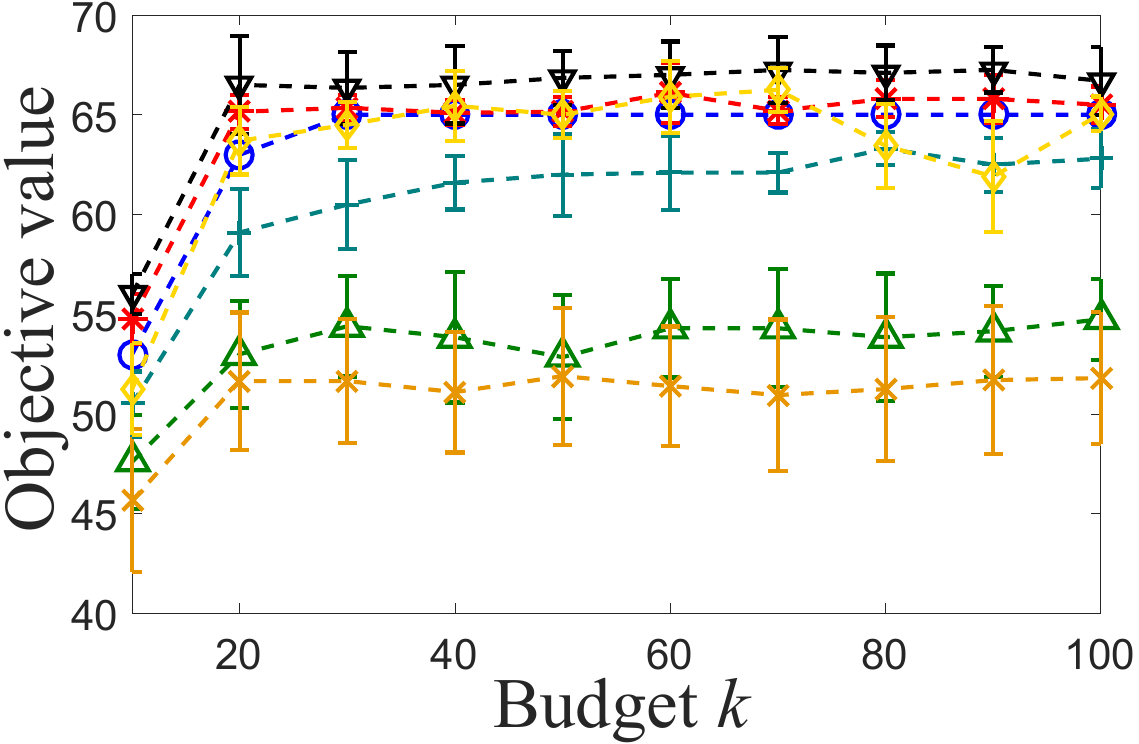}
\end{minipage}
\begin{minipage}[c]{0.32\linewidth}\centering
        \includegraphics[width=1\linewidth]{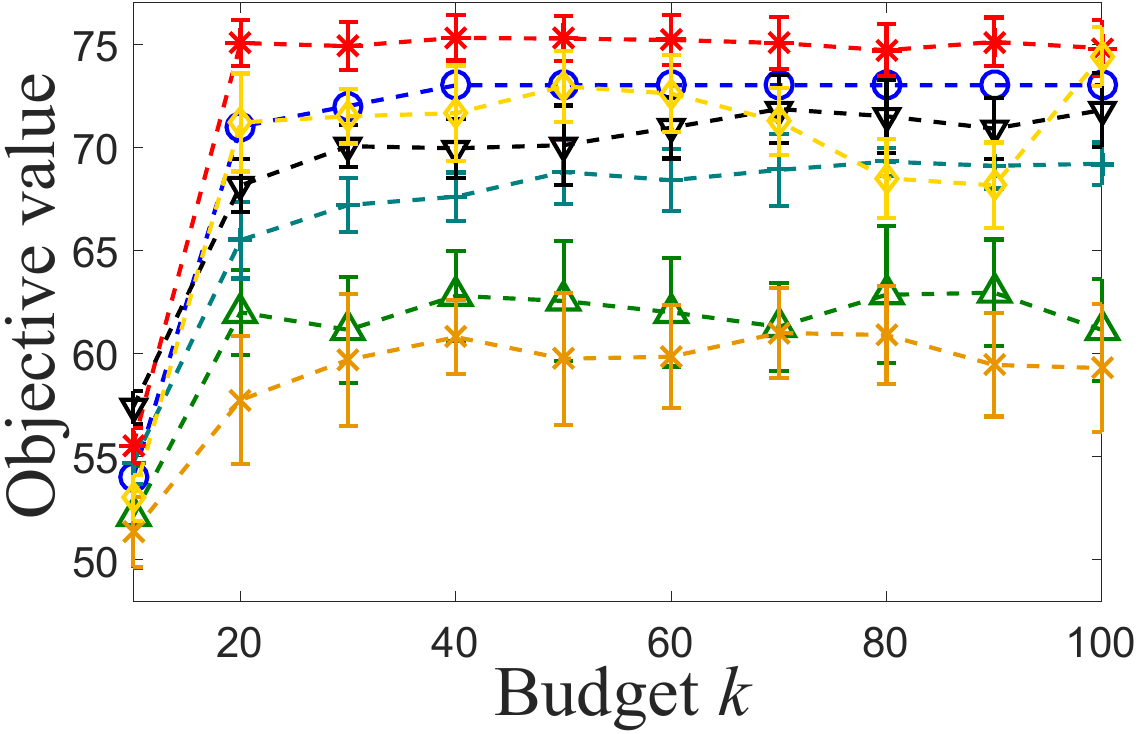}
\end{minipage}\\\vspace{0.3em}
\begin{minipage}[c]{0.32\linewidth}\centering
    \small(a) \textit{email-Eu-core}
\end{minipage}
\begin{minipage}[c]{0.32\linewidth}\centering
    \small(b) \textit{frb45-21-mis}
\end{minipage}
\begin{minipage}[c]{0.32\linewidth}\centering
    \small(c) \textit{frb53-24-mis}
\end{minipage}\vspace{-0.3em}
\caption{Comparison on the application of directed vertex cover with costs. The three data sets \textit{email-Eu-core}, \textit{frb45-21-mis} and \textit{frb53-24-mis} are used to generate three problem instances. Note that $\gamma$ is always 1.}\label{fig_vc}
\end{figure*}

\begin{table*}[t!]\centering
\caption{The objective values (mean+std.) of the compared algorithms on the application of directed vertex cover with costs, using the data sets \textit{email-Eu-core}, \textit{frb45-21-mis} and \textit{frb53-24-mis} with $q \in \{1,2,\ldots,12\}$. For each $q$, the largest values are bolded. In the rows of ``\#best", the largest values are bolded. The average rank of each algorithm on all $q$ is computed, where the smallest values are bolded. The ``\#direct win" denotes the number of $q$ on which the GSEMO has a larger objective value than the corresponding algorithm ($1$ tie is counted as $0.5$ win), where significant cells by the sign-test~\citep{demsar:06} with confidence level $0.05$ are bolded.\vspace{6pt}}\label{table_cover}
\small
\begin{lrbox}{\tablebox}
\begin{tabular}{c|r|rrrr|rr}
\hline
\multicolumn{8}{c}{}\\[-8pt]
\multicolumn{8}{c}{\textit{email-Eu-core}}\\
\hline
&&&&&&&\\[-8pt]
$q$ & \multicolumn{1}{c|}{GSEMO} & \multicolumn{1}{c}{DG} & \multicolumn{1}{c}{SDG(0.1)} & \multicolumn{1}{c}{SDG(0.2)} & \multicolumn{1}{c|}{Multi-SDG} & \multicolumn{1}{c}{GSEMO$_{g-c}$} & \multicolumn{1}{c}{NSGA-II} \\
\hline
&&&&&&&\\[-8pt]
1 &\bf{60.00$\pm$0.000}&  42    &40.35$\pm$1.014  &38.35$\pm$1.711   &41.40$\pm$1.020 &\bf{60.00$\pm$0.000}    &  59.40$\pm$0.490\\
2 &\bf{118.70$\pm$0.557}   &  115   &90.35$\pm$2.798  & 83.55$\pm$4.183   &91.15$\pm$6.239& 118.20$\pm$1.030    &  115.60$\pm$1.530\\
3 & 169.40$\pm$0.860   &   166  &142.15$\pm$3.825 & 130.85$\pm$4.564 &142.35$\pm$6.725&\bf{169.60$\pm$0.663}    &  166.70$\pm$1.552 \\
4 &\bf{196.85$\pm$0.910}  &  191   &168.20$\pm$3.995   &159.05$\pm$3.905   &173.90$\pm$6.098 &  196.30$\pm$1.100   &  192.90$\pm$1.546 \\
5 &227.65$\pm$1.014   &  222   &199.25$\pm$4.170 &190.3$\pm$4.681    &199.05$\pm$8.743& \bf{228.15$\pm$1.152}    &  224.30$\pm$1.833\\
6 & 261.70$\pm$1.382   &  253   & 233.05$\pm$3.598 &222.10$\pm$6.640  &233.30$\pm$7.721 &\bf{262.05$\pm$1.244}   &  257.70$\pm$1.764\\
7 &\bf{298.95$\pm$0.805}  &  289   & 266.35$\pm$5.102 &256.10$\pm$4.795    &269.70$\pm$8.379 &  297.85$\pm$1.526   &  295.75$\pm$2.321 \\
8 &\bf{328.85$\pm$1.526}  &  321   &295.70$\pm$4.473   & 285.80$\pm$7.298   &297.00$\pm$10.640  &  328.30$\pm$1.735    &  326.10$\pm$2.468\\
9 &\bf{360.35$\pm$1.152}  &  351   & 323.95$\pm$5.581 &313.05$\pm$5.268   &324.00$\pm$7.918   &  359.50$\pm$1.987   &  356.90$\pm$1.814 \\
10&\bf{391.15$\pm$0.792}  &  386   &352.00$\pm$4.919 &335.95$\pm$7.871   &354.80$\pm$10.921&  390.00$\pm$1.643   &  386.70$\pm$2.124\\
11&\bf{417.65$\pm$1.652} &  412   & 373.70$\pm$6.543  & 362.25$\pm$8.526  &379.60$\pm$10.519&  416.60$\pm$2.853          &  412.65$\pm$1.768\\
12&\bf{445.40$\pm$1.428}   &  432   & 396.45$\pm$6.217 & 383.70$\pm$6.092   &399.40$\pm$11.872&  443.55$\pm$2.179        &  439.35$\pm$3.103\\
\hline
&&&&&&&\\[-8pt]
\#best &   \multicolumn{1}{c|}{\bf 9} &  \multicolumn{1}{c}{0} &  \multicolumn{1}{c}{0}  &  \multicolumn{1}{c}{0} &  \multicolumn{1}{c|}{0} &  \multicolumn{1}{c}{4} &   \multicolumn{1}{c}{0}  \\
\hline
&&&&&&&\\[-8pt]
average rank &   \multicolumn{1}{c|}{\bf 1.292} &  \multicolumn{1}{c}{4.000} &  \multicolumn{1}{c}{5.917}  &  \multicolumn{1}{c}{7.000} &  \multicolumn{1}{c|}{5.083} &  \multicolumn{1}{c}{1.708} &   \multicolumn{1}{c}{3.000}  \\
\hline
\multicolumn{2}{c|}{}&&&&&&\\[-8pt]
\multicolumn{2}{c|}{GSEMO: \#direct win} & \multicolumn{1}{c}{\bf12} & \multicolumn{1}{c}{\bf12}  & \multicolumn{1}{c}{\bf12} & \multicolumn{1}{c|}{\bf12} & \multicolumn{1}{c}{8.5} &  \multicolumn{1}{c}{\bf12} \\
\hline
\hline
\multicolumn{8}{c}{}\\[-8pt]
\multicolumn{8}{c}{\textit{frb45-21-mis}}\\
\hline
&&&&&&&\\[-8pt]
$q$ & \multicolumn{1}{c|}{GSEMO} & \multicolumn{1}{c}{DG} & \multicolumn{1}{c}{SDG(0.1)} & \multicolumn{1}{c}{SDG(0.2)} & \multicolumn{1}{c|}{Multi-SDG} & \multicolumn{1}{c}{GSEMO$_{g-c}$} & \multicolumn{1}{c}{NSGA-II} \\
\hline
&&&&&&&\\[-8pt]
1 &9.25$\pm$0.536      & 8    &5.30$\pm$0.843  &4.95$\pm$0.589  &7.20$\pm$0.510 &\bf{9.55$\pm$0.497}    &  7.10$\pm$0.768\\
2 & 18.60$\pm$0.663    & 18   &12.60$\pm$1.463 &12.70$\pm$1.584  &16.65$\pm$0.853&\bf{19.15$\pm$0.726}    &  17.50$\pm$1.072\\
3 &29.35$\pm$0.572     &  28  &22.15$\pm$1.711 &21.00$\pm$1.924  &27.40$\pm$0.800&\bf{29.90$\pm$0.436}    &  29.20$\pm$1.122 \\
4 &\bf{41.15$\pm$1.108}     &  37  &31.35$\pm$2.475  &30.50$\pm$2.802  &38.20$\pm$1.288&41.05$\pm$0.865    &  40.10$\pm$1.375\\
5 &\bf{53.30$\pm$0.954}     &  52  &42.35$\pm$2.700  &41.40$\pm$2.458  &49.70$\pm$1.345&52.80$\pm$1.077     &  52.45$\pm$0.865\\
6 &65.80$\pm$1.208     & 65   &54.30$\pm$3.051  &52.15$\pm$2.725  &61.95$\pm$1.658&\bf{66.45$\pm$1.658}    &  65.80$\pm$1.778\\
7 &78.05$\pm$0.973     & 77   &64.30$\pm$3.132  &64.15$\pm$2.971  &73.60$\pm$1.715&\bf{78.20$\pm$2.337}    &  78.15$\pm$1.314\\
8 &91.35$\pm$0.654     & 88   &76.95$\pm$3.248  &74.45$\pm$4.068  &86.85$\pm$1.982&\bf{91.75$\pm$1.410}    &  90.25$\pm$1.374\\
9 &\bf{105.20$\pm$0.872}    &102   &89.35$\pm$4.175  &86.55$\pm$4.904  &99.60$\pm$2.010&105.00$\pm$1.844    &  105.15$\pm$0.963\\
10 &\bf{119.65$\pm$0.910}   &118   &102.85$\pm$2.886 &98.25$\pm$4.097  &111.65$\pm$2.515&117.000$\pm$2.470   &  119.30$\pm$1.187\\
11&\bf{130.90$\pm$1.091}    &127   &111.40$\pm$4.398  &110.50$\pm$4.006 &124.60$\pm$2.888&129.00$\pm$1.581  &  129.65$\pm$1.388\\
12&\bf{144.15$\pm$0.910}    &141   &126.35$\pm$4.246 &122.65$\pm$5.597 &136.40$\pm$2.035&142.75$\pm$2.586   &  142.80$\pm$1.288\\
\hline
&&&&&&&\\[-8pt]
\#best &   \multicolumn{1}{c|}{\bf6} &  \multicolumn{1}{c}{0} &  \multicolumn{1}{c}{0}  &  \multicolumn{1}{c}{0} &  \multicolumn{1}{c|}{0} &  \multicolumn{1}{c}{\bf6} &   \multicolumn{1}{c}{0}  \\
\hline
&&&&&&&\\[-8pt]
average rank &   \multicolumn{1}{c|}{\bf 1.625} &  \multicolumn{1}{c}{3.833} &  \multicolumn{1}{c}{6.083}  &  \multicolumn{1}{c}{6.917} &  \multicolumn{1}{c|}{4.833} &  \multicolumn{1}{c}{1.917} &   \multicolumn{1}{c}{2.792}  \\
\hline
\multicolumn{2}{c|}{}&&&&&&\\[-8pt]
\multicolumn{2}{c|}{GSEMO: \#direct win} & \multicolumn{1}{c}{\bf12} & \multicolumn{1}{c}{\bf12}  & \multicolumn{1}{c}{\bf12} & \multicolumn{1}{c|}{\bf12} & \multicolumn{1}{c}{6} &  \multicolumn{1}{c}{\bf12} \\
\hline
\hline
\multicolumn{8}{c}{}\\[-8pt]
\multicolumn{8}{c}{\textit{frb53-24-mis}}\\
\hline
&&&&&&&\\[-8pt]
$q$ & \multicolumn{1}{c|}{GSEMO} & \multicolumn{1}{c}{DG} & \multicolumn{1}{c}{SDG(0.1)} & \multicolumn{1}{c}{SDG(0.2)} & \multicolumn{1}{c|}{Multi-SDG} & \multicolumn{1}{c}{GSEMO$_{g-c}$} & \multicolumn{1}{c}{NSGA-II} \\
\hline
&&&&&&&\\[-8pt]
1 &10.00$\pm$0.000 &  10    &6.95$\pm$0.589  &7.00$\pm$0.548   &8.95$\pm$0.384 &\bf{10.30$\pm$0.458}    &  8.50$\pm$0.742\\
2 & 20.45$\pm$0.497  &  19    &16.45$\pm$1.161 &15.85$\pm$1.236 &19.65$\pm$0.477&\bf{21.10$\pm$0.539}    &  19.20$\pm$0.872\\
3 &\bf{33.60$\pm$0.490}   &  33    &27.00$\pm$1.517 &26.25$\pm$1.997 &31.60$\pm$0.663&33.50$\pm$0.806    &  32.50$\pm$1.396 \\
4 &\bf{48.50$\pm$0.742}   &  46    &39.05$\pm$2.692  &37.25$\pm$1.920 &45.75$\pm$1.090&47.10$\pm$0.943   &  47.55$\pm$1.244\\
5 &\bf{60.95$\pm$0.865}   &  58    &49.15$\pm$2.351  &47.80$\pm$2.502 &56.15$\pm$1.062&58.95$\pm$0.865   &  59.10$\pm$1.136\\
6 &\bf{74.75$\pm$1.220}   & 73     &62.90$\pm$3.223  &59.55$\pm$3.584 &68.60$\pm$1.562&71.30$\pm$1.520   &  72.95$\pm$1.658\\
7 &\bf{89.60$\pm$0.490}   &  87    &75.55$\pm$2.765  &74.00$\pm$3.479 &83.05$\pm$1.802&85.70$\pm$2.002   &  88.65$\pm$1.682\\
8 &\bf{105.50$\pm$0.592}  &  104   &90.20$\pm$3.234  &89.45$\pm$3.612 &98.55$\pm$2.269&102.05$\pm$1.631   &  104.05$\pm$1.161\\
9 &\bf{122.60$\pm$0.583}  &  121   &107.45$\pm$2.655 &102.75$\pm$3.223&114.40$\pm$1.772&118.45$\pm$1.936   &  119.85$\pm$1.276\\
10&\bf{140.75$\pm$0.829} &  136   &124.30$\pm$3.148 &118.75$\pm$3.998&130.80$\pm$2.462&136.10$\pm$2.844   &  138.25$\pm$1.894\\
11&\bf{157.95$\pm$0.218}  & 157    &137.95$\pm$4.006 &133.70$\pm$4.371&147.15$\pm$3.054&153.30$\pm$2.170  &  156.70$\pm$1.487\\
12&\bf{172.60$\pm$1.068}  & 168    &151.60$\pm$3.555 &143.65$\pm$4.567&160.30$\pm$3.084 &167.40$\pm$1.985   &  170.85$\pm$1.652\\
\hline
&&&&&&&\\[-8pt]
\#best &   \multicolumn{1}{c|}{\bf 10} &  \multicolumn{1}{c}{0} &  \multicolumn{1}{c}{0}  &  \multicolumn{1}{c}{0} &  \multicolumn{1}{c|}{0} &  \multicolumn{1}{c}{2} &   \multicolumn{1}{c}{0}  \\
\hline
&&&&&&&\\[-8pt]
average rank &   \multicolumn{1}{c|}{\bf 1.208} &  \multicolumn{1}{c}{3.125} &  \multicolumn{1}{c}{6.083}  &  \multicolumn{1}{c}{6.917} &  \multicolumn{1}{c|}{4.750} &  \multicolumn{1}{c}{3.083} &   \multicolumn{1}{c}{2.833}  \\
\hline
\multicolumn{2}{c|}{}&&&&&&\\[-8pt]
\multicolumn{2}{c|}{GSEMO: \#direct win} & \multicolumn{1}{c}{\bf11.5} & \multicolumn{1}{c}{\bf12}  & \multicolumn{1}{c}{\bf12} & \multicolumn{1}{c|}{\bf12} & \multicolumn{1}{c}{\bf10} &  \multicolumn{1}{c}{\bf12} \\
\hline
\end{tabular}
\end{lrbox}
\scalebox{0.73}{\usebox{\tablebox}}
\end{table*}

\begin{figure*}[t!]\centering
\begin{minipage}[c]{0.4\linewidth}\centering
        \includegraphics[width=\linewidth]{figures/time_legend}
\end{minipage}\\\vspace{0.2em}
\begin{minipage}[c]{0.32\linewidth}\centering
        \includegraphics[width=\linewidth]{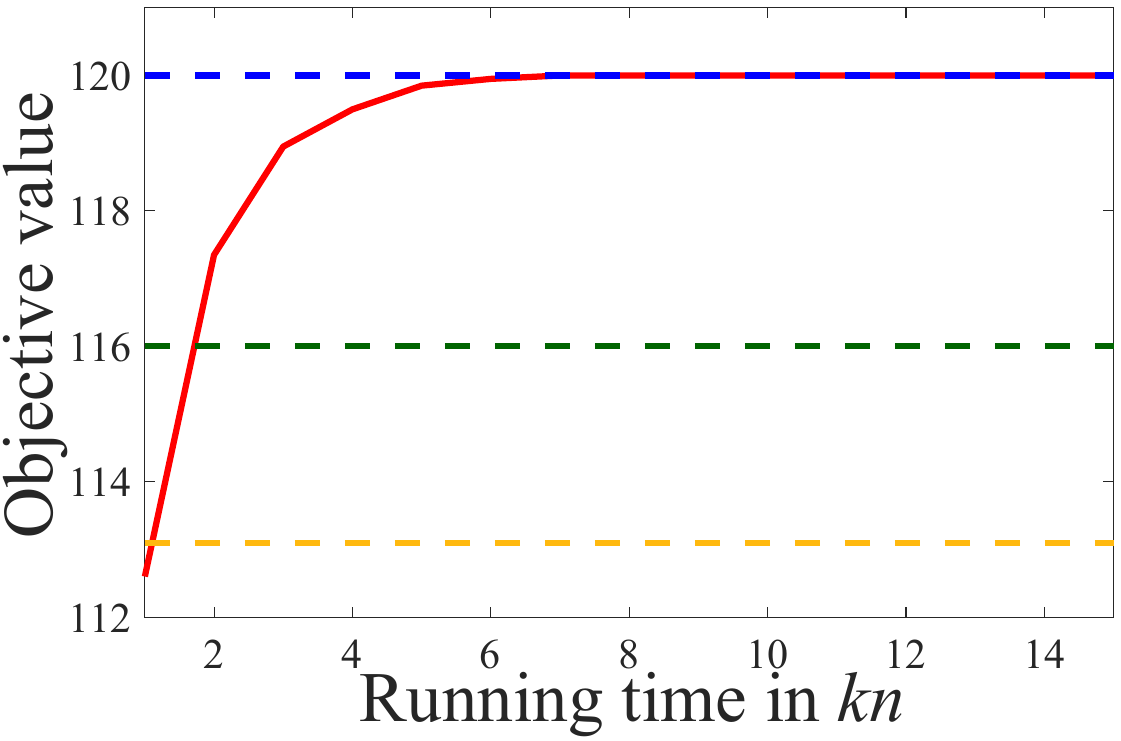}
\end{minipage}
\begin{minipage}[c]{0.32\linewidth}\centering
        \includegraphics[width=\linewidth]{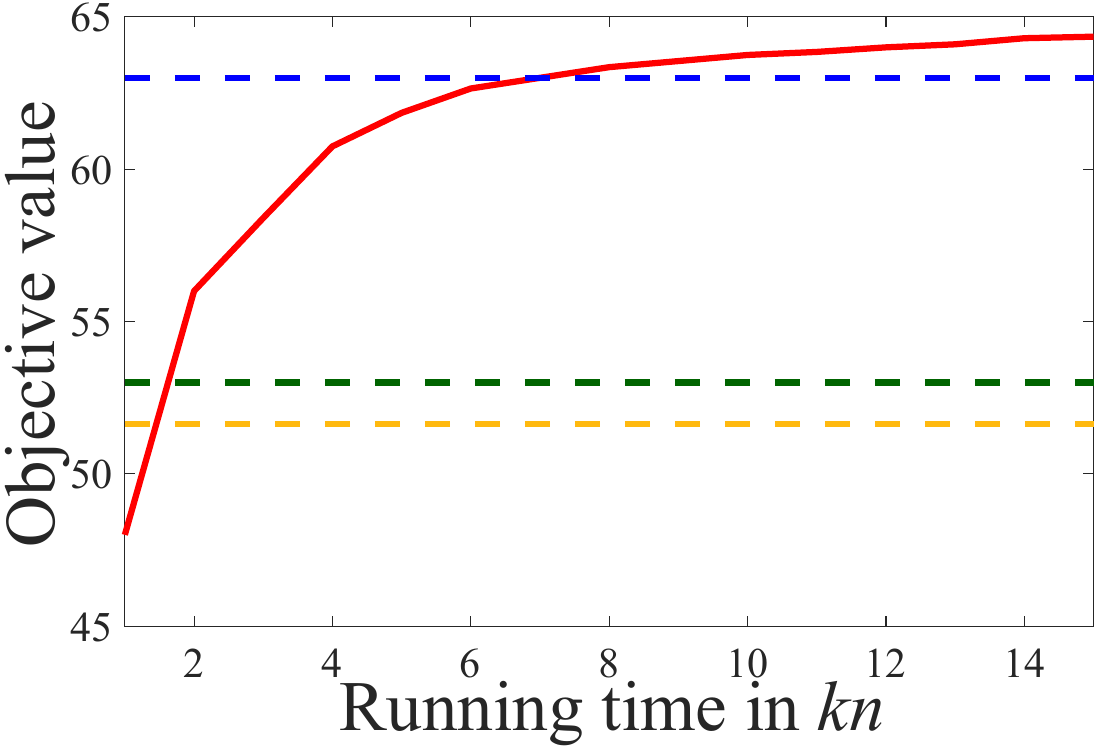}
\end{minipage}
\begin{minipage}[c]{0.32\linewidth}\centering
        \includegraphics[width=\linewidth]{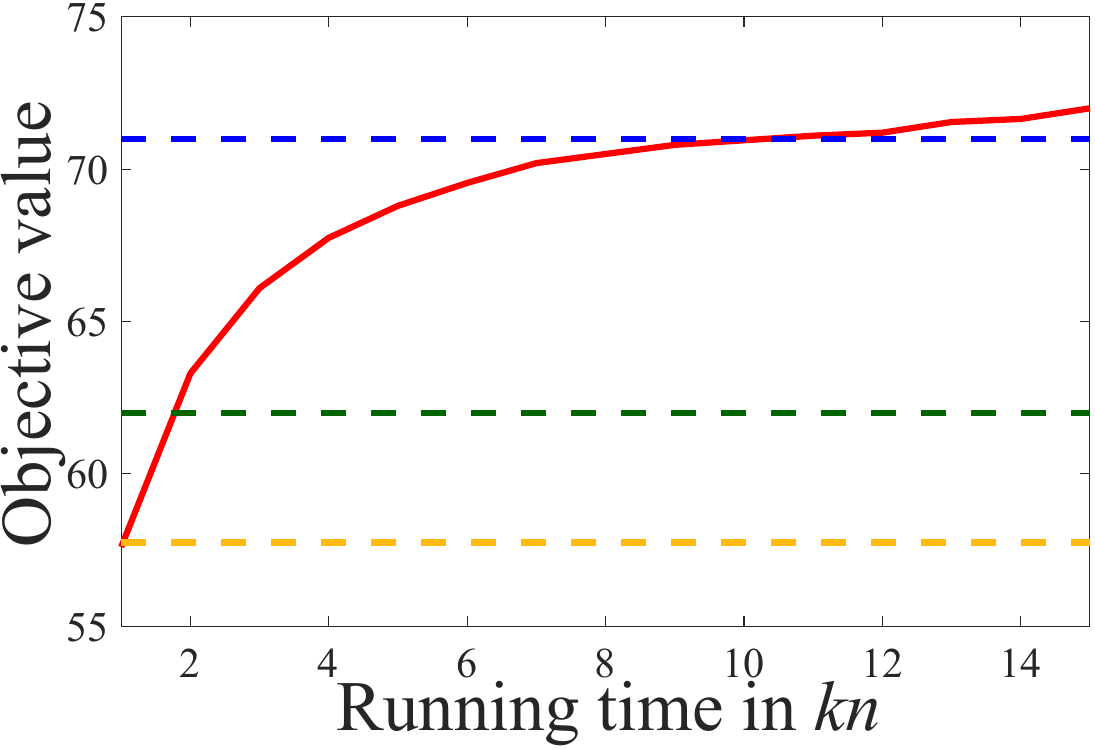}
\end{minipage}\\\vspace{0.3em}
\begin{minipage}[c]{0.32\linewidth}\centering
    \small(a) \textit{email-Eu-core}
\end{minipage}
\begin{minipage}[c]{0.32\linewidth}\centering
    \small(b) \textit{frb45-21-mis}
\end{minipage}
\begin{minipage}[c]{0.32\linewidth}\centering
    \small(c) \textit{frb53-24-mis}
\end{minipage}\vspace{-0.3em}
\caption{Performance vs. running time (i.e., \#objective evaluations) of the GSEMO on directed vertex cover with costs, where $k=20$.}\label{fig_vc-time}
\end{figure*}

\begin{figure*}[t!]\centering
\begin{minipage}[c]{0.32\linewidth}\centering
        \includegraphics[width=\linewidth]{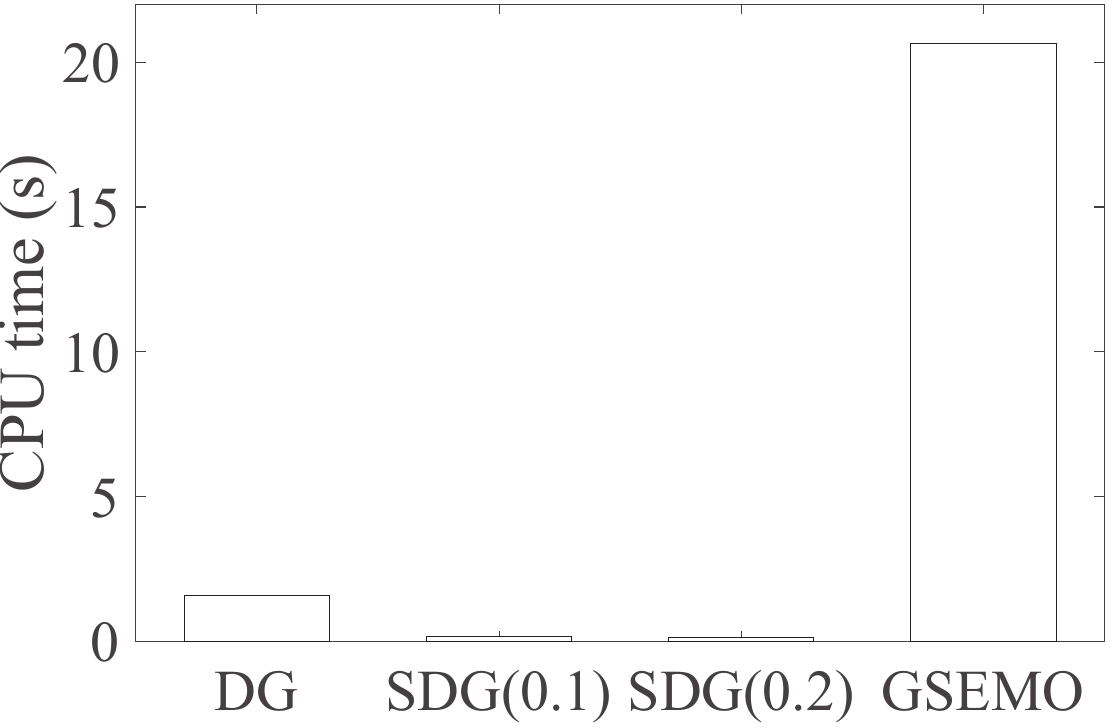}
\end{minipage}
\begin{minipage}[c]{0.32\linewidth}\centering
        \includegraphics[width=\linewidth]{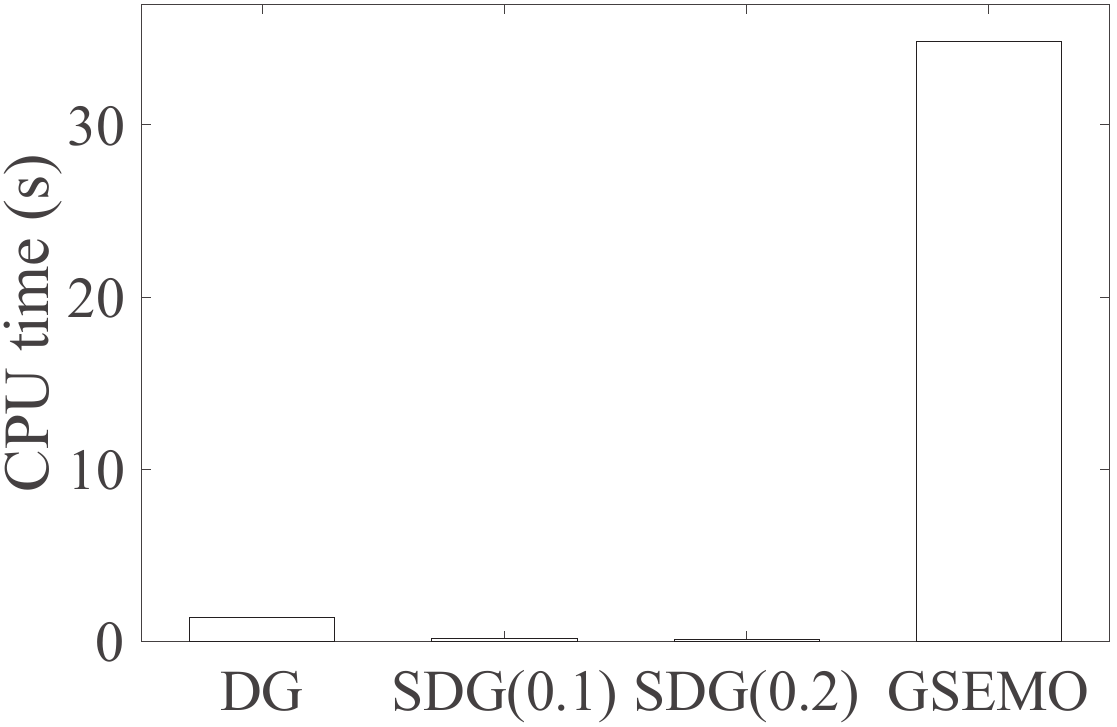}
\end{minipage}
\begin{minipage}[c]{0.32\linewidth}\centering
        \includegraphics[width=\linewidth]{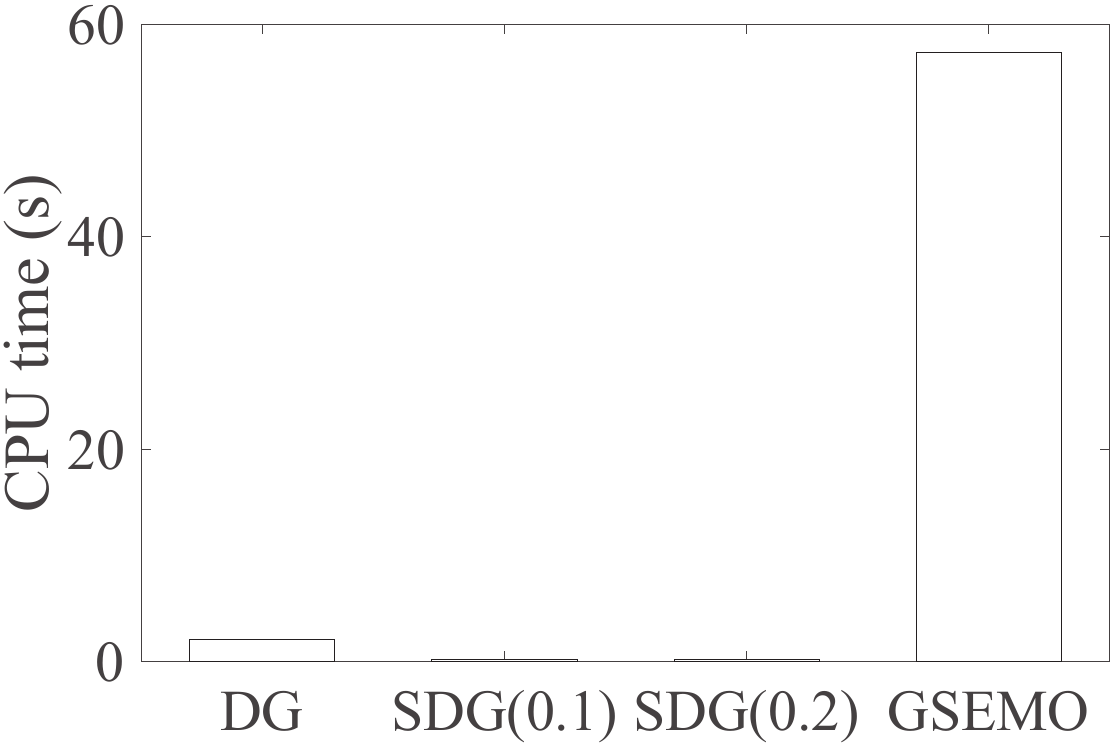}
\end{minipage}\\\vspace{0.3em}
\begin{minipage}[c]{0.32\linewidth}\centering
    \small(a) \textit{email-Eu-core}
\end{minipage}
\begin{minipage}[c]{0.32\linewidth}\centering
    \small(b) \textit{frb45-21-mis}
\end{minipage}
\begin{minipage}[c]{0.32\linewidth}\centering
    \small(c) \textit{frb53-24-mis}
\end{minipage}\vspace{-0.3em}
\caption{CPU running time of the DG, SDG(0.1), SDG(0.2) and GSEMO on directed vertex cover with costs, where $k=20$.}\label{fig_vc-cputime}
\end{figure*}

\section{Conclusion}\label{sec-conlusion}

In this paper, we analyze the approximation performance of the GSEMO for solving the problem $\arg\max\nolimits_{X \subseteq V} g(X)-c(X)$ s.t. $|X| \leq k$, where $g(X)$ is a non-negative monotone approximately submodular function and $c(X)$ is a non-negative modular function, and thus $(g(X)-c(X))$ can be non-monotone and non-submodular in general. We prove that by maximizing the distorted objective function $(1-\gamma/k)^{k-|X|}g(X)-c(X)+(|X|/k)c(V)$ and the subset size $|X|$ simultaneously, the GSEMO within $O(n^2(\log n+k))$ expected running time can find a subset $X$ satisfying that $|X|\leq k$ and $g(X)-c(X) \geq (1-e^{-\gamma})\cdot g(X^*) -c(X^*)$, where $\gamma$ is the submodularity ratio of $g$ and $X^*$ denotes an optimal subset. This reaches the best-known polynomial-time approximation guarantee, previously obtained by the distorted greedy algorithm. Furthermore, we prove that by maximizing the original objective function $(g(X)-c(X))$ and $|X|$ simultaneously, the GSEMO fails to achieve such an approximation guarantee in polynomial running time. Experimental results on the applications of Bayesian experimental design and directed vertex cover show the superior performance of the GSEMO over several distorted greedy-based algorithms and the popular NSGA-II algorithm. They also show that the performance of the GSEMO using the original objective function $(g(X)-c(X))$ can degrade largely sometimes, verifying the theoretical analysis.

Submodular optimization is originally defined for set functions, where a solution is a subset. Now it has been extended to the situations where a solution is a multiset or a sequence. Thus, it is expected to examine the performance of EAs for these extensions of submodular optimization. There has been some preliminary efforts toward this direction, e.g.,~\citep{qian2018sequence,qian2018constrained,qian2018multiset}. It is also interesting to study the behavior of EAs for submodular optimization under uncertain environments. For example, it has been proved that the GSEMO can maintain the approximation guarantee efficiently under dynamic constraints~\citep{aaai2019dynamic,aaai2021dynamic}, and can be robust against noise~\citep{qian2017subset,qian2019distributed}.

\section{Acknowledgments}

The author wants to thank the editor/associate editor and anonymous reviewers for their helpful comments and suggestions. This work was supported by the NSFC (62022039) and the Jiangsu NSF (BK20201247).

\small

\bibliographystyle{apalike}
\bibliography{ECJ-2019-116R2}

\end{document}